\documentclass[twoside,11pt]{article}

\usepackage[preprint]{jmlr2e}

\usepackage{lipsum}
\usepackage{amsfonts}
\usepackage{graphicx}
\usepackage{epstopdf}
\usepackage{algorithm}
\usepackage{algorithmic}
\usepackage{subcaption}
\usepackage{type1cm}        
\usepackage{makeidx}         
\usepackage{graphicx}        
\usepackage{multicol}        
\usepackage[bottom]{footmisc}

\makeindex             

\usepackage{graphicx}
\usepackage{amsfonts,amsmath, amssymb, amstext,color}
\usepackage{graphicx} 

\usepackage{mathrsfs}
\usepackage{comment}
\usepackage{caption}
\usepackage{subcaption}

\usepackage{mathtools}

\newcommand{\Wbf}{\mathbf{W}}

\newcommand{\KL}{{\rm KL}}

\newcommand{\Zbf}{\mathbf{Z}}
\newcommand{\Rrm}{{\rm R}}

\newcommand{\Pb}{\mathbb{P}}

\newcommand{\Acal}{\mathcal{A}}
\newcommand{\Xcal}{\mathcal{X}}
\newcommand{\dettwoX}{\mathrm{det_{2X}}}

\newcommand{\Nbb}{\mathbb{N}}
\newcommand{\GP}{\mathrm{GP}}

\newcommand{\PC}{\mathscr{PC}}

\newcommand{\Xbf}{\mathbf{X}}

\newcommand{\myspan}{\mathrm{span}}

\newcommand{\Gauss}{\mathrm{Gauss}}
\newcommand{\zbf}{\mathbf{z}}
\newcommand{\myIm}{\mathrm{Im}}
\newcommand{\Ical}{\mathcal{I}}
\newcommand{\ubf}{\mathbf{u}}

\newcommand{\mapto}{\ensuremath{\rightarrow}}

\newcommand{\approach}{\ensuremath{\rightarrow}}
\newcommand{\imply}{\ensuremath{\Rightarrow}}

\newcommand{\equivalent}{\ensuremath{\Longleftrightarrow}}
\newcommand{\inclusion}{\ensuremath{\hookrightarrow}}

\newcommand{\B}{\mathcal{B}}

\newcommand{\N}{\mathbb{N}}
\newcommand{\R}{\mathbb{R}}

\newcommand{\bP}{\mathbb{P}}
\newcommand{\bE}{\mathbb{E}}

\newcommand{\Fcal}{\mathcal{F}}

\renewcommand{\S}{\mathcal{S}}

\newcommand{\la}{\langle}
\newcommand{\ra}{\rangle}

\renewcommand{\H}{\mathcal{H}}
\renewcommand{\L}{\mathcal{L}}

\newcommand{\Lcal}{\mathcal{L}}

\def\supp{{\rm supp}}

\def\trace{{\rm tr}}

\newcommand\HS{{\rm HS}}
\newcommand\eHS{{\rm HS_X}}

\def\Sym{{\rm Sym}}

\def\tr{{\rm tr}}

\def\Tr{{\rm Tr}}

\newcommand{\x}{\mathbf{x}}

\newcommand{\z}{\mathbf{z}}

\renewcommand{\v}{\mathbf{v}}

\renewcommand{\b}{\mathbf{b}}

\def\1{\mathbf{1}}

\newcommand{\detX}{{\rm det_X}}
\newcommand{\trX}{{\rm tr_X}}

\newcommand{\Ncal}{\mathcal{N}}

\newcommand{\Bsc}{\mathscr{B}}
\newcommand{\Csc}{\mathscr{C}}

\newcommand{\logdet}{{\rm logdet}}

\newcommand{\dettwo}{{\rm det_2}}

\jmlrheading{x}{2022}{1-xx}{x/xx}{x/xx}{xx}{H\`a Quang Minh}

\ShortHeadings{R\'enyi divergences in RKHS and Gaussian process settings}{H\`a Quang Minh}
\firstpageno{1}

\begin{document}

\title{Kullback-Leibler and R\'enyi divergences in reproducing kernel Hilbert space and Gaussian process settings}

\author{\name H\`a Quang Minh\email minh.haquang@riken.jp \\
       \addr RIKEN Center for Advanced Intelligence Project\\
       1-4-1 Nihonbashi, Chuo-ku\\
       Tokyo, Japan}

\editor{}

\maketitle

\begin{abstract}
In this work, we present formulations for regularized Kullback-Leibler and R\'enyi divergences via the Alpha Log-Determinant (Log-Det) divergences between positive Hilbert-Schmidt operators on Hilbert spaces in two different settings, namely (i) covariance operators and Gaussian measures defined on reproducing kernel Hilbert spaces (RKHS); and (ii) Gaussian processes with squared integrable sample paths. For characteristic kernels, the first setting leads to divergences between arbitrary Borel probability measures on a complete, separable metric space.
We show that the Alpha Log-Det divergences are continuous in the Hilbert-Schmidt norm, which enables us to apply laws of large numbers for Hilbert space-valued random variables. As a consequence of this,
we show that, in both settings, the infinite-dimensional divergences can be consistently and efficiently estimated from their
finite-dimensional versions, using finite-dimensional Gram matrices/Gaussian measures and finite sample data, with {\it dimension-independent} sample complexities in all cases.
RKHS methodology plays a central role in the theoretical analysis in both settings.
The mathematical formulation is illustrated by numerical experiments.
\end{abstract}


\section{Introduction}
\label{section-introduction}

Gaussian measures and 
Symmetric Positive Definite (SPD) matrices, particularly covariance matrices,
play a central role in many areas of mathematics, statistics, and have been playing an increasingly important role in
many applications in 
machine learning, optimization, computer vision, and related fields, 
see 
e.g. 
\citep{Mostow:1955,
	LawsonLim:2001, 
	Bhatia:2007,LogEuclidean:SIAM2007, Pennec:IJCV2006,Tuzel:PAMI2008,Kulis:2009,Cherian:PAMI2013,LogE:CVPR2013-1}. 
In particular, much recent work in the literature has been carried out to exploit the rich geometrical structures of the set {$\Sym^{++}(n)$} of {$n \times n$} 
SPD
matrices, including its 
open convex cone and various Riemannian manifold structures.

Among the most widely studied and applied Riemannian 
metrics on {$\Sym^{++}(n)$} are 
the affine-invariant metric
\citep{Mostow:1955,
	LawsonLim:2001, 
	Bhatia:2007, Pennec:IJCV2006,
	Tuzel:PAMI2008}, which corresponds to the Fisher-Rao metric between zero-mean Gaussian densities on $\R^n$;
%
Log-Euclidean metric 
\citep{LogEuclidean:SIAM2007,
	WangGDD12,
	LogE:CVPR2013-1,LogE:CVPR2013-2}; Wasserstein-Riemannian metric \citep{Takatsu2011wasserstein,Bhatia:2018Bures,Malago:WassersteinGaussian2018},
which corresponds to the $\Lcal^2$-Wasserstein distance between zero-mean Gaussian measures on $\R^n$
\citep{givens84,dowson82,olkin82,knott84}; and Log-Cholesky metric \citep{lin2019riemannian}, which is based on the Cholesky decomposition of SPD matrices.

%
The convex cone structure of {$\Sym^{++}(n)$}, on the other hand, gives rise to distance-like functions such as 
the Alpha Log-Determinant divergences \citep{Chebbi:2012Means}, which correspond to R\'enyi divergences between zero-mean Gaussian densities \citep{Renyi:1961}, including in particular the Kullback-Leibler (KL) divergence
\citep{Kullback1951information}.
While generally not satisfying the properties of metric distances,  such as symmetry and triangle inequality, these divergences are fast to compute and have been shown to work well
in practice \citep{Kulis:2009,Sra:NIPS2012,Cherian:PAMI2013}.
The entropic regularization of the $\Lcal^2$-Wasserstein distance also leads to a symmetric divergence, namely the Sinkhorn divergence, between zero-mean Gaussian measures on $\R^n$ \citep{Janati2020entropicOT,Mallasto2020entropyregularized}.
Recent work attempting to unify the various distances/divergences on $\Sym^{++}(n)$ include 
\citep{amari2018information,cichocki15,thanwerdas2019affine,thanwerdas2022geometry,Minh:GSI2019,Minh:Alpha2022}.

{\bf Infinite-dimensional setting}. The generalization of distances/divergences for Gaussian measures and covariance matrices on $\R^n$ to the infinite-dimensional setting of Gaussian measures and covariance operators on Hilbert spaces has been carried out by various authors.
Among these, 
the $\Lcal^2$-Wasserstein distance 
between Gaussian measures on a Hilbert space has the same expression 
as in the finite-dimensional setting
\citep{Gelbrich:1990Wasserstein,cuesta1996:WassersteinHilbert}.
The same holds true for its entropic regularization, the Sinkhorn divergence \citep{Minh2020:EntropicHilbert}.
In general, however, the infinite-dimensional formulations are substantially more complex than the finite-dimensional ones. The KL divergence between two Gaussian measures on a Hilbert space, for example, is finite if and only if 
the two measures are equivalent, a consequence of the classical Feldman-Hajek Theorem \citep{Feldman:Gaussian1958,Hajek:Gaussian1958} (see further in Section \ref{section:infinite-generalization}). In order to obtain finite, well-defined distances/divergences between all pairs of Gaussian measures/covariance operators,
regularization on the covariance operators is often  necessary. This is due to the fact that operations such as inversion, logarithm, and determinant, are only well-defined for specific classes of operators. This is 
the case for the affine-invariant Riemannian metric \citep{Larotonda:2007}, the Log-Hilbert-Schmidt metric
\citep{MinhSB:NIPS2014}, which generalizes the Log-Euclidean metric, the Alpha and Alpha-Beta Log-Determinant divergences \citep{Minh:LogDet2016,Minh:LogDetIII2018,Minh:2019AlphaBeta,Minh:Positivity2020,Minh:2020regularizedDiv}.
The settings for these distances/divergences are the sets of positive definite unitized trace class/
Hilbert-Schmidt operators, which are positive trace class/Hilbert-Schmidt operators plus a positive scalar multiple of the identity operator so that operations such as inversion, logarithm, and determinant, are well-defined. 

{\bf Reproducing kernel Hilbert space (RKHS) setting}.
From the computational and practical viewpoint, this setting is particularly interesting since
many quantities of interest admit closed forms via kernel Gram matrices which can be efficiently computed.
Examples include the kernel Maximum Mean Discrepancy (MMD) \citep{Gretton:MMD12a} and the RKHS covariance operators, 
the latter resulting in powerful nonlinear algorithms with
substantial improvements over finite-dimensional covariance matrices, see e.g. \citep{ProbDistance:PAMI2006,Covariance:CVPR2014,MinhSB:NIPS2014,Minh:Covariance2017,zhang2019:OTRKHS} for examples of applications in computer vision. 

{\bf Gaussian process setting}.
The study of distances/divergences between Gaussian processes and covariance operators of stochastic processes in general has been carried out by many authors in statistics, particularly functional data analysis, and machine learning.
The Kullback-Leibler (KL) divergence plays a crucial role in Gaussian processes for machine learning
\citep{Rasmussen:2006Gaussian}. 
In \citep{Matthews2016sparseKL,Sun2019functionalKL}, the authors studied the KL divergence between stochastic processes
in the setting of functional variational Bayesian neural networks. We compare in particular the formulation proposed in \cite{Sun2019functionalKL} with the current work in Section \ref{section:estimate-Gaussian-process}.
The problem of minimizing the KL divergence between two infinite-dimensional probability measures
was studied in \citep{pinski2015kullbacktheory,pinski2015algorithms}.
In \citep{Panaretos:jasa2010}, a test statistic was proposed for the equality of Gaussian process covariance operators based on the approximation of their Hilbert-Schmidt distance.
This approach was generalized to covariance operators of general functional random processes in \citep{Fremdt:2013testing}. 
Further related work in this direction includes e.g. \citep{Paparoditis2016Bootstrap,Boente2018Testing}.
Since these methods are based on the Hilbert-Schmidt distance, they
do not take into account the intrinsic properties of the set of covariance operators, such as symmetry and positivity. In \citep{Pigoli:2014}, several non-Euclidean distances between covariance operators were studied, including
the square root and Procrustes distances, the latter also known as the Bures-Wasserstein distance,
which corresponds to
the $\Lcal^2$-Wasserstein distance between two centered Gaussian measures on Hilbert space in optimal transport.
%
The $2$-Wasserstein distance was applied to Gaussian processes in \citep{masarotto2019procrustes,Mallasto:NIPS2017Wasserstein},
with the treatment in the latter being valid in the finite-dimensional setting.
The entropic Wasserstein distances between Gaussian processes were studied in \citep{Minh2021:FiniteEntropicGaussian},
with dimension-independent sample complexities, in contrast to the exact Wasserstein distance.

{\bf Contributions of current work}.
In this work, we focus on the infinite-dimensional Alpha Log-Determinant (Log-Det) divergences as formulated in \citep{Minh:LogDet2016,Minh:Positivity2020}, which correspond to R\'enyi divergences, including in particular Kullback-Leibler (KL) divergence, between Gaussian measures on a Hilbert space \citep{Minh:2020regularizedDiv}.
These are defined via the concepts of {\it extended Fredholm} and {\it extended Hilbert-Carleman determinants}, generalizing the classical Fredholm and Hilbert-Carleman determinants \citep{Fredholm:1903, Gohberg:1969,Simon:1977}. The following are the concrete contributions of the current work
\begin{enumerate}
	\item We show that the Alpha Log-Det divergences in \citep{Minh:LogDet2016,Minh:Positivity2020} between positive Hilbert-Schmidt operators
	converge in Hilbert-Schmidt norm, leading to their consistent finite-rank/finite-dimensional approximations.
	\item In \citep{Minh:LogDet2016}, in the RKHS setting,
	the empirical divergence formulas, in terms of kernel Gram matrices, are presented for two data sets of the same size.
	In this work, we present formulas that are valid for two data sets of any finite sizes.
	\item Based on the Hilbert-Schmidt norm convergence, we apply the laws of large numbers for Hilbert space-valued random variables to show that, in the RKHS setting, the Alpha Log-Det divergences and, consequently, the R\'enyi and KL divergences, can be consistently estimated from finite sample data.
	The sample complexities are {\it dimension-independent} in all cases.
	These results thus provide rigorous theoretical guarantees for the algorithms that employ 
	these divergences in practice.
	\item We apply the Alpha Log-Det divergences to formulate regularized versions of the R\'enyi and KL divergences between Gaussian processes, and more generally, covariance operators of stochastic processes.
	By representing these quantities via RKHS covariance and cross-covariance operators,
	we obtain finite-dimensional approximations of these divergences, all with {\it dimension-independent} sample complexities.
\end{enumerate}
{\bf Other prior and related work}.
In the setting of RKHS covariance operators, the approach most closely related to 
our framework
is \citep{ProbDistance:PAMI2006}, which 
computed
probabilistic distances in RKHS. 
This was
employed by 
\citep{Covariance:CVPR2014} to compute Bregman divergences between RKHS covariance operators
and applied in computer vision applications.
This approach is mathematically valid in the strictly finite-dimensional setting, since the infinite determinant is well-defined only for specific classes of infinite-dimensional operators. 
Our framework provides a rigorous formulation for the infinite-dimensional RKHS setting and 
the sample complexities presented in the current work also provides theoretical guarantees for the algorithms 
utilizing these divergences.

%
%
%
In \citep{cuturi2005semigroup}, 
a 
form of 
regularized determinant for the covariance operator of the average of two measures was considered.
This is different from our setting, where we compute divergences between two covariance operators. Another related work is \citep{kondor2003kernel}, which defined
the Bhattacharyya kernel in RKHS by projecting onto the subspace spanned by the input data.
In \citep{sledge2021estimating}, the authors propose several notions of matrix-based R\'enyi cross-entropies in RKHS.
However, these generally are {\it not} valid divergences, see Section \ref{section:renyi-RKHS-infinite}
for the detailed discussion.


\begin{remark}
We note that while the above approaches are formulated directly in the RKHS setting, using {\it kernelization},
our framework is formulated in the general Hilbert space setting.
In the current work, we present results using this general formulation in two different settings,
namely covariance operators defined on (i) an RKHS and (ii) an $\Lcal^2$ space in the setting of Gaussian processes, which is not an RKHS.
\end{remark}

\section{Finite-dimensional setting}
\label{section:background}
We first recall the Kullback-Leibler and R\'enyi divergences between two probability density functions $P_1$ and $P_2$ on $\R^n$.
Let $0 < r < 1$ be fixed. The R\'enyi divergence of order $r$
between $P_1$ and $P_2$ is  defined to be, see e.g. \citep{Pardo:2005}
\begin{align}
	\label{equation:dR-integral}
	D_{R,r}(P_1|| P_2) = -\frac{1}{r(1-r)}\log\int_{\R^n}P_1(x)^{r}P_2(x)^{1-r}dx.
\end{align}
We note that this formula differs from the original definition given by \citep{Renyi:1961},
namely 	$D_{R,r}(P_1|| P_2) = -\frac{1}{(1-r)}\log\int_{\R^n}P_1(x)^{r}P_2(x)^{1-r}dx$, by a factor $1/r$ and
gives a non-trivial limit when $r \approach 0$ (the original formula gives $D_{R,0}(P_1||P_2) = 0$).
As $r \approach 1$ and $r \approach 0$, the R\'enyi divergence gives the 
Kullback-Leibler 
(KL) divergence \citep{Kullback1951information} ${\KL}(P_1||P_2)$ and ${\KL}(P_2||P_1)$, respectively,
\begin{align}
	\label{equation:dKL-integral}
	&\lim_{r \approach 1}D_{R,r}(P_1||P_2) = {\KL}(P_1|| P_2) = \int_{\R^n}P_1(x)\log\frac{P_1(x)}{P_2(x)}dx,
	\\
	&	\lim_{r \approach 0}D_{R,r}(P_1||P_2) = {\KL}(P_2|| P_1) = \int_{\R^n}P_2(x)\log\frac{P_2(x)}{P_1(x)}dx.
\end{align}
For Gaussian densities on $\R^n$, the 
KL and R\'enyi divergences admit closed form formulas, corresponding to the {\it Alpha Log-Determinant divergences} between
SPD matrices , as follows.

{\bf Log-Determinant divergences}. Let $\Sym^{++}(n)$ denote the set of $n \times n$ symmetric positive definite (SPD) matrices.
For {$A,B \in \Sym^{++}(n)$},
$-1 \leq \alpha \leq 1$,
the {\it Alpha Log-Determinant}  (or Log-Det) {\it {$\alpha$}-divergence} between 
{$A$} and {$B$} is a parametrized 
family of divergences defined by (see \citep{Chebbi:2012Means})
{
	\begin{align}\label{equation:logdet-alpha-finite}
		d^{\alpha}_{\logdet}(A,B) &= \frac{4}{1-\alpha^2}\log\frac{\det(\frac{1-\alpha}{2}A + \frac{1+\alpha}{2}B)}{\det(A)^{\frac{1-\alpha}{2}}\det(B)^{\frac{1+\alpha}{2}}},\;\;\;\; -1 < \alpha < 1,
		\\
		d^1_{\logdet}(A,B) &= \trace(B^{-1}A - I) - \log\det(B^{-1}A),
		\label{equation:alpha+1-finite}
		\\
		d^{-1}_{\logdet}(A,B) &= \trace(A^{-1}B-I) - \log\det(A^{-1}B).
		\label{equation:alpha-1-finite}
	\end{align}
}
For two Gaussian densities $\Ncal(m_1, C_1)$ and $\Ncal(m_2,C_2)$ on $\R^n$,
\begin{align}
	\label{equation:Renyi-gaussian-finite}
	D_{R,r}(\Ncal(m_1,C_1)||\Ncal(m_2, C_2)) &= \frac{1}{2}\la m_2-m_1, [(1-r)C_1 + rC_2]^{-1}(m_2 - m_1)\ra
	\nonumber
	\\
	&	+\frac{1}{2}d^{2r-1}_{\logdet}(C_1,C_2), 
	\\
	{\KL}(\Ncal(m_1,C_1)|| \Ncal(m_2,C_2)) &= \frac{1}{2}\la m_2-m_1, C_2^{-1}(m_2 - m_1)\ra + \frac{1}{2}d^1_{\logdet}(C_1,C_2).
	\label{equation:KL-gaussian-finite}
\end{align}

\section{Infinite-dimensional generalization}
\label{section:infinite-generalization}

We now describe the generalization of the finite-dimensional results in
Section \ref{section:background} to the infinite-dimensional setting, in particular to Gaussian measures on a Hilbert space. 
This formulation was presented in \citep{Minh:LogDet2016,Minh:2019AlphaBeta,Minh:Positivity2020,Minh:2020regularizedDiv}.

Throughout the paper, let $\H,\H_1,\H_2$ be infinite-dimensional separable real Hilbert spaces.
Let $\Lcal(\H_1,\H_2)$ denote the set of bounded linear operators between $\H_1$ and $\H_2$.
For $\H_1 = \H_2 = \H$, we write $\Lcal(\H)$.
Let $\Sym(\H) \subset \Lcal(\H)$ denote
the set of bounded, self-adjoint operators, $\Sym^{+}(\H)$
the set of positive operators,
i.e.
$\Sym^{+}(\H) = \{A \in \Sym(\H): A=A^{*},\la x, Ax \ra \geq 0\forall x \in \H\}$,
 $\Sym^{++}(\H)$
the set of strictly positive operators,
i.e.
$\Sym^{++}(\H) = \{A \in \Sym(\H): A=A^{*},\la x, Ax \ra > 0\forall x \in \H, x \neq 0\}$.
The set of trace class operators on $\H$ is defined to be
$\Tr(\H) = \{A \in \Lcal(\H): ||A||_{\tr} = \sum_{k=1}^{\infty}\la e_k, (A^{*}A)^{1/2}e_k\ra < \infty\}$,
where $\{e_k\}_{k=1}^{\infty}$ is any orthonormal basis in $\H$ and
the trace norm
$||A||_{\tr}$ is independent of the choice of such basis. For $A \in \Tr(\H)$, the trace of $A$ is 
$\trace(A) = \sum_{k=1}^{\infty}\la e_k, A e_k\ra = \sum_{k=1}^{\infty}\lambda_k$, where $\{\lambda_k\}_{k=1}^{\infty}$ denote
the eigenvalues of $A$.
For two separable Hilbert spaces $\H_1,\H_2$,
the set of Hilbert-Schmidt operators between $\H_1$ and $\H_2$
is defined to be, see e.g. \citep{Kadison:1983},
$\HS(\H_1,\H_2) = \{A \in \Lcal(\H_1,\H_2): ||A||_{\HS}^2 = \trace(A^{*}A) = \sum_{k=1}^{\infty}||Ae_k||^2_{\H_2} < \infty\}$,
the Hilbert-Schmidt norm $||A||_{\HS}$ being independent of the choice of orthonormal basis $\{e_k\}_{k=1}^{\infty}$
in $\H_1$. For $\H_1 = \H_2 = \H$, we write $\HS(\H)$.
The set $\HS(\H_1,\H_2)$ is itself a Hilbert space with 
the Hilbert-Schmidt inner product $\la A,B\ra_{\HS} = \trace(A^{*}B) = \sum_{k=1}^{\infty}\la Ae_k, Be_k\ra_{\H_2}$.

{\bf  R\'enyi and KL divergences between Gaussian measures on a Hilbert space}.
Let us now state the general definition for R\'enyi and KL divergences.
For two measures $\mu$ and $\nu$ on a measure space $(\Omega, \Fcal)$, with $\mu$ $\sigma$-finite, 
$\nu$ is said to be {\it absolutely continuous} with respect to $\mu$, denoted by $\nu << \mu$, if for any $A \in \Fcal$, $\mu(A) = 0 \imply \nu(A) = 0$. In this case, the Radon-Nikodym derivative $\frac{d\nu}{d\mu} \in \Lcal^1(\mu)$ is 
well-defined.
The 
KL divergence between $\nu$ and $\mu$ is defined by
\begin{align}
	\KL(\nu||\mu) = 
	\left\{
	\begin{matrix}
		\int_{\Omega}\log\left\{\frac{d\nu}{d\mu}(x)\right\}d\nu(x) & \text{if $\nu << \mu$},
		\\
		\infty & \text{otherwise}.
	\end{matrix}
	\right.
\end{align} 
Likewise, for $0 < r <1$, the  R\'enyi divergence of order $r$ is defined by
\begin{align}
D_{R,r}(\nu ||\mu) = 
\left\{
\begin{matrix}
-\frac{1}{r(1-r)}\log\int_{\Omega}\left(\frac{d\nu}{d\mu}\right)^rd\mu(x) & \text{if $\nu << \mu$},
\\
\infty & \text{otherwise}.
\end{matrix}
\right.
\end{align}
If $\mu << \nu$ and $\nu << \mu$, then we say that $\mu$ and $\nu$ are {\it equivalent}, denoted by $\mu \sim \nu$.
We say that $\mu$ and $\nu$ are {\it mutually singular}, denoted by $\mu \perp \nu$, if there exist $A,B \in \Fcal$ such that
$\mu(A) = \nu(B) = 1$ and $A\cap B = \emptyset$.

{\bf Equivalence of Gaussian measures}.
On $\R^n$, two Gaussian densities are always equivalent and the R\'enyi and KL divergences between them are always finite, as given in Section \ref{section:background}. On a Hilbert space with $\dim(\H) = \infty$, the situation is drastically different. 
Let $Q,R$ be two self-adjoint, positive trace class operators on $\H$ such that $\ker(Q) = \ker(R) = \{0\}$. Let $m_1, m_2 \in \H$. 
A fundamental result in the theory of Gaussian measures is the Feldman-Hajek Theorem (\cite{Feldman:Gaussian1958,Hajek:Gaussian1958}), which states that 
two Gaussian measures $\mu = \Ncal(m_1,Q)$ and 
$\nu = \Ncal(m_2, R)$
are either mutually singular or
equivalent, {\color{black}that is either $\mu \perp \nu$ or $\mu \sim \nu$}.
The necessary and sufficient conditions for
the equivalence of the two Gaussian measures $\nu$ and $\mu$ are given by the following.

\begin{theorem}
	(\textbf{\cite{Bogachev:Gaussian}, Corollary 6.4.11,\cite{DaPrato:PDEHilbert},
	Theorems  1.3.9 and 1.3.10})
\label{theorem:Gaussian-equivalent}
	Let $\H$ be a separable Hilbert space. Consider two Gaussian measures $\mu = \Ncal(m_1, Q)$ and
	$\nu = \Ncal(m_2, R)$ on $\H$. Then $\mu$ and $\nu$ are equivalent if and only if the following
	hold
	\begin{enumerate}
		\item $m_2 - m_1 \in \myIm(Q^{1/2})$.
		\item There exists  $S \in  \Sym(\H) \cap \HS(\H)$, without the eigenvalue $1$, such that
		$R = Q^{1/2}(I-S)Q^{1/2}$.
	\end{enumerate}
\end{theorem}

With the above equivalence condition, the R\'enyi and KL divergences between 
two Gaussian measures $\mu = \Ncal(m_1,Q)$ and $\nu = \Ncal(m_2,R)$ on $\H$ admit explicit formulas. If $\mu \perp \nu$, then
$\KL(\nu|| \mu) = D_{R,r}(\nu||\mu) = \infty$, $0 < r < 1$. If $\mu \sim \nu$, then we have the following result. 
\begin{theorem}
	[\cite{Minh:2020regularizedDiv}]
	\label{theorem:KL-gaussian}
	Let $\mu = \Ncal(m_1, Q)$, $\nu = \Ncal(m_2, R)$, with $\ker(Q) = \ker{R} = \{0\}$, and
	$\mu \sim \nu$.
	Let $S \in \HS(\H)\cap \Sym(\H)$, $I-S > 0$, 
	be such that $R = Q^{1/2}(I-S)Q^{1/2}$, then the KL divergence between $\nu$ and $\mu$ is
	\begin{align}
		\label{equation:KL-gaussian-infinite}
		\KL(\nu ||\mu) = \frac{1}{2}||Q^{-1/2}(m_2-m_1)||^2 -\frac{1}{2}\log\dettwo(I-S).
	\end{align}
For $0 < r < 1$, the R\'enyi divergence of order $r$ between $\nu$ and $\mu$ is given by
\begin{align}
\label{equation:Renyi-gaussian-infinite}
D_{\Rrm,r}(\nu || \mu) 
&= \frac{1}{2}||[I-(1-r)S]^{-1/2}Q^{-1/2}(m_2-m_1)||^2 
\nonumber
\\
&\quad +\frac{1}{2r(1-r)}\log\det[(I-S)^{r-1}(I-(1-r)S)].
\end{align}
\end{theorem}
One can verify directly that in the finite-dimensional setting, these reduce to  Eqs.\eqref{equation:KL-gaussian-finite} and \eqref{equation:Renyi-gaussian-finite}, respectively.
In Eqs.\eqref{equation:KL-gaussian-infinite} and \eqref{equation:Renyi-gaussian-infinite}, $\det$ and $\dettwo$ refer to the Fredholm and
Hilbert-Carleman determinants (see e.g. \cite{Simon:1977}), respectively.
 Let $A \in \Tr(\H)$ and $\{\lambda_k\}_{k\in \Nbb}$ denote the set of eigenvalues of $A$. The Fredholm determinant $\det(I+A)$
 is given by
\begin{align}
	\det(I+A) = \prod_{k=1}^{\infty}(1+\lambda_k).
\end{align} 
For operators of the form $I+A$, $A \in \HS(\H)$, the Fredholm determinant is generally not well-defined. Instead, we have the following Hilbert-Carleman determinant
\begin{align}
	\dettwo(I+A) = \det[(I+A)\exp(-A)].
\end{align}
In particular, if $A \in \Tr(\H)$, then
\begin{align}
\dettwo(I+A) = \det(I+A)\exp(-\trace(A)).
\end{align}
For zero-mean Gaussian densities on $\R^n$, there is a direct correspondence between the R\'enyi and KL divergences in Eqs.\eqref{equation:Renyi-gaussian-finite}, \eqref{equation:KL-gaussian-finite} and the Alpha Log-Det divergences in Eqs.\eqref{equation:logdet-alpha-finite}, \eqref{equation:alpha+1-finite}, \eqref{equation:alpha-1-finite}.
This direct correspondence is no longer valid in the infinite-dimensional Hilbert space setting.
This is because if $A \in \Tr(\H)\cap \Sym^{++}(\H)$, with eigenvalues $\{\lambda_k\}_{k \in \Nbb}$, then $\lim_{k \approach \infty}\lambda_k = 0$, so that $A^{-1}$ is unbounded and $\log\det(A) = \sum_{k=1}^{\infty}\log{\lambda_k} = -\infty$, hence the expressions in Eqs.\eqref{equation:logdet-alpha-finite}, \eqref{equation:alpha+1-finite}, \eqref{equation:alpha-1-finite}
are generally not well-defined or unbounded.

{\bf Extended trace class operators and extended Fredholm determinant}. 
To obtain the infinite-dimensional generalization of
Eqs.\eqref{equation:logdet-alpha-finite}, \eqref{equation:alpha+1-finite}, \eqref{equation:alpha-1-finite} that is
valid 
for all positive trace class operators $A,B$, i.e. $A,B \in \Sym^{+}(\H) \cap \Tr(\H)$, 
we consider the setting
of {\it positive definite unitized trace class operators} \citep{Minh:LogDet2016}. 
The set of {\it extended (unitized) trace class operators} is defined by
\citep{Minh:LogDet2016} to be
\begin{align}
	\Tr_X(\H) = \{ A + \gamma I: A \in \Tr(\H), \gamma \in \R\}.
\end{align}
For $(A+\gamma I) \in \Tr_X(\H)$, the {\it extended trace} is defined to be \citep{Minh:LogDet2016}
\begin{align}
	\trX(A+\gamma I) = \trace(A) + \gamma.
\end{align}
By this definition, $\trX(I) = 1$, in contrast to
the standard trace 
$\trace(I) = \infty$.
%
Let $A \in \Tr(\H)$. 
The Fredholm determinant $\det(I+A)$
is generalized in \citep{Minh:LogDet2016} to the {\it extended Fredholm determinant}, which, for 
$\dim(\H) = \infty$ and $\gamma \neq 0$, is given by
\begin{align}
	\detX(A+\gamma I) = \gamma \det\left(I + [A/\gamma]\right),
\end{align}
which reduced to the Fredholm determinant when $\gamma = 1$.

{\bf Positive definite (unitized) trace class operators}.
We recall that an operator $A \in \Lcal(\H)$ is said to be {\it positive definite} \citep{Petryshyn:1962} if there exists
a constant $M_A > 0$ such that $\la x, Ax\ra \geq M_A||x||^2$ $\forall x \in \H$.
This condition is equivalent to requiring that $A$ be both {\it strictly positive}, 
i.e. $\la x, Ax\ra > 0$ $\forall x \neq 0$, and {\it invertible}, with $A^{-1}\in \Lcal(\H)$.
Let $\bP(\H)$ denote the set of {\it self-adjoint positive definite} bounded operators on $\H$,
namely
$\bP(\H) = \{A \in \L(\H), A^{*} = A, \exists M_A > 0 \; s.t. \la x, Ax\ra \geq M_A||x||^2 \; \forall x \in \H\}$.
We write $A > 0 \equivalent A \in \bP(\H)$.
The set of {\it positive definite unitized
	trace class operators} is then defined to be
\begin{align}
	\PC_1(\H) = \Pb(\H) \cap \Tr_X(\H) = \{A+\gamma I > 0: A \in \Tr(\H), \gamma \in \R\}.
\end{align}

On the set $\PC_1(\H)$, we obtain the following generalization of the Alpha Log-Determinant divergence,
which satisfies all properties of a divergence (\cite{Minh:LogDet2016})
\begin{definition}
	(\textbf{Alpha Log-Det divergences between positive definite trace class operators} \citep{Minh:LogDet2016})
	\label{def:logdet}
	Assume that {$\dim(\H) = \infty$}. For {$-1 < \alpha < 1$},
	the Log-Det {$\alpha$}-divergence 
	$d^{\alpha}_{\logdet}[(A+\gamma I), (B+ \mu I)]$ for $(A+\gamma I), (B+\mu I) \in \PC_1(\H)$
	is defined to be
	{		%
		\begin{align}\label{equation:det1}
			&d^{\alpha}_{\logdet}
			[(A+\gamma I), (B+ \mu I)]
			\nonumber
			\\
			& 
			= \frac{4}{1-\alpha^2}  
			\log\left[\frac{\detX\left(\frac{1-\alpha}{2}(A+\gamma I) + \frac{1+\alpha}{2}(B+\mu I)\right)}{\detX(A+\gamma I)^{\beta}\detX(B + \mu I)^{1-\beta}}\left(\frac{\gamma}{\mu}\right)^{\beta - \frac{1-\alpha}{2}}\right],
		\end{align}
	}
	where {$\beta = \frac{(1-\alpha)\gamma}{(1-\alpha) \gamma + (1+\alpha)\mu}$} and {$1-\beta = \frac{(1+\alpha)\mu}{(1-\alpha) \gamma + (1+\alpha)\mu}$}.
	
\end{definition}
For $\alpha = \pm 1$, $d^{\alpha}_{\logdet}$ is defined based on the limits as $\alpha \approach \pm 1$, as follows.
\begin{align}
	&d^{1}_{\logdet}[(A+\gamma I), (B+\mu I)] = \lim_{\alpha \approach 1^{-}}d^{\alpha}_{\logdet}[(A+\gamma I), (B+\mu I)]
	\nonumber
	\\
	& = \left(\frac{\gamma}{\mu}-1\right)\log\frac{\gamma}{\mu} 
	+ \trX[(B+\mu I)^{-1}(A+\gamma I) - I] - \frac{\gamma}{\mu}\log\detX[(B+\mu I)^{-1}(A+\gamma I)],
	\label{equation:alpha+1}
	\\
	&d^{-1}_{\logdet}[(A+\gamma I), (B+\mu I)] = \lim_{\alpha \approach -1^{+}}d^{\alpha}_{\logdet}[(A+\gamma I), (B+\mu I)]
	\nonumber
	\\
	& = 
	\left(\frac{\mu}{\gamma}-1\right)\log\frac{\mu}{\gamma} 
	+ \trX\left[(A+\gamma I)^{-1}(B+\mu I) - I\right] - \frac{\mu}{\gamma}\log\detX[(A+\gamma I)^{-1}(B+\mu I)].
	\label{equation:alpha-1}
\end{align}

{\bf Special case}. In the case $\gamma = \mu$, $d^{\alpha}_{\logdet}[(A+\gamma I), (B+\gamma I)]$ is a direct generalization of the finite-dimensional formulation
{\small
	\begin{align}
		d^{\alpha}_{\logdet}[(A+\gamma I), (B+ \gamma I)] &= \frac{4}{1-\alpha^2}\log\left[\frac{\detX(\frac{1-\alpha}{2}(A+\gamma I) + \frac{1+\alpha}{2}(B+\gamma I))}{\detX(A+\gamma I)^{\frac{1-\alpha}{2}}\detX(B+ \gamma I)^{\frac{1+\alpha}{2}}}\right],
		\\
		d^{1}_{\logdet}[(A+\gamma I), (B+\gamma I)] &= \trX[(B+\gamma I)^{-1}(A+\gamma I)-I] - \log\detX[(B+\gamma I)^{-1}(A+\gamma I)],
		\\
		d^{-1}_{\logdet}[(A+\gamma I), (B+\gamma I)] &= \trX[(A+\gamma I)^{-1}(B+\gamma I) - I] - \log\detX[(A+\gamma I)^{-1}(B+\gamma I)].
	\end{align} 
}
In particular, for $A,B \in \Sym^{++}(n)$, by letting $\gamma =0$, we recover Eqs.\eqref{equation:logdet-alpha-finite},
\eqref{equation:alpha+1-finite}, and \eqref{equation:alpha-1-finite}.

The formulas for the R\'enyi and KL divergences between Gaussian densities on $\R^n$ in Eqs. \eqref{equation:Renyi-gaussian-finite} and \eqref{equation:KL-gaussian-finite}, respectively, 
are not directly generalizable to the infinite-dimensional setting. However, using the Alpha Log-Det divergences between positive definite trace class operators, we can define their regularized versions on a separable Hilbert space $\H$.

\begin{definition}
	\label{definition:renyi-regularized}
	(\textbf{Regularized R\'enyi and KL divergences between Gaussian measures on Hilbert space})
Let $m_1, m_2 \in \H$ and $C_1, C_2 \in \Sym^{+}(\H) \cap \Tr(\H)$. Let $\gamma \in \R, \gamma > 0$ be fixed.
The {\it regularized  R\'enyi divergence} of order $r$, $0 \leq r \leq 1$, between the Gaussian measures
$\Ncal(m_1, C_1)$, $\Ncal(m_2,C_2)$, is defined to be
\begin{align}
	D_{\Rrm,r}^{\gamma}[\Ncal(m_1, C_1)|| \Ncal(m_2,C_2)]
	&=\frac{1}{2}\la m_1 - m_2,[(1-r)(C_1+\gamma I) + r(C_2 + \gamma I)]^{-1}(m_1 - m_2)\ra
	\\
	&\quad +\frac{1}{2}d^{2r-1}_{\logdet}[(C_1 + \gamma I), (C_2 + \gamma I)].
	\nonumber
\end{align}
In particular, for $r =1$, we obtain the {\it regularized KL divergence}
\begin{align}
	\KL^{\gamma}[\Ncal(m_1, C_1)|| \Ncal(m_2,C_2)]
	&= \frac{1}{2}\la m_1 - m_2, (C_2 + \gamma I)^{-1}(m_1 - m_2)\ra
	\\
	& \quad +\frac{1}{2}d^{1}_{\logdet}[(C_1 + \gamma I), (C_2 + \gamma I)].
	\nonumber
\end{align}
\end{definition}

The correspondence between the Alpha Log-Det and regularized R\'enyi and KL divergences with the exact R\'enyi and KL divergences between two {\it equivalent} Gaussian measures (in Theorem \ref{theorem:KL-gaussian}) is given in the following
\begin{theorem}
	[\citep{Minh:2020regularizedDiv}, Theorems 2 and 3]
	\label{theorem:exact-regularized-correspondence}
	Let $\nu_0 \sim \Ncal(m,C_0)$, $\nu \sim \Ncal(m, C)$, $m \in \H, C_0,C \in \Sym^{++}(\H) \cap \Tr(\H)$, be two 
	{\it equivalent} Gaussian measures, that is $m-m_0 \in \myIm(C_0^{1/2})$ and there exists $S \in \Sym(\H) \cap \HS(\H)$ such that
	$C = C_0^{1/2}(I-S)C_0^{1/2}$. Then
	\begin{align}
	\lim_{\gamma \approach 0^{+}}D^{\gamma}_{\Rrm,r}(\nu||\nu_0) = D_{\Rrm,r}(\nu ||\nu_0),\\
		\lim_{\gamma \approach 0^{+}}\KL^{\gamma}(\nu||\nu_0) = \KL(\nu ||\nu_0).
	\end{align}
In particular
	\begin{align}
		\lim_{\gamma \approach 0}d^{2r-1}_{\logdet}[(C+\gamma I), (C_0 +\gamma I)]&= \frac{1}{r(1-r)}\log\det[(I-(1-r)S)(I-S)^{r-1}]
		\\
		& = 2D_{\Rrm,r}(\nu ||\nu_0), \;\;\; 0 < r < 1
		\\
		\lim_{\gamma \approach 0}d^1_{\logdet}[(C+\gamma I), (C_0 + \gamma I)] &= -\log\dettwo(I-S) = 2 {\KL}(\nu||\nu_0).
	\end{align}
	
\end{theorem}

{\bf Generalization to positive definite (unitized) Hilbert-Schmidt operators}.
Definition \ref{def:logdet} is extended to the more general Alpha-Beta Log-Det divergences in
\citep{Minh:2019AlphaBeta}. It is extended further to all {\it positive definite (unitized) Hilbert-Schmidt operators}, as follows \citep{Minh:Positivity2020}.
The set of {\it extended (or unitized) Hilbert-Schmidt operators} on $\H$ is defined in \citep{Larotonda:2007} to be
\begin{align}
	\HS_X(\H) = \{ A + \gamma I: A \in \HS(\H), \gamma \in \R\}.
\end{align}
This set includes in particular $\Tr_X(\H)$. 
It is a Hilbert space under the {\it extended Hilbert-Schmidt inner product} and {\it extended Hilbert-Schmidt norm}
\begin{align}
	\la A+\gamma I, B + \nu I\ra_{\eHS} = \la A,B\ra_{\HS} + \gamma\nu,\;\;\;||A + \gamma I||^2_{\eHS} = ||A||^2_{\HS} + \gamma^2.
\end{align}
Under $\la ,\ra_{\eHS}$, the scalar operators $\gamma I$, $\gamma \in \R$, are orthogonal to the Hilbert-Schmidt operators.
With the norm $||\;||_{\eHS}$,
$||I||_{\eHS} = 1$, in contrast to the Hilbert-Schmidt norm, where $||I||_{\HS} = \infty$.
With the extended Hilbert-Schmidt operators, we define 
the set of {\it positive definite (unitized) Hilbert-Schmidt} operators on $\H$ to be
\begin{align}
	\PC_2(\H) = \bP(\H) \cap \HS_X(\H) = \{A+ \gamma I > 0: A^{*} = A, A \in \HS(\H), \gamma \in \R\}.
\end{align}
This is a Hilbert manifold, being an open subset of the Hilbert space $\HS_X(\H)$.

Similar to the Fredholm determinant, we generalize the Hilbert-Carleman determinant 
$\dettwo(I+A)$, $A \in \HS(\H)$, to the
the {\it extended Hilbert-Carleman} determinant, which, for $(A+\gamma I) \in \HS_X(\H)$, is defined by \citep{Minh:Positivity2020}
\begin{align}
	\dettwoX(A + \gamma I) = \detX[(A+\gamma I)\exp(-A/\gamma))]= \gamma \dettwo\left(I + \frac{A}{\gamma}\right).
\end{align}
In particular, for $A \in \Tr(\H)$,
\begin{align}
	\dettwoX(A+\gamma I) = \detX(A+\gamma I)\exp(-\trace(A)/\gamma).
\end{align}

On $\PC_2(\H)$, we obtain the following generalization of Definition \ref{def:logdet}.
\begin{definition}
	(\textbf{Alpha Log-Det divergences between positive definite Hilbert-Schmidt operators \citep{Minh:Positivity2020}})
	\label{def:logdet-HS}
	Let $-1 \leq \alpha \leq 1$ be fixed. For
	$(A+\gamma I), (B+\mu I) \in \PC_2(\H)$, define $\frac{\gamma}{\mu}I + \Lambda = (B+\mu I)^{-1/2}(A+\gamma I)(B+\mu)^{-1/2}$. The $\alpha$-Log-Det divergence between $(A+\gamma I)$ and $(B+\mu I)$ is defined to be
	{\small
		\begin{align}
			d^{\alpha}_{\logdet}[(A+\gamma I), (B+\mu I)] &=
			\frac{4}{1-\alpha^2}\log\detX\left[\frac{1-\alpha}{2}\left(\frac{\gamma}{\mu}I+\Lambda\right)^{1-\beta} + \frac{1+\alpha}{2}\left(\frac{\gamma}{\mu}I+\Lambda\right)^{-\beta}\right]
			\\
			&\quad + \frac{4(\beta-\frac{1-\alpha}{2})}{1-\alpha^2}\log\left(\frac{\gamma}{\mu}\right), \;\;\; -1 < \alpha < 1,
		\end{align}
	}
	where $\beta$ is as defined in Definition \ref{def:logdet}.
	At the limiting points $\alpha = \pm 1$,
	{\small
		\begin{align}
			d^{1}_{\logdet}[(A+\gamma I), (B+\mu I)] & = \left(\frac{\gamma}{\mu}-1\right)\left(1+\log\frac{\gamma}{\mu}\right) - \frac{\gamma}{\mu}\log\dettwoX[(B+\mu I)^{-1}(A+\gamma I)],
			\label{equation:alpha+1-HS}
			\\
			d^{-1}_{\logdet}[(A+\gamma I), (B+ \mu I)] & = \left(\frac{\mu}{\gamma}-1\right)\left(1+\log\frac{\mu}{\gamma}\right)
			-\frac{\mu}{\gamma}\log\dettwoX[(A+\gamma I)^{-1}(B+\mu I)].
			\label{equation:alpha-1-HS} 
		\end{align}
	}
\end{definition}
The case $\gamma = \mu$ is of particular interest, as it gives rise to divergences between positive Hilbert-Schmidt operators themselves. The following result is immediate
\begin{theorem}
\label{theorem:Log-Det-positive-HS}
Let $\gamma \in \R, \gamma > 0$ be fixed. Then $d^{\alpha}_{\logdet}[(A+\gamma I),(B+\gamma I)]$
is a divergence on the set $\Sym^{+}(\H) \cap \HS(\H)$.
\end{theorem}
{\bf Discussion on the Hilbert-Schmidt operator setting}. The setting of positive Hilbert-Schmidt operators is more general than that of covariance operators, which are positive trace class operators.
With the formulation using extended Hilbert-Schmidt operators and extended Hilbert-Carleman determinant, we obtain
the convergence in Hilbert-Schmidt norm of the Alpha Log-Det divergence, leading to its finite-rank/finite-dimensional approximation (Section \ref{section:approximation}), along with its finite sample complexity in the RKHS and Gaussian process settings (see Sections \ref{section:RKHS}, \ref{section:renyi-RKHS-infinite}, and \ref{section:estimate-Gaussian-process})).
In particular, laws of large numbers for Hilbert space-valued random variables can be applied to obtain sample complexities. In this work, we make use of the following result, which
is a consequence of a general result by Pinelis (\cite{Pinelis1994optimum}, Theorem 3.4).
\begin{proposition}
	[\cite{SmaleZhou2007}]
	\label{proposition:Pinelis}
	Let $(Z,\Acal, \rho)$ be a probability space and
	$\xi:(Z,\rho) \mapto \H$ be a random variable.
	Assume that $\exists M > 0$ such that $||\xi|| \leq M < \infty$ almost surely.
	Let $\sigma^2(\xi)= \bE||\xi||^2$. Let $(z_i)_{i=1}^m$ be independently sampled 
	from $(Z, \rho)$.
	For any $0 < \delta < 1$, with probability at least $1-\delta$,
	\begin{align}
		\left\|\frac{1}{m}\sum_{i=1}^m\xi(z_i) - \bE\xi \right\| \leq \frac{2M\log\frac{2}{\delta}}{m} + \sqrt{\frac{2\sigma^2(\xi)\log\frac{2}{\delta}}{m}}.
	\end{align}
\end{proposition}

\section{Finite-rank and finite-dimensional approximations}
\label{section:approximation}

In this section, we show that the Alpha Log-Det divergence $d^{\alpha}_{\logdet}$, as defined in Definitions \ref{def:logdet}
and \ref{def:logdet-HS}, converges in the Hilbert-Schmidt norm. This convergence enables us to obtain finite-rank/finite-dimensional approximation of $d^{\alpha}_{\logdet}$. In particular, 
in the RKHS setting in Section \ref{section:estimate-RKHS}, we can apply laws of large numbers for
Hilbert space-valued random variables to obtain finite, dimension-independent sample complexities
for $d^{\alpha}_{\logdet}$.

In the following, consider the set
\begin{align} 
	\label{equation:omega-set}
	\Omega = \{X \in \Sym(\H) \cap \HS(\H): I+X > 0\} \supset \Sym^{+}(\H) \cap \HS(\H).
\end{align} 
Let $0 \leq \beta \leq 1$ be fixed.
Define the following function $F_{\beta}: \Omega \mapto \R$ by
\begin{align}
	\label{equation:F-beta}
	F_{\beta}(X) &= \frac{1}{\beta(1-\beta)}\log\det[\beta (I+X)^{1-\beta} + (1-\beta)(I+X)^{-\beta}], \;\;0 < \beta < 1.
	\\
	F_0(X) &= \lim_{\beta \approach 0}F_{\beta}(X) = -\log\dettwo(I+X),
	\\
	F_1(X) &= \lim_{\beta \approach 1}F_{\beta}(X) = -\log\dettwo[(I+X)^{-1}].
\end{align}
The following results show that the function $d^{\alpha}_{\logdet}$, $-1\leq \alpha \leq 1$, can be expressed in terms of 
$F_{\beta}$ for some $\beta$, $0 \leq \beta \leq 1$. Furthermore, $F_{\beta}$ is continuous in the $||\;||_{\HS}$ norm.

\begin{proposition}
	\label{proposition:logdet-F-beta-representation}
	Let $-1\leq \alpha \leq 1$ be fixed. Let $\gamma > 0$ be fixed.
	For $(A+\gamma I), (B+\gamma I) \in \PC_2(\H)$, define $\Lambda \in \Sym(\H)\cap \HS(\H)$ by
	$I + \Lambda = (B+\gamma I)^{-1/2}(A+\gamma I)(B+\gamma I)^{-1/2}$. Then
	\begin{align}
		d^{\alpha}_{\logdet}[(A+\gamma I), (B+\gamma I)]
		= F_{\frac{1-\alpha}{2}}(\Lambda).
	\end{align}
\end{proposition}

\begin{theorem}
	\label{theorem:F-A-continuity-HS}
	Let $0 \leq \beta \leq 1$ be fixed.
	Let $A,B \in \Omega$, with $\Omega$ as defined in Eq.\eqref{equation:omega-set}.
	Then
	\begin{align}
		0 \leq
		F_{\beta}(A)
		&\leq ||(I+A)^{-1}||\;||A||^2_{\HS}.
		\\
		|F_{\beta}(A) - F_{\beta}(B)| 
		&\leq  ||(I+A)^{-1}||\;||(I+B)^{-1}||
		\nonumber
		\\
		&\quad \times [||A||_{\HS} + ||B||_{\HS} + ||A||_{\HS}||B||_{\HS}]||A-B||_{\HS}.
	\end{align}
	Furthermore, for  $A,B \in \Sym^{+}(\H) \cap \HS(\H)$,
	\begin{align}
		& 0 \leq F_{\beta}(A) \leq \frac{1}{2}||A||^2_{\HS},
		\\
		&|F_{\beta}(A) - F_{\beta}(B)|
		\leq \frac{1}{2}||A-B||_{\HS}(||A||_{\HS} + ||B||_{\HS}).
	\end{align}
\end{theorem}
In particular, the cases $\beta = \pm 1$ involve the Hilbert-Carleman log-determinants $\log\dettwo(I+A)$, $\log\dettwo[(I+A)^{-1}]$, which are of interest in their own right. We thus state their properties  explicitly in the following
\begin{theorem}
	\label{theorem:logdet-Hilbert-Carleman-bound}
	Let $A,B \in \Omega$, with $\Omega$ as defined in Eq.\eqref{equation:omega-set}. Then
	\begin{align}
		|\log\dettwo(I+A)| &= -\log\dettwo(I+A) \leq ||(I+A)^{-1}||\;||A||^2_{\HS},
		\\
		|\log\dettwo[(I+A)^{-1}]| &= - \log\dettwo[(I+A)^{-1}] \leq ||(I+A)^{-1}||\;||A||^2_{\HS}.
	\end{align}
	\begin{align}
		&|\log\dettwo(I+A) - \log\dettwo(I+B)|
		\nonumber 
		\\
		&\leq ||(I+A)^{-1}||\;||(I+B)^{-1}||\;[||A||_{\HS}+ ||B||_{\HS} + ||A||_{\HS}||B||_{\HS}]||A-B||_{\HS},
		\\
		&|\log\dettwo[(I+A)^{-1}] - \log\dettwo[(I+B)^{-1}]|
		\nonumber 
		\\
		&\leq ||(I+A)^{-1}||\;||(I+B)^{-1}||\;[||A||_{\HS}+ ||B||_{\HS} + ||A||_{\HS}||B||_{\HS}]||A-B||_{\HS}.
	\end{align}
	Furthermore, for $A,B \in \Sym^{+}(\H) \cap \HS(\H)$,
	\begin{align}
		&|\log\dettwo(I+A)| = -\log\dettwo(I+A) \leq \frac{1}{2}||A||^2_{\HS},
		\\
		&|\log\dettwo[(I+A)^{-1}]| = -\log\dettwo[(I+A)^{-1}] \leq \frac{1}{2}||A||^2_{\HS},
		\\
		&|\log\dettwo(I+A) - \log\dettwo(I+B)| \leq \frac{1}{2}[||A||_{\HS} + ||B||_{\HS}]||A-B||_{\HS},
		\\
		&|\log\dettwo[(I+A)^{-1}] - \log\dettwo[(I+B)^{-1}]| \leq \frac{1}{2}[||A||_{\HS} + ||B||_{\HS}]||A-B||_{\HS}.
	\end{align}
\end{theorem}

\begin{remark}
	For $A \in \HS(\H)$, Theorem 6.4 in \citep{Simon:1977} states that
	$|\dettwo(I+A)| \leq \exp\left(\frac{1}{2}||A||_{\HS}^2\right).$
	By the strict motonoticity of the $\log$ function, this implies
	$\log|\dettwo(I+A)| \leq \frac{1}{2}||A||^2_{\HS}$.
	For $I+A \in \PC_2(\H)$, $\dettwo(I+A)> 0$, and thus
	$\log\dettwo(I+A) \leq \frac{1}{2}||A||^2_{\HS}$. However, this inequality is trivial since in this case
	we always have $\log\dettwo(I+A) \leq 0$.
	In contrast, Theorem \ref{theorem:logdet-Hilbert-Carleman-bound} gives the non-trivial inequality
	$|\log\dettwo(I+A)| \leq \frac{1}{2}||A||^2_{\HS}$ when $A \in \Sym^{+}(\H) \cap \HS(\H)$.
\end{remark}

As a consequence of Proposition \ref{proposition:logdet-F-beta-representation} and Theorem
\ref{theorem:F-A-continuity-HS}, we obtain the following result, which shows
that $d^{\alpha}_{\logdet}[(A+\gamma I), (B+\gamma I)]$ is continuous in the 
$||\;||_{\HS}$ norm. 
\begin{theorem}
	[\textbf{Continuity of Alpha Log-Det divergences in Hilbert-Schmidt norm}]
	\label{theorem:convergence-logdet-infinite}
	Let $-1 \leq \alpha \leq 1$ be fixed. For $(A+\gamma I), (B+\gamma I) \in \PC_2(\H)$,
	\begin{align}
		&d^{\alpha}_{\logdet}[(A+\gamma I),(B+\gamma I)] \leq (\gamma +||B||)||(A+\gamma I)^{-1}||\;||(B+\gamma I)^{-1}||\;||A-B||^2_{\HS}.
	\end{align}
	In particular, if $A,B \in \Sym^{+}(\H) \cap \HS(\H)$, then $\forall \gamma \in \R, \gamma > 0$,
	\begin{align}
		&d^{\alpha}_{\logdet}[(A+\gamma I),(B+\gamma I)] \leq \frac{\gamma + ||B||}{\gamma^2}||A-B||^2_{\HS}.
	\end{align}
\end{theorem}
Similarly, based on Proposition \ref{proposition:logdet-F-beta-representation} and Theorem
\ref{theorem:F-A-continuity-HS}, we obtain the following result, which shows
that $d^{\alpha}_{\logdet}[(A+\gamma I), (B+\gamma I)]$ can be approximated
by a sequence $d^{\alpha}_{\logdet}[(A_n +\gamma I), (B_n +\gamma I)]$, 
where $\lim\limits_{n \approach \infty}||A_n -A||_{\HS} = \lim\limits_{n \approach\infty}||B_n - B||_{\HS} = 0$.
\begin{theorem}
	[\textbf{Approximation of Alpha Log-Det divergences}]
	\label{theorem:logdet-approx-infinite-sequence}
	Let $\gamma \in \R, \gamma > 0$ be fixed.
	Let $-1 \leq \alpha \leq 1$ be fixed.
	Let $\{A_n\}_{n \in \Nbb},\{B_n\}_{n \in \Nbb}, A,B \in \Sym^{+}(\H) \cap \HS(\H)$. Then
	\begin{align}
		&|d^{\alpha}_{\logdet}[(\gamma I+A_n), (\gamma I+B_n)] - d^{\alpha}_{\logdet}[(\gamma I+A), (\gamma I+B)]|
		\nonumber
		\\
		&\leq \frac{1}{\gamma^2}\left(1+\frac{1}{\gamma}||B_n||\right)\left(1+\frac{1}{\gamma}||B||\right) \left[||A_n -A||_{\HS} + \left(1+\frac{1}{2\gamma}||A_n|| + \frac{1}{2\gamma}||A||\right)||B_n - B||_{\HS}\right]
		\\
		&\quad \times \left[||A_n||_{\HS}+||B_n||_{\HS} + ||A||_{\HS} + ||B||_{\HS} + \frac{1}{\gamma}(||A_n||_{\HS} + ||B_n||_{\HS})(||A||_{\HS} + ||B||_{\HS})\right].
		\nonumber
	\end{align}
	Thus if $\lim\limits_{n \approach \infty}||A_n - A||_{\HS} = \lim\limits_{n \approach \infty}||B_n - B||_{\HS} = 0$, then
	\begin{align}
		\lim_{n\approach \infty}d^{\alpha}_{\logdet}[(\gamma I+A_n), (\gamma I+B_n)] = d^{\alpha}_{\logdet}[(\gamma I+A), (\gamma I+B)].
	\end{align}
\end{theorem}

\section{Log-Det divergences in RKHS}
\label{section:RKHS}
We now consider the Log-Det divergences between two RKHS covariance operators
induced by {\it two Borel probability measures} via {\it one positive definite kernel} on the same metric space.
In this setting, the Log-Det divergences admit closed form expressions in terms of the corresponding kernel Gram matrices.
This setting has been applied practically, see e.g. \citep{MinhSB:NIPS2014,Covariance:CVPR2014} for computer vision applications.
Throughout this section, we assume the following
\begin{enumerate}
	\item {\bf (A1) } {$\Xcal$} is a complete separable metric space.
	\item {\bf (A2) }{$K: \Xcal \times \Xcal \mapto \R$} is a continuous positive definite kernel.
	\item {\bf (A3)} $\rho, \rho_1,\rho_2$ are Borel probability measures on $\Xcal$, such that
	\begin{align}
		\int_{\Xcal}K(x,x)d\rho(x) < \infty, \;\;\; \int_{\Xcal}K(x,x)d\rho_j(x) < \infty, \;\; j=1,2.
	\end{align}
\end{enumerate}
Let $\H_K$ be the reproducing kernel Hilbert space (RKHS) induced by {$K$}, then $\H_K$ is separable (\citep{Steinwart:SVM2008}, Lemma 4.33).
Let {$\Phi: \Xcal \mapto \H_K$} be the corresponding canonical feature map
\begin{align}
	\label{equation:feat-map}
	\Phi(x) = K_x, \;\;\text{where  $K_x:\Xcal \mapto \R$ is defined by } K_x(y) = K(x,y) \forall y \in \Xcal.
\end{align}
Then $K(x,y) = \la \Phi(x), \Phi(y)\ra_{\H_K}$ $\forall (x,y) \in \Xcal \times \Xcal$
and the probability measure $\rho$ satisfies 
\begin{align}
	\int_{\Xcal}||\Phi(x)||_{\H_K}^2d\rho(x) = \int_{\Xcal}K(x,x)d\rho(x) < \infty.
\end{align}
The following RKHS mean vector $\mu_{\Phi} \in \H_K$ and RKHS covariance operators {$L_K, C_{\Phi}:\H_K \mapto \H_K$} induced by the feature map $\Phi$ are well-defined and are given by
\begin{align}
	\mu_{\Phi} &= \mu_{\Phi,\rho} = \int_{\Xcal}\Phi(x)d\rho(x) \in \H_K, \;\;\;
	\\
	L_K &= L_{K, \rho} = \int_{\Xcal}\Phi(x) \otimes \Phi(x)d\rho(x) = \int_{\Xcal}K_x \otimes K_x d\rho(x),
	\\
	C_{\Phi} &= C_{\Phi,\rho} = \int_{\Xcal}(\Phi(x)-\mu_{\Phi})\otimes (\Phi(x)-\mu_{\Phi})d\rho(x) = L_K - \mu_{\Phi} \otimes \mu_{\Phi}.
\end{align}
Both $L_K$ and $C_{\Phi}$ are positive, trace class operators on $\H_K$.
Let {$\Xbf =[x_1, \ldots, x_m]$,$m \in \N$,} be a data matrix randomly sampled from {$\Xcal$} according to
$\rho$, where {$m \in \Nbb$} is the number of observations.
The feature map {$\Phi$} on {$\Xbf$} 
defines
the bounded linear operator
$\Phi(\Xbf): \R^m \mapto \H_K, \Phi(\Xbf)\b = \sum_{j=1}^mb_j\Phi(x_j) , \b \in \R^m$.
The corresponding RKHS empirical mean vector and RKHS covariance operators for {$\Phi(\Xbf)$}
are defined to be
\begin{align}
	\mu_{\Phi(\Xbf)} &= \frac{1}{m}\sum_{j=1}^m\Phi(x_j) = \frac{1}{m}\Phi(\Xbf)\1_m,
	\\
	L_{K,\Xbf} & = \frac{1}{m}\sum_{j=1}^m\Phi(x_j)\otimes \Phi(x_j) = \frac{1}{m}\sum_{j=1}^m K_{x_j} 
	\otimes K_{x_j} =  \frac{1}{m}\Phi(\Xbf)\Phi(\Xbf)^{*}: \H_K \mapto \H_K
	\\
	C_{\Phi(\Xbf)} &= \frac{1}{m}\sum_{j=1}^m(\Phi(x_j) - \mu_{\Phi(\Xbf)})\otimes(\Phi(x_j) - \mu_{\Phi(\Xbf)}) = \frac{1}{m}\Phi(\Xbf)J_m\Phi(\Xbf)^{*}: \H_K \mapto \H_K,
	\label{equation:covariance-operator}
\end{align}
where $J_m = I_m -\frac{1}{m}\1_m\1_m^T,\1_m = (1, \ldots, 1)^T \in \R^m$, is the centering matrix
with $J_m^2 = J_m$.

{\bf Log-Det divergences between RKHS covariance operators}.
Let $\rho_1,\rho_2$ be two Borel probability measures on $\Xcal$ satisfying Assumption A3.
Let $\Xbf^i = (x^i_j)_{j=1}^{m_i}$, $i=1,2$, be randomly sampled from $\Xcal$ according to $\rho_i$.
We now present the closed form formulas for
the Log-Det divergences $d^{\alpha}_{\logdet}[(C_{\Phi(\Xbf^1)} + \gamma I_{\H_K}), (C_{\Phi(\Xbf^2)} + \mu I_{\H_K})]$, for $\gamma,\mu > 0$.

%

In the following, for $- 1 < \alpha < 1$, define the quantities
\begin{align}
	s(\gamma, \mu, \alpha) &= (1-\alpha)\gamma + (1+\alpha)\mu,\;
	c(\alpha, \gamma, \mu) = \frac{4\log(s(\gamma, \mu, \alpha)/2)}{1-\alpha^2} 
	- \frac{2\log\gamma}{1+\alpha} - \frac{2\log\mu}{1-\alpha}.
\end{align}

We first consider the more general form $d^{\alpha}_{\logdet}[(AA^{*} + \gamma I_{\H}), (BB^{*} + \mu I_{\H})]$,
where $A:\H_1 \mapto \H$, $B:\H_2 \mapto \H$. {\it The following results generalize those in \citep{Minh:LogDet2016},
	which assume $\H_1 = \H_2$. Furthermore, they provide expressions in terms of the Hilbert-Carleman determinant $\dettwo$, which are used subsequently for the sample complexity analysis}.
The case $\H_1 \neq \H_2$ is also necessary for the Gaussian process setting later in Sections \ref{section:estimate-Gaussian-process} and \ref{section:Gaussian-process-full}.
\begin{theorem}\label{theorem:LogDet-infinite-2}
	Let {$\H_1, \H_2,\H$} be separable Hilbert spaces. 
	Let $\gamma, \mu\in \R, \gamma > 0, \mu > 0$ be fixed.
	Let {$A:\H_1 \mapto \H$}, $B:\H_2 \mapto \H$ be compact  operators
	such that $A^{*}A \in \Tr(\H_1)$, $B^{*}B \in \Tr(\H_2)$.
	Assume that {$\dim(\H) = \infty$}. For {$-1 < \alpha < 1$},
	{$d^{\alpha}_{\logdet}[(AA^{*}+ \gamma I_{\H}),(BB^{*}+ \mu I_{\H})]$} is given by
	{
		\begin{align}
			&d^{\alpha}_{\logdet}
			[(AA^{*}+ \gamma I_{\H}),(BB^{*}+ \mu I_{\H})] 
			\\
			%
			&
			=\frac{4}{1-\alpha^2}\logdet
			\left[
			\frac{1}{s(\gamma, \mu, \alpha)}
			\begin{pmatrix}
				(1-\alpha)A^{*}A & \sqrt{1-\alpha^2}A^{*}B\\
				\sqrt{1-\alpha^2}B^{*}A & (1+\alpha)B^{*}B
			\end{pmatrix}
			+ I_{\H_1 \oplus \H_2}
			\right]
			\nonumber
			\\
			&\quad-\frac{4\beta}{1-\alpha^2}
			\logdet\left(\frac{A^{*}A}{\gamma} + I_{\H_1}\right) 
			- \frac{4(1-\beta)}{1-\alpha^2}
			\logdet\left(\frac{B^{*}B}{\mu} + I_{\H_2}\right)
			+ c(\alpha, \gamma, \mu)
			\nonumber
			\\
				&
			=\frac{4}{1-\alpha^2}\log\dettwo
			\left[
			\frac{1}{s(\gamma, \mu, \alpha)}
			\begin{pmatrix}
				(1-\alpha)A^{*}A & \sqrt{1-\alpha^2}A^{*}B\\
				\sqrt{1-\alpha^2}B^{*}A & (1+\alpha)B^{*}B
			\end{pmatrix}
			+ I_{\H_1 \oplus \H_2}
			\right]
			\nonumber
			\\
			&\quad -\frac{4\beta}{1-\alpha^2}
			\log\dettwo\left(\frac{A^{*}A}{\gamma} + I_{\H_1}\right) 
			- \frac{4(1-\beta)}{1-\alpha^2}
			\log\dettwo\left(\frac{B^{*}B}{\mu} + I_{\H_2}\right)
			+ c(\alpha, \gamma, \mu).
		\end{align}
	}
\end{theorem}
The following is the corresponding result for the case $\dim(\H) < \infty$.

\begin{theorem}\label{theorem:LogDet-finite-2}
	Let {$\H_1, \H_2,\H$} be 
	separable Hilbert spaces. 
	Let $\gamma, \mu\in \R, \gamma > 0, \mu > 0$ be fixed. Let $A:\H_1 \mapto \H$, $B:\H_1 \mapto \H$ be compact  operators
	such that $A^{*}A \in \Tr(\H_1)$, $B^{*}B\in \Tr(\H_2)$. Assume that {$\dim(\H) < \infty$}.
	For {$-1 < \alpha < 1$},
	{$d^{\alpha}_{\logdet}[(AA^{*}+ \gamma I_{\H}),(BB^{*}+ \mu I_{\H})]$} is given by
	{
		\begin{align}
			&d^{\alpha}_{\logdet}[(AA^{*}+ \gamma I_{\H}),(BB^{*}+ \mu I_{\H})]
			\\
			&=\nonumber
			\frac{4}{1-\alpha^2}\logdet
			\left[
			\frac{1}{s(\gamma, \mu, \alpha)}
			\begin{pmatrix}
				(1-\alpha)A^{*}A & \sqrt{1-\alpha^2}A^{*}B\\
				\sqrt{1-\alpha^2}B^{*}A & (1+\alpha)B^{*}B
			\end{pmatrix}
			+ I_{\H_1 \oplus \H_2}
			\right]
			\nonumber
			\\
			&-\frac{2}{1+\alpha}\logdet\left(\frac{A^{*}A}{\gamma} + I_{\H_1}\right) - \frac{2}{1-\alpha}\logdet\left(\frac{B^{*}B}{\mu} + I_{\H_2}\right)
			+ c(\alpha, \gamma, \mu)\dim(\H).
			\nonumber
		\end{align}
	}
\end{theorem}
{\bf Special case}. If $\gamma = \mu$, then $\beta = \frac{1-\alpha}{2}$, $s(\gamma, \gamma, \alpha) = 2\gamma$, and $c(\alpha, \gamma, \gamma) = 0$.
In this case,
the expressions in Theorems \ref{theorem:LogDet-infinite-2} and \ref{theorem:LogDet-finite-2} are identical.

Applying Theorems \ref{theorem:LogDet-infinite-2}, \ref{theorem:LogDet-finite-2} to $C_{\Phi(\Xbf^i)}$, $i=1,2$,
we obtain $d^{\alpha}_{\logdet}$ between RKHS covariance operators.
Define the following kernel Gram matrices $K[\Xbf^i]$, $K[\Xbf^1,\Xbf^2]$ by
\begin{align}
	K[\Xbf^i] &= \Phi(\Xbf^i)^{*}\Phi(\Xbf^i),\;
	(K[\Xbf^i])_{jk} = K(x^i_j, x^i_k), \; 1\leq j,k \leq m_i, i=1,2,
	\\
	K[\Xbf^1,\Xbf^2] &= \Phi(\Xbf^1)^{*}\Phi(\Xbf^2), \; (K[\Xbf^1,\Xbf^2])_{jk} = K(x^1_j,x^2_k), \; 1 \leq j \leq m_1, 1\leq k \leq m_2,
	\\
	\; K[\Xbf^2,\Xbf^1]&= (K[\Xbf^1,\Xbf^2])^T.
\end{align}

{\it The following results generalize those in \citep{Minh:LogDet2016}, which assume $m_1 = m_2 = m$}.
\begin{theorem}
	(\textbf{Alpha Log-Det divergences between RKHS covariance operators, $-1 < \alpha < 1$, infinite-dimensional case})
	\label{theorem:LogDet-RKHS-infinite} 
	Assume Assumptions A1-A3.
	Assume that {$\dim(\H_K) = \infty$}. Let {$\gamma > 0$, $\mu > 0$}.
	For {$-1 < \alpha < 1$},
	{$d_{\logdet}^{\alpha}[(C_{\Phi(\Xbf^1)} + \gamma I_{\H_K}), (C_{\Phi(\Xbf^2)} + \mu I_{\H_K})]$}  is given by
	%
	{
		\small
		\begin{align}
			\label{equation:d-alpha-logdet-RKHS-infinite}
			&d_{\logdet}^{\alpha}[(C_{\Phi(\Xbf^1)} + \gamma I_{\H_K}), (C_{\Phi(\Xbf^2)} + \mu I_{\H_K})]
			\\
			&=\frac{4}{1-\alpha^2}
			\logdet
			\left[
			{\frac{1}{
					s(\gamma, \mu, \alpha)
				}
				\begin{pmatrix}
					\frac{(1-\alpha)J_{m_1}K[\Xbf^1]J_{m_1}}{m_1} & \frac{\sqrt{1-\alpha^2}J_{m_1}K[\Xbf^1,\Xbf^2]J_{m_2}}{\sqrt{m_1m_2}}\\
					\frac{\sqrt{1-\alpha^2}J_{m_2}K[\Xbf^2,\Xbf^1]J_{m_1}}{\sqrt{m_1m_2}} & \frac{(1+\alpha)J_{m_2}K[\Xbf^2]J_{m_2}}{m_2}
				\end{pmatrix}
			}{}
			+ I_{m_1 + m_2}
			\right]
			\nonumber
			\\
			&\quad -\frac{4\beta}{1-\alpha^2}
			\logdet
			\left(\frac{J_{m_1}K[\Xbf^1]J_{m_1}}{m_1\gamma} + I_{m_1}\right)
			-\frac{4(1-\beta)}{1-\alpha^2}
			\left(\frac{J_{m_2}K[\Xbf^2]J_{m_2} }{m_2\mu}+ I_{m_2}\right) 
			+ c(\alpha,\gamma,\mu).
			\nonumber
			%
		\end{align}
	}
\end{theorem}

The corresponding formula when $\dim(\H_K) < \infty$ is given by the following result.

\begin{theorem}
	(\textbf{Alpha Log-Det divergences between RKHS covariance operators, $-1 < \alpha < 1$ - finite-dimensional case})
	\label{theorem:LogDet-RKHS-finite} 
	Assume Assumptions A1-A3.
	Assume that {$\dim(\H_K) < \infty$}. Let {$\gamma > 0$, $\mu > 0$}.
	Then for {$-1 < \alpha < 1$},
	{\small
		\begin{align}\label{equation:LogDet-RKHS-finite}
			&d_{\logdet}^{\alpha}[(C_{\Phi(\Xbf^1)} + \gamma I_{\H_K}), (C_{\Phi(\Xbf^2)} + \mu I_{\H_K})]  \nonumber
			\\
			&=\frac{4}{1-\alpha^2}  \nonumber
			\logdet
			\left[
			\frac{1}{
				s(\gamma, \mu, \alpha)}
			\begin{pmatrix}
				\frac{(1-\alpha)J_{m_1}K[\Xbf^1]J_{m_1}}{m_1} & \frac{\sqrt{1-\alpha^2}J_{m_1}K[\Xbf^1,\Xbf^2]J_{m_2}}{\sqrt{m_1m_2}}\\
				\frac{\sqrt{1-\alpha^2}J_{m_2}K[\Xbf^2,\Xbf^1]J_{m_1}}{\sqrt{m_1m_2}} & \frac{(1+\alpha)J_{m_2}K[\Xbf^2]J_{m_2}}{m_2}
			\end{pmatrix}
			+ I_{m_1 + m_2}
			\right]
			\nonumber
			\\
			&\quad -\frac{2}{1+\alpha}\log\det\left(\frac{J_{m_1}K[\x]J_{m_1}}{m_1\gamma} + I_{m_1}\right) \nonumber
			-\frac{2}{1-\alpha}\logdet\left(\frac{J_{m_2}K[\Xbf^2]J_{m_2}}{m_2\mu} + I_{m_2}\right) \nonumber
			\\
			&\quad + c(\alpha, \gamma, \mu)\dim(\H_K).
			%
		\end{align}
	}
\end{theorem}
Note that the formulas in Theorems \ref{theorem:LogDet-RKHS-infinite} and \ref{theorem:LogDet-RKHS-finite} coincide when $\gamma = \mu$.
Consider the limiting cases $\alpha = \pm 1$. It suffices to consider $\alpha =1$, since
$d^{-1}[(A+\gamma I), (B+\mu I)] = d^{1}_{\logdet}[(B+\mu I), (A+\gamma I)]$.
The following is the infinite-dimensional version
\begin{theorem}
	\label{theorem:logdet-alpha-1-AAstar-BBstar-representation-switch}
	Let $\H_1, \H_2, \H$ be separable Hilbert spaces.
	Assume that $\dim(\H) = \infty$.
	Let $\gamma,\mu \in \R, \gamma > 0, \mu > 0$ be fixed.
	Let $A: \H_1 \mapto \H$, $B: \H_2 \mapto \H$ be compact, such that $A^{*}A\in \Tr(\H_1)$, $B^{*}B \in \Tr(\H_2)$, then $AA^{*}, BB^{*}\in \Tr(\H)$, and
	\begin{align}
		&d^1_{\logdet}[(AA^{*} + \gamma I_{\H}), (BB^{*} + \mu I_{\H})]
		\nonumber 
		\\
		&=\left(\frac{\gamma}{\mu}-1 - \log\frac{\gamma}{\mu}\right) +\frac{\gamma}{\mu}\trace\left[\frac{A^{*}A}{\gamma}\right] - \frac{\gamma}{\mu}\trace\left[\left(\frac{B^{*}B}{\mu} + \frac{B^{*}AA^{*}B}{\gamma\mu}\right)\left(I+\frac{B^{*}B}{\mu}\right)^{-1}\right]
		\\
		&\quad +\frac{\gamma}{\mu}\log\det\left(\frac{B^{*}B}{\mu}+I_{\H_2}\right) 
		-\frac{\gamma}{\mu} \log\det\left(\frac{A^{*}A}{\gamma}+I_{\H_1}\right)
		\nonumber
		\\
		&= \left(\frac{\gamma}{\mu}-1 - \log\frac{\gamma}{\mu}\right) - \frac{\gamma}{\mu}\log\dettwo\left(I_{\H_1}+\frac{A^{*}A}{\gamma}\right) 
		\\
		&\quad  - \frac{\gamma}{\mu}\log\dettwo\left[\left(I_{\H_2}+\frac{B^{*}B}{\mu}\right)^{-1}\right] 
		- \trace\left[\frac{B^{*}AA^{*}B}{\gamma\mu}\left(I_{\H_2}+\frac{B^{*}B}{\mu}\right)^{-1}\right].
		\nonumber
	\end{align}
\end{theorem}

The following is the corresponding finite-dimensional version
\begin{theorem}
	\label{theorem:logdet-alpha-1-AAstar-BBstar-representation-switch-finite}
	Let $\H_1, \H_2, \H$ be separable Hilbert spaces.
	Assume that $\dim(\H) < \infty$.
	Let $\gamma,\mu \in \R, \gamma > 0, \mu > 0$ be fixed.
	Let $A: \H_1 \mapto \H$, $B: \H_2 \mapto \H$ be compact, such that $A^{*}A\in \Tr(\H_1)$, $B^{*}B \in \Tr(\H_2)$, then $AA^{*}, BB^{*}\in \Tr(\H)$, and
	\begin{align}
		&d^1_{\logdet}[(AA^{*} + \gamma I_{\H}), (BB^{*} + \mu I_{\H})]	
		= \left(\frac{\gamma}{\mu}-1 - \log\frac{\gamma}{\mu}\right)\dim(\H)
		\\
		&\quad + \trace\left[\frac{A^{*}A}{\mu}\right] - \frac{\gamma}{\mu}\trace\left[\left(\frac{B^{*}B}{\mu} + \frac{B^{*}AA^{*}B}{\gamma \mu}\right)\left(I_{\H_2}+\frac{B^{*}B}{\mu}\right)^{-1}\right]
		\nonumber
		\\
		&\quad+ \log\det\left(\frac{B^{*}B}{\mu} + I_{\H_2}\right)
		- \log\det\left(\frac{A^{*}A}{\gamma} + I_{\H_1}\right).
		\nonumber 	
	\end{align}
\end{theorem}

The formulas for $d^1_{\logdet}$ in Theorems \ref{theorem:logdet-alpha-1-AAstar-BBstar-representation-switch} and 
\ref{theorem:logdet-alpha-1-AAstar-BBstar-representation-switch-finite} coincide when $\gamma = \mu$.
Applying Theorems \ref{theorem:logdet-alpha-1-AAstar-BBstar-representation-switch},
\ref{theorem:logdet-alpha-1-AAstar-BBstar-representation-switch-finite} to $C_{\Phi(\Xbf^i)}$,
$i=1,2$, we obtain
$d^1_{\logdet}$ between RKHS covariance operators, as follows.
\begin{theorem}
	(\textbf{Alpha Log-Det divergences between RKHS covariance operators, $\alpha=1$ - infinite-dimensional case})
	\label{theorem:LogDet-RKHS-alpha=1-infinite}
	Assume Assumptions A1-A3.
	Let $\gamma, \mu \in \R, \gamma >0, \mu > 0$ be fixed. Assume that $\dim(\H_K) = \infty$. Then
	{\small
		\begin{align}
			\label{equation:d1-logdet-RKHS}
			&d^1_{\logdet}[(C_{\Phi(\Xbf^1)} + \gamma I_{\H_K}), (C_{\Phi(\Xbf^2)} + \mu I_{\H_K})] 
			=\left(\frac{\gamma}{\mu}-1 - \log\frac{\gamma}{\mu}\right) +\frac{\gamma}{\mu}\trace\left[\frac{J_{m_1}K[\Xbf^1]J_{m_1}}{m_1\gamma}\right]
			\\
			& - \frac{\gamma}{\mu}\trace\left[\left(\frac{J_{m_2}K[\Xbf^2]J_{m_2}}{m_2\mu} + \frac{J_{m_2}K[\Xbf^2,\Xbf^1]J_{m_1}K[\Xbf^1,\Xbf^2]J_{m_2}}{m_1m_2\gamma\mu}\right)\left(I_{m_2}+\frac{J_{m_2}K[\Xbf^2]J_{m_2}}{m_2\mu}\right)^{-1}\right]
			\nonumber
			\\
			&\quad +\frac{\gamma}{\mu}\log\det\left(\frac{J_{m_2}K[\Xbf^2]J_{m_2}}{m_2\mu}+I_{m_2}\right) 
			-\frac{\gamma}{\mu} \log\det\left(\frac{J_{m_1}K[\Xbf^1]J_{m_1}}{m_1\gamma}+I_{m_1}\right)
			\nonumber
			\\
			&= \left(\frac{\gamma}{\mu}-1 - \log\frac{\gamma}{\mu}\right)
			\\
			& \quad - \frac{\gamma}{\mu}\log\dettwo\left(I_{m_1}+\frac{J_{m_1}K[\Xbf^1]J_{m_1}}{m_1\gamma}\right) 
			-\frac{\gamma}{\mu}\log\dettwo\left[\left(I_{m_2}+\frac{J_{m_2}K[\Xbf^2]J_{m_2}}{m_2\mu}\right)^{-1}\right] 
			\nonumber
			\\
			&\quad - \trace\left[\frac{J_{m_2}K[\Xbf^2,\Xbf^1]J_{m_1}K[\Xbf^1,\Xbf^2]J_{m_2}}{m_1m_2\gamma\mu}\left(I_{m_2}+\frac{J_{m_2}K[\Xbf^2]J_{m_2}}{m_2\mu}\right)^{-1}\right].
			\nonumber
		\end{align}
	}	
\end{theorem}
The following is the corresponding version when $\dim(\H_K) < \infty$.
\begin{theorem}
	(\textbf{Alpha Log-Det divergences between RKHS covariance operators, $\alpha=1$ - finite-dimensional case})
	\label{theorem:LogDet-RKHS-alpha=1-finite}
	Assume Assumptions A1-A3.
	Let $\gamma, \mu \in \R, \gamma >0, \mu > 0$ be fixed. Assume that $\dim(\H_K) < \infty$. Then
	{\small
		\begin{align}
			&d^1_{\logdet}[(C_{\Phi(\Xbf^1)} + \gamma I_{\H_K}), (C_{\Phi(\Xbf^2)} + \mu I_{\H_K})] 
			\nonumber
			\\
			&= \left(\frac{\gamma}{\mu}-1 - \log\frac{\gamma}{\mu}\right)\dim(\H_K) +
			\trace\left[\frac{J_{m_1}K[\Xbf^1]J_{m_1}}{m_1\mu}\right] 
			\\
			& \quad - \frac{\gamma}{\mu}\trace\left[\left(\frac{J_{m_2}K[\Xbf^2]J_{m_2}}{m_2\mu} + \frac{J_{m_2}K[\Xbf^2,\Xbf^1]J_{m_1}K[\Xbf^1, \Xbf^2]J_{m_2}}{m_1m_2\gamma \mu}\right)\left(I_{m_2}+\frac{J_{m_2}K[\Xbf^2]J_{m_2}}{m_2\mu}\right)^{-1}\right]
			\nonumber
			\\
			&\quad + \log\det\left(\frac{J_{m_2}K[\Xbf^2]J_{m_2}}{m_2\mu} + I_{m_2}\right)
			- \log\det\left(\frac{J_{m_1}K[\Xbf^1]J_{m_1}}{m_1\gamma} + I_{m_1}\right).
			\nonumber 	
		\end{align}
	}
\end{theorem}

\subsection{Consistency of Alpha Log-Det divergences in the RKHS setting}
\label{section:estimate-RKHS}
We recall that $C_{\Phi(\Xbf^i)}$ is the empirical version of $C_{\Phi,\rho_i}$, $i=1,2$.
In general, since $\rho_1,\rho_2$ are unknown, $d^{\alpha}_{\logdet}[(C_{\Phi,\rho_1} + \gamma I_{\H_K}), (C_{\Phi, \rho_2}+\gamma I_{\H_K})]$ is unknown. 
On the other hand, from the sample data $\Xbf^1,\Xbf^2$,
we have explicit expressions for 
$d^{\alpha}_{\logdet}[(C_{\Phi(\Xbf^1),\rho_1} + \gamma I_{\H_K}), (C_{\Phi(\Xbf^2)}+\gamma I_{\H_K})]$.
In this section, we show that
\begin{align*}
	\text{ As $m_1, m_2 \approach \infty$,  }\;\;\;&d^{\alpha}_{\logdet}[(C_{\Phi(\Xbf^1),\rho_1} + \gamma I_{\H_K}), (C_{\Phi(\Xbf^2)}+\gamma I_{\H_K})]
	\\
	\text{ consistently estimates  }\;\;\;
	&d^{\alpha}_{\logdet}[(C_{\Phi,\rho_1} + \gamma I_{\H_K}), (C_{\Phi, \rho_2}+\gamma I_{\H_K})]
\end{align*} 
with dimension-independent sample complexities.
For simplicity, we assume the following
\begin{align}
\text{(Assumption \textbf{A4}) }\;\;\;\sup_{x \in \Xcal}K(x,x) \leq \kappa^2.
\end{align}
The convergence of $\mu_{\Phi(\Xbf)}$ and $C_{\Phi(\Xbf)}$ towards $\mu_{\Phi}$ and $C_{\Phi}$, respectively, is 
given 
by the following
\begin{theorem}
	(\textbf{Convergence of RKHS mean and covariance operators - bounded kernels \citep{Minh:2021EntropicConvergenceGaussianMeasures}})
	\label{theorem:CPhi-concentration}
	Assume Assumptions A1-A4.
	Let $\Xbf = (x_i)_{i=1}^m$, $m \in \Nbb$, be independently  sampled from $(\Xcal, \rho)$. Then
	$||\mu_{\Phi}||_{\H_K} \leq \kappa$, $||\mu_{\Phi(\Xbf)}||_{\H_K} \leq \kappa$ $\forall \Xbf \in \Xcal^m$,
	$||C_{\Phi}||_{\HS(\H_K)} \leq 2\kappa^2$,
	$||C_{\Phi(\Xbf)}||_{\HS(\H_K)}  \leq 2\kappa^2$ $\forall \Xbf \in \Xcal^m$.
	For any $0 < \delta <1$, with probability at least $1-\delta$, 
	{\small
		\begin{align}
			||\mu_{\Phi(\Xbf)} - \mu_{\Phi}||_{\H_K} & \leq \kappa\left(\frac{2\log\frac{4}{\delta}}{m} + \sqrt{\frac{2\log\frac{4}{\delta}}{m}}\right),\;\;
			\\
			||C_{\Phi(\Xbf)} - C_{\Phi}||_{\HS(\H_K)} &\leq 3\kappa^2\left(\frac{2\log\frac{4}{\delta}}{m} + \sqrt{\frac{2\log\frac{4}{\delta}}{m}}\right).
		\end{align}
	}
\end{theorem}

Combining Theorem \ref{theorem:CPhi-concentration} with Theorem \ref{theorem:logdet-approx-infinite-sequence}, we obtain the following finite sample estimate of the Alpha Log-Det divergences between RKHS covariance operators

\begin{theorem}
	(\textbf{Estimation of Alpha Log-Det divergences between RKHS covariance operators})
	\label{theorem:logdet-approximate-RKHS-infinite}
	Assume Assumptions A1-A4. 
	Let $-1 \leq \alpha \leq 1$ be fixed.
	Let $\gamma \in \R, \gamma > 0$ be fixed.
	Let $\Xbf^i = (x^i_j)_{j=1}^{m_i}$, $i=1,2$, be independently sampled
	from $(\Xcal, \rho_i)$.
	For any $0 < \delta < 1$, with probability at least $1-\delta$,
	{\small
		\begin{align}
			&\left|d^{\alpha}_{\logdet}[(C_{\Phi(\Xbf^1)} + \gamma I_{\H_K}), (C_{\Phi(\Xbf^2)} + \gamma I_{\H_K}) ] 
			-d^{\alpha}_{\logdet}[(C_{\Phi,\rho_1} + \gamma I_{\H_K}), (C_{\Phi, \rho_2} + \gamma I_{\H_K})]
			\right|
			\nonumber
			\\
			&\leq \frac{12\kappa^4}{\gamma^2}\left(1+ \frac{2\kappa^2}{\gamma}\right)^2\left(2 +  \frac{\kappa^2}{\gamma}\right)
			\left[\left(\frac{2\log\frac{8}{\delta}}{m_1} + \sqrt{\frac{2\log\frac{8}{\delta}}{m_1}}\right)
			+ \left(1+\frac{2\kappa^2}{\gamma}\right)\left(\frac{2\log\frac{8}{\delta}}{m_2} + \sqrt{\frac{2\log\frac{8}{\delta}}{m_2}}\right)
			\right].
		\end{align}
	}
\end{theorem}
If $\kappa$ is an absolute constant, e.g. for exponential kernels such as Gaussian kernels,
then the convergence in Theorem \ref{theorem:logdet-approximate-RKHS-infinite} is completely {\it dimension-independent}.

\section{Regularized Kullback-Leibler and R\'enyi divergences in RKHS}
\label{section:renyi-RKHS-infinite}

Let us now apply the formulation of regularized R\'enyi and KL divergences stated in Definition 
\ref{definition:renyi-regularized} to the RKHS setting. Assume Assumptions A1-A3.
Two Borel probability measures $\rho_1,\rho_2$ on $\Xcal$ satisfying A3 give rise to two corresponding Gaussian measures on $\H_K$, 
namely $\Ncal(\mu_{\Phi,\rho_1}, C_{\Phi,\rho_1})$, $\Ncal(\mu_{\Phi,\rho_2}, C_{\Phi,\rho_2})$. The following shows that, for characteristic kernels, for any $\gamma \in \R, \gamma > 0$, the divergence
$D^{\gamma}_{\Rrm,r}[\Ncal(\mu_{\Phi,\rho_1}, C_{\Phi,\rho_1})||\Ncal(\mu_{\Phi,\rho_2}, C_{\Phi,\rho_2})]$ is a divergence between $\rho_1$ and $\rho_2$.

\begin{theorem}
	\label{theorem:divergence-RKHS-characteristic}
	Assume Assumptions A1-A3. Assume further that $K$ is characteristic. Let $\gamma \in \R, \gamma > 0$ be fixed. Let $0 \leq r \leq 1$ be fixed. Then
	\begin{align}
		D^{\gamma}_{\Rrm,r}[\Ncal(\mu_{\Phi,\rho_1}, C_{\Phi,\rho_1})||\Ncal(\mu_{\Phi,\rho_2}, C_{\Phi,\rho_2})]&\geq 0,
		\\
		D^{\gamma}_{\Rrm,r}[\Ncal(\mu_{\Phi,\rho_1}, C_{\Phi,\rho_1})||\Ncal(\mu_{\Phi,\rho_2}, C_{\Phi,\rho_2})] &= 0 \equivalent \rho_1 = \rho_2.
	\end{align}
\end{theorem}

Let $\Xbf^i=(x^i_j)_{j=1}^{m_i}$ be sampled from $\Xcal$ according to $\rho_i$, $i=1,2$. 
Using the Alpha Log-Det divergence formulas in RKHS, we show that
the quantity $D_{\Rrm,r}^{\gamma}[\Ncal(\mu_{\Phi(\Xbf^1)}, C_{\Phi(\Xbf^1)})|| \Ncal(\mu_{\Phi(\Xbf^2)}, C_{\Phi(\Xbf^2)})]$ admits a closed form formula in terms of the kernel Gram matrices, as follows.
\begin{theorem}
	[\textbf{Empirical R\'enyi divergences in RKHS}]
	\label{theorem:Renyi-RKHS}
	Assume Assumptions A1-A3. Let $\gamma \in \R, \gamma > 0$ be fixed. Let $0\leq r \leq 1$ be fixed. Then
	\begin{align}
		\label{equation:Renyi-RKHS}
		D_{\Rrm,r}^{\gamma}[\Ncal(\mu_{\Phi(\Xbf^1)}, C_{\Phi(\Xbf^1)})|| \Ncal(\mu_{\Phi(\Xbf^2)}, C_{\Phi(\Xbf^2)})]
		= \frac{1}{2\gamma}S_1 + \frac{1}{2}S_2.
	\end{align}
	Here $S_2 = d^{2r-1}_{\logdet}[(C_{\Phi(\Xbf^1)} + \gamma I_{\H_K}), (C_{\Phi(\Xbf^2)} + \gamma I_{\H_K})]$
	is as given in Theorems \ref{theorem:LogDet-RKHS-infinite} and \ref{theorem:LogDet-RKHS-finite} (note that they coincide when $\gamma =\mu$). 
	The quantity $S_1$ is given by
	{\small
		\begin{align}
			&S_1 = \frac{1}{m_1^2}\1_{m_1}^TK[\Xbf^1]\1_{m_1} + \frac{1}{m_2^2}\1_{m^2}^TK[\Xbf^2]\1_{m_2} -\frac{2}{m_1m_2}\1_{m_1}^TK[\Xbf^1,\Xbf^2]\1_{m_2}
			\\
			&-
			\begin{pmatrix}
				\v_1
				\\
				\v_2
			\end{pmatrix}^T
			\left[\gamma I_{m_1+m_2} + 
			\begin{pmatrix}
				\frac{1-r}{m_1}J_{m_1}K[\Xbf^1]J_{m_1} & \sqrt{\frac{r(1-r)}{m_1m_2}}J_{m_1}K[\Xbf^1,\Xbf^2]J_{m_2}
				\\
				\sqrt{\frac{r(1-r)}{m_1m_2}}J_{m_2}K[\Xbf^2,\Xbf^1]J_{m_1} &\frac{r}{m_2}J_{m_2}K[\Xbf^2]J_{m_2}
			\end{pmatrix}
			\right]^{-1}
			\begin{pmatrix}
				\v_1
				\\
				\v_2
			\end{pmatrix}
			.
			\nonumber
		\end{align}
	}
	The vectors $\v_1 \in \R^{m_1}, \v_2 \in \R^{m_2}$ are given by
	{\small
		\begin{equation}
			\begin{aligned}
				\v_1 = \sqrt{\frac{1-r}{m_1}}J_{m_1}
				\left(\frac{1}{m_1}K[\Xbf^1]\1_{m_1}  - \frac{1}{m_2}K[\Xbf^1,\Xbf^2]\1_{m_2}\right),
				\\
				\v_2 = \sqrt{\frac{r}{m_2}}J_{m_2}
				\left(\frac{1}{m_1}K[\Xbf^2,\Xbf^1]\1_{m_1}  - \frac{1}{m_2}K[\Xbf^2]\1_{m_2}\right).
			\end{aligned}
		\end{equation}
	}
\end{theorem}

The quantity $S_1$ in Theorem \ref{theorem:Renyi-RKHS} reduces to a much simpler form
when $r=1$, in which case we obtain the closed form formula for the KL divergence in RKHS, as follows.
\begin{theorem}
	[\textbf{Empirical KL divergence in RKHS}]
	\label{theorem:KL-RKHS}
	Assume Assumptions A1-A3. Let $\gamma \in \R, \gamma > 0$ be fixed. Then
	\begin{align}
		\label{equation:KL-RKHS}
		\KL^{\gamma}[\Ncal(\mu_{\Phi(\Xbf^1)}, C_{\Phi(\Xbf^1)})|| \Ncal(\mu_{\Phi(\Xbf^2)}, C_{\Phi(\Xbf^2)})]
		= \frac{1}{2\gamma}S_1 + \frac{1}{2}S_2.
	\end{align}
	Here $S_2 = d^{1}_{\logdet}[(C_{\Phi(\Xbf^1)} + \gamma I_{\H_K}), (C_{\Phi(\Xbf^2)} + \gamma I_{\H_K})]$
	is as given in Theorems \ref{theorem:LogDet-RKHS-alpha=1-infinite} and \ref{theorem:LogDet-RKHS-alpha=1-finite}
	(note that they coincide when $\gamma =\mu$). The quantity $S_1$ is given by
	\begin{align}
		S_1 &= \frac{1}{m_1^2}\1_{m_1}^TK[\Xbf^1]\1_{m_1} + \frac{1}{m_2^2}\1_{m^2}^TK[\Xbf^2]\1_{m_2} -\frac{2}{m_1m_2}\1_{m_1}^TK[\Xbf^1,\Xbf^2]\1_{m_2}
		\\
		&\quad- \v_2^T\left(\gamma I_{m_2} + \frac{1}{m_2}J_{m_2}K[\Xbf^2]J_{m_2}\right)^{-1}\v_2.
		\nonumber
	\end{align}
	Here $\v_2 \in \R^{m_2}$ is given by $\v_2 = \sqrt{\frac{1}{m_2}}J_{m_2}
	\left(\frac{1}{m_1}K[\Xbf^2,\Xbf^1]\1_{m_1}  - \frac{1}{m_2}K[\Xbf^2]\1_{m_2}\right)$.
\end{theorem}

From Theorems \ref{theorem:logdet-approximate-RKHS-infinite}, \ref{theorem:Renyi-RKHS}, and \ref{theorem:KL-RKHS}, we obtain the following sample complexity for the empirical estimates of 
the R\'enyi divergences (with $r=1$ corresponding to the KL divergence) 
\begin{theorem}
	[\textbf{Estimation of R\'enyi and KL divergences in RKHS}]
	\label{theorem:Renyi-approximate-RKHS-infinite}
	Assume Assumptions A1-A4.
	Let $0 \leq r \leq 1$ be fixed.
	Let $\gamma \in \R, \gamma > 0$ be fixed.
	Let $\Xbf^i = (x^i_j)_{j=1}^{m_i}$ be independently sampled
	from $(\Xcal, \rho)$. For any $0 < \delta < 1$, with probability at least $1-\delta$,
	{\small
		\begin{align}
			&\left|D_{\Rrm,r}^{\gamma}[\Ncal(\mu_{\Phi(\Xbf^1)}, C_{\Phi(\Xbf^1)})|| \Ncal(\mu_{\Phi(\Xbf^2)}, C_{\Phi(\Xbf^2)})] - D_{\Rrm,r}^{\gamma}[\Ncal(\mu_{\Phi,\rho_1}, C_{\Phi, \rho_1})|| \Ncal(\mu_{\Phi,\rho_2}, C_{\Phi, \rho_2})]\right|
			\nonumber
			\\
			& \leq C_1(\kappa,\gamma,r)  \left(\frac{2\log\frac{16}{\delta}}{m_1} + \sqrt{\frac{2\log\frac{16}{\delta}}{m_1}}\right)
			+ C_2(\kappa,\gamma,r)\left(\frac{2\log\frac{16}{\delta}}{m_2} + \sqrt{\frac{2\log\frac{16}{\delta}}{m_2}}\right),
		\end{align}
	}
where $C_1(\kappa, \gamma, r) = \frac{6\kappa^4(1-r)}{\gamma^2} + \frac{2\kappa^2}{\gamma} + \frac{6\kappa^4}{\gamma^2}\left(1+ \frac{2\kappa^2}{\gamma}\right)^2\left(2 +  \frac{\kappa^2}{\gamma}\right)$, 
$C_2(\kappa, \gamma, r) = \frac{6\kappa^4r}{\gamma^2} + \frac{2\kappa^2}{\gamma} + \frac{6\kappa^4}{\gamma^2}\left(1+ \frac{2\kappa^2}{\gamma}\right)^2\left(2 +  \frac{\kappa^2}{\gamma}\right)
\left(1+\frac{2\kappa^2}{\gamma}\right)$.
\end{theorem}
As in Theorem \ref{theorem:logdet-approximate-RKHS-infinite}, if $\kappa$ is an absolute constant,
then the sample complexity in Theorem \ref{theorem:Renyi-approximate-RKHS-infinite} is completely {\it dimension-independent}.
Theorem \ref{theorem:Renyi-approximate-RKHS-infinite} provides the theoretical basis for the following Algorithm \ref{algorithm:estimate-divergence-RKHS} for estimating regularized R\'enyi and KL divergences between Borel probability measures on $\Xcal$.

\begin{algorithm}[h]
	\caption{Consistent estimates of regularized R\'enyi and KL divergences between Borel probability measures via RKHS Gaussian measures}
	\label{algorithm:estimate-divergence-RKHS}
	\begin{algorithmic}
		\REQUIRE Kernel $K$
		\REQUIRE Sample data sets $\Xbf^i = (x^i_j)_{j=1}^{m_i}$, $i=1,2$, from two (unknown) Borel probability measures $\rho_1,\rho_2$ 
		\REQUIRE Regularization parameter $\gamma > 0$, $0< r < 1$ (for the R\'enyi divergence)
		
		\textbf{Procedure}
		\STATE{Compute $d^{2r-1}_{\logdet}$ according to Eq.\eqref{equation:d-alpha-logdet-RKHS-infinite}}
		\STATE{Compute $D^{\gamma}_{\Rrm,r}$ according to Eq. \eqref{equation:Renyi-RKHS}}
		\STATE{Compute $d^1_{\logdet}$ according to Eq.\eqref{equation:d1-logdet-RKHS}}
		\STATE{Compute $\KL^{\gamma}$ according to Eq. \eqref{equation:KL-RKHS}}
		\RETURN{$D^{\gamma}_{\Rrm,r}$ and $\KL^{\gamma}$}
	\end{algorithmic}
\end{algorithm}

{\bf Discussion of related work}. In \citep{sledge2021estimating},
the authors define a so-called {\it univariate Gram operator} for a universal kernel $K$, which is
the operator $M_K = \mu_K\otimes \mu_K:\H_K \mapto \H_K$, where $\mu_K = \int_{T}K_td\rho(t)$. 
This is a rank-one operator, with eigenvector $\mu_K$ and eigenvalue $||\mu_K||^2_{\H_K}$.
For $\alpha > 0, \alpha \neq 1$, the following quantity is defined
\begin{align*}
	C_{\alpha}(M_{K^1}||M_{K^2}) = \frac{1}{\alpha-1}\log\trace[M_{K^{1}}^{\alpha}M_{K^2}^{1-\alpha}] - \frac{1}{\alpha-1}\log\trace(M_{K^{1}}), \;\;\text{if $\supp(M_{K^2}) \subseteq \supp(M_{K^1})$} 
\end{align*}
and $C_{\alpha}(M_{K^1}||M_{K^2}) = \infty$ otherwise. 
Since the RKHS $\H_{K^1}, \H_{K^2}$ are generally different,
the product $M_{K^{1}}^{\alpha}M_{K^2}^{1-\alpha}$ is well-defined only if $\mu_{K^2} \in \H_{K^1}$.
Moreover, the quantity $C_{\alpha}(A||B)$ is generally {\it not} guaranteed to be positive, 
even in the finite-dimensional setting, unless we assume in addition that $\trace(A) = \trace(B) = 1$
(in that case it is known as the {\it quantum R\'enyi divergence} or {\it quantum R\'enyi relative entropy}, see e.g. \citep{mosonyi2011quantum}).
The following is a counter-example on $\R^3$: if 
$	A =
	\begin{pmatrix} 0.2623  &  0.4491 &   0.3993\\
		0.4491  &  0.9155 &   0.3544\\
		0.3993  &  0.3544  &  1.5180
	\end{pmatrix}$,
	$B =
	\begin{pmatrix}
		0.6968  &  0.9409 &   0.9471\\
		0.9409  &  1.6211  &  1.4961\\
		0.9471  &  1.4961  &  1.6847\\
	\end{pmatrix}
$,
then $C_{1/2}(A||B) = -0.2006$.

\section{Regularized divergences between Gaussian processes}
\label{section:estimate-Gaussian-process}
We now turn to the closely related but different setting of Gaussian processes with squared integrable sample paths.
We first consider the zero-mean processes. We show that regularized R\'enyi and KL divergences between
centered Gaussian processes are well-defined and can be consistently and efficiently estimated from their
finite-dimensional versions, all with {\it dimension-independent} sample complexities. 
Since covariance functions of stochastic processes are positive definite kernels, RKHS methodology also plays a crucial role
in the theoretical analysis here. Compared with Sections \ref{section:RKHS} and \ref{section:renyi-RKHS-infinite},
apart from RKHS covariance operators, we also need to exploit RKHS {\it cross-covariance operators}
(more detail below).

\begin{remark}
For general stochastic processes with zero-mean and squared integrable sample paths, the divergences considered here become divergences 
between their corresponding covariance operators.
\end{remark}

Throughout this section, we assume the following
\begin{enumerate}
	\item {\bf B1} $T$ is a $\sigma$-compact metric space, that is
	$T = \cup_{i=1}^{\infty}T_i$, where $T_1 \subset T_2 \subset \cdots$, with each $T_i$ being compact.
	
	\item {\bf B2} $\nu$ is a non-degenerate Borel probability measure on $T$, that is $\nu(B) > 0$
	for each open set $B \subset T$.
	\item {\bf B3} $K, K^1, K^2: T \times T \mapto \R$ are continuous, symmetric, positive definite kernels
	and $\exists \kappa > 0, \kappa_1 > 0, \kappa_2 > 0$ such that
	\begin{align}
		\label{equation:Assumption-3}
		\int_{T}K(x,x)d\nu(x) \leq \kappa^2, \;\; \int_{T}K^i(x,x)d\nu(x) \leq \kappa_i^2.
	\end{align}
	\item {\bf B4} $\xi\sim \GP(0, K)$, $\xi^i \sim \GP(0, K^i)$, $i=1,2$, are {\it centered} Gaussian processes
	with covariance functions $K, K^i$, respectively.
\end{enumerate}
Let $\H_K$ denote the reproducing kernel Hilbert space (RKHS) of functions on $T$ induced by $K$.
Let $K_x:T \mapto \R$ be defined by $K_x(t) = K(x,t)$. Assumption A3 implies that
%
\begin{align}
	\int_{T}K(x,t)^2d\nu(t) < \infty \; \forall x \in T, \;\;\; \int_{T \times T}K(x,t)^2d\nu(x)d\nu(t) < \infty.
\end{align}
It follows that $K_x \in \Lcal^2(T,\nu)$ $\forall x \in T$, 
hence $\H_K \subset \Lcal^2(T,\nu)$ \citep{
	Sun2005MercerNoncompact}. Define the following linear operator
\begin{align}
	&R_K = R_{K,\nu}: \Lcal^2(T, \nu) \mapto \H_K,\;\;
	R_Kf = \int_{T}K_tf(t)d\nu(t), \;\;\; (R_Kf)(x) = \int_{T}K(x,t)f(t)d\nu(t).
	\label{equation:RK}
\end{align}
The operator $R_K$ is bounded, with
$||R_K:\Lcal^2(T,\nu) \mapto \H_K|| \leq \sqrt{\int_{T}K(t,t)d\nu(t)} \leq \kappa$.
Its adjoint operator is 
$R_K^{*}: \H_K \mapto \Lcal^2(T, \nu) = J:\H_K \inclusion \Lcal^2(T, \nu)$, the inclusion operator
from $\H_K$ into $\Lcal^2(T,\nu)$ (see e.g. \cite{Rosasco:IntegralOperatorsJMLR2010}). 
$R_K$ and $R_{K}^{*}$ together
induce 
the following self-adjoint, positive, trace class operator (e.g. \citep{CuckerSmale,Sun2005MercerNoncompact,Rosasco:IntegralOperatorsJMLR2010})
%
%
\begin{align}
	\label{equation:CK}
	&C_K = C_{K,\nu} = R_K^{*}R_K: \Lcal^2(T,\nu) \mapto \Lcal^2(T, \nu),
	\;\;
	(C_Kf)(x)  = \int_{T}K(x,t)f(t)d\nu(t), \;\; \forall f \in \Lcal^2(T,\nu),
	\\
	&\trace(C_K) = \int_{T}K(x,x)d\nu(x) \leq \kappa^2, \;\;||C_K||_{\HS(\Lcal^2(T,\nu))}^2  = \int_{T \times T}K(x,t)^2d\nu(x)d\nu(t) \leq \kappa^4.
\end{align}
Let $\{\lambda_k\}_{k \in \Nbb}$ be the eigenvalues
of $C_K$, 
with normalized eigenfunctions $\{\phi_k\}_{k\in \Nbb}$
forming an orthonormal basis in $\Lcal^2(T,\nu)$.
Mercer's Theorem (see version in \citep{Sun2005MercerNoncompact}) states that 
\begin{align}
	K(x,y) = \sum_{k=1}^{\infty}\lambda_k \phi_k(x)\phi_k(y)\;\;\;\forall (x,y) \in T \times T,
\end{align}
where the series converges absolutely for each pair $(x,y) \in T \times T$ and uniformly on any compact subset of $T$.
Thus $K$ and $C_K$ completely determine each other.

Consider now the correspondence between the trace class operator $C_K$ as defined in 
Eq.\eqref{equation:CK}
and measurable Gaussian processes with paths in $\Lcal^2(T,\nu)$,
as
established 
in 
\citep{Rajput1972gaussianprocesses}.
Let $(\Omega, \Fcal, P)$ be a probability space, 
$\xi = (\xi(t))_{t \in T}
= (\xi(\omega,t))_{t \in T}$ be 
a real
Gaussian process on $(\Omega, \Fcal,P)$, with mean $m$ and covariance function $K$, denoted by $\xi\sim \GP(m,K)$, where 
\begin{align}
	m(t) = \bE{\xi(t)},\;\;K(s,t) = \bE[(\xi(s) - m(s))(\xi(t) - m(t))], \;\;s,t \in T.
\end{align}
The stochastic process $\xi$ is said to be {\it (jointly) measurable} if the map
$\xi:\Omega \times T \mapto \R$, defined by $(\omega,t) \mapto \xi(\omega,t)$ is measurable with respect to
the $\sigma$-algebras $\Fcal \times \Bsc(T)$ and $\Bsc(\R)$. 
By  (\cite{Rajput1972gaussianprocesses}, Theorem 2 and Corollary 1),
for a measurable Gaussian process $\xi$,
the sample paths
$\xi(\omega,\cdot) \in \H = \Lcal^2(T, \nu)$ almost $P$-surely, i.e.
$\int_{T}\xi^2(\omega,t)d\nu(t) < \infty$ almost $P$-surely
if and only if
\begin{align}
	\label{equation:condition-Gaussian-process-paths}
	\int_{T}m^2(t)d\nu(t) < \infty, \;\;\; \int_{T}K(t,t)d\nu(t) < \infty.
\end{align}
The condition for $K$ in Eq.\eqref{equation:condition-Gaussian-process-paths} is precisely assumption A3.
In this case, $\xi$ induces the following {\it Gaussian measure} $P_{\xi}$ on $(\H, \Bsc(\H))$:
$	P_{\xi}(B) = P\{\omega \in \Omega: \xi(\omega, \cdot) \in B\}, \; B \in \Bsc(\H)$,
with mean $m \in \H$ and covariance operator
$C_K: \H \mapto \H$, defined by Eq.\eqref{equation:CK}.
Conversely, let $\mu$ be a Gaussian measure on $(\H
, \Bsc(\H))$, then
there is a measurable
Gaussian process $\xi = (\xi(t))_{t \in T}$
with sample paths in $\H$,
with induced probability measure $P_{\xi} = \mu$.

Thus let $D$ be a divergence on the set $\Gauss(\H)$ of Gaussian measures on $\H$, then for two measurable Gaussian processes
$\xi^i \sim \GP(m_i, K^i)$ with paths in $\H=\Lcal^2(T,\nu)$, we can define the divergence between them by
(see also \citep{Panaretos:jasa2010,Fremdt:2013testing,Pigoli:2014,
	masarotto2019procrustes})
\begin{align}
D_{\GP}(\xi^1||\xi^2) = D[\Ncal(m_1, C_{K^1})||\Ncal(m_2,C_{K^2})].
\end{align}
In the following, we focus on the case $D = \KL^{\gamma}, D^{\gamma}_{\Rrm,r}$, with $m_1 =m_2 = 0$,
and their finite-dimensional approximations. We note that in our setting, this definition makes sense even without
the measurability assumption, since by Mercer's Theorem, a distance/divergence between the covariance operators $C_{K^i}$
corresponds to a distance/divergence between the covariance functions $K^i$, $i=1,2$. 

{\bf Comparison with Sections \ref{section:RKHS} and \ref{section:renyi-RKHS-infinite}}. In Sections \ref{section:RKHS} and \ref{section:renyi-RKHS-infinite}, we study divergences between covariance operators/Gaussian measures defined on an RKHS
	$\H_K$ induced by a positive definite kernel $K$.
In the current section, we study $D[\Ncal(0, C_{K^1})||\Ncal(0,C_{K^2})]$, expressed as Alpha Log-Det divergences between covariance operators on $\Lcal^2(T,\nu)$. Note that if $\dim(\H_{K^i}) = \infty$, then the sample paths of the Gaussian process $\xi^i$ lie outside of $\H_{K^i}$ almost surely (see e.g. \citep{lukic2001stochastic}). Thus
the analysis in this section requires different techniques, specifically RKHS cross-covariance operators, which we discuss next.

\subsection{Estimation of divergences from finite covariance matrices}
\label{section:estimation-divergence-finite-covariance-marices}

In general, it is often necessary to approximate the divergences between the infinite-dimensional Gaussian measures
$\Ncal(0,C_{K^1})$ and $\Ncal(0,C_{K^2})$ on $\Lcal^2(T,\nu)$ from their finite-dimensional versions.
Let $\Xbf = (x_i)_{i=1}^m$ be independently sampled from $(T,\nu)$. 
The Gaussian process assumption $\xi^i \sim \GP(0,K^i)$ means that $(\xi^i(., x_j))_{j=1}^m$ are $m$-dimensional Gaussian random variables,
with $(\xi^i(., x_j))_{j=1}^m \sim \Ncal(0, K^i[\Xbf])$, where
$(K^i[\Xbf])_{jk} = K^i(x_j, x_k)$, $1 \leq j,k \leq m$.
Assume for the moment that the covariance matrices $K^i[\Xbf]$ are {\it known}.
In this section, we show that the finite-dimensional divergences
	\begin{align*}
		&
		D_{\Rrm,r}^{\gamma}\left[\Ncal\left(0, \frac{1}{m}K^1[\Xbf]\right)\bigg\vert\bigg\vert \Ncal\left(0, \frac{1}{m}K^2[\Xbf]\right)\right] = D_{\Rrm,r}\left[\Ncal\left(0, K^1[\Xbf] + m\gamma I \right)\bigg\vert\bigg\vert \Ncal\left(0, K^2[\Xbf] + m\gamma I\right)\right],
		\\
		&{\KL}^{\gamma}\left[\Ncal\left(0, \frac{1}{m}K^1[\Xbf]\right)\bigg\vert\bigg\vert \Ncal\left(0, \frac{1}{m}K^2[\Xbf]\right)\right] = \KL[\Ncal(0, K^1[\Xbf]+m\gamma I)||\Ncal(0,K^2[\Xbf]+m\gamma I)]
		\end{align*}
consistently estimate the infinite-dimensional divergences
		 \begin{align*}
		&D_{\Rrm,r}^{\gamma}[\Ncal(0, C_{K^1})||\Ncal(0, C_{K^2})] \text{and }{\KL}^{\gamma}[\Ncal(0, C_{K^1})||\Ncal(0, C_{K^2})],
	\end{align*}
respectively,
as $m \approach \infty$, with {\it dimension-independent} sample complexities.

We first note that, as linear operators, $\frac{1}{m}K^{i}[\Xbf]: \R^m \mapto \R^m$ and
$C_{K^i}: \Lcal^2(T,\nu) \mapto \Lcal^2(T,\nu)$, with $\R^m$ and $\Lcal^2(T,\nu)$ being two different Hilbert spaces.
Thus Theorem \ref{theorem:logdet-approx-infinite-sequence} {\it cannot} be applied.
Instead, we first represent both the theoretical divergence $D_{\Rrm,r}^{\gamma}[\Ncal(0, C_{K^1})||\Ncal(0, C_{K^2})]$
and empirical divergence $D_{\Rrm,r}^{\gamma}\left[\Ncal\left(0, \frac{1}{m}K^1[\Xbf]\right)\bigg\vert\bigg\vert \Ncal\left(0, \frac{1}{m}K^2[\Xbf]\right)\right]$ (and similarly for the KL divergences) via RKHS covariance and cross-covariance operators acting on the same Hilbert spaces.
The convergence of these operators then leads to the consistency of $D_{\Rrm,r}^{\gamma}$ and $\KL^{\gamma}$.
This approach was used in \citep{Minh2021:FiniteEntropicGaussian} for the consistency of the entropic Wasserstein distances, using the trace class continuity of the Fredholm determinant. In the current work, we employ the representations of  $D_{\Rrm,r}^{\gamma}$ and $\KL^{\gamma}$ via the {\it Hilbert-Carleman determinant} and exploit its Hilbert-Schmidt continuity to obtain the consistency and sample complexities of the corresponding divergences (in the current setting, this is {\it not} possible with the representations using the Fredholm determinant).

{\bf RKHS covariance and cross-covariance operators}. Let us first review these operators (see also \citep{Minh2021:FiniteEntropicGaussian}).
Let $K^1,K^2$ be two kernels satisfying Assumptions B1-B4, and $\H_{K^1}, \H_{K^2}$ the corresponding
RKHS. Let $R_{K^i}:\Lcal^2(T,\nu) \mapto \H_{K^i}$, $i=1,2$ be as defined in Eq.\eqref{equation:RK}.
Together, they define the following {\it RKHS cross-covariance operators}
%
%
{
	\begin{align}
		\label{equation:R12-operator}
		R_{ij}&= R_{K^i}R_{K^j}^{*}: \H_{K^j} \mapto \H_{K^i},\;\;\;i,j=1,2,
		\\
		R_{ji} &= R_{K^j}R_{K^i}^{*}: \H_{K^i}\mapto \H_{K^j} = R_{ij}^{*},
		\\
		R_{ij} & = \int_{T}(K^i_t \otimes K^j_t)d\nu(t),\;\;
		R_{ij}f = \int_{T}K^i_t \la f, K^j_t\ra_{\H_{K^j}}d\nu(t),  
		\\
		R_{ij}f(x) &= \int_{T}K^i_t(x)f(t)d\nu(t) = \int_{T}K^i(x,t)f(t)d\nu(t),\;\; f\in \H_{K^j}.
	\end{align}
}
In particular, $R_{ii} = L_{K^{i}}$, with the {\it RKHS covariance operator} $L_{K}$ defined by
{
	\begin{align}
		\label{equation:LK}
		L_K &= R_KR_K^{*}: \H_K \mapto \H_K, \;\;L_K = \int_{T}(K_t \otimes K_t) d\nu(t),\;
		\\
		L_Kf(x) &= \int_{T}K_t(x)\la f, K_t\ra_{\H_K}d\nu(t) = \int_{T}K(x,t)f(t)d\nu(t),\;\; f\in \H_K.
	\end{align}
}
$L_K$ has the same nonzero eigenvalues as $C_K$ and thus $L_K \in \Sym^{+}(\H_K) \cap \Tr(\H_K)$, with 
{
	\begin{align}
		\trace(L_K) = \trace(C_K) \leq \kappa^2,\;\;||L_K||_{\HS(\H_K)} = ||C_K||_{\HS(\Lcal^2(T,\nu))} \leq \kappa^2.
	\end{align}
}
\begin{lemma}
	[\citep{Minh2021:FiniteEntropicGaussian}]
	\label{lemma:R12-HS}
	Under Assumptions A1-A3,
	$R_{ij} \in \HS(\H_{K^j}, \H_{K^i})$, with
	$||R_{ij}||_{\HS(\H_{K^j}, \H_{K^i})} \leq \kappa_i\kappa_j$, $i,j=1,2$.
\end{lemma}

{\bf Empirical RKHS covariance and cross-covariance operators}. Let $\Xbf = (x_i)_{i=1}^m$ be 
independently sampled from $T$ according to $\nu$. 
It gives rise to
the following bounded {\it sampling operator}
(see e.g. \citep{SmaleZhou2007})
{
	\begin{align}
		&S_{\Xbf}: \H_{K} \mapto \R^m, \;\;\; S_{\Xbf}f = (f(x_i))_{i=1}^m = (\la f, K_{x_i}\ra)_{i=1}^m
		\;\;
		\\
		&\text{with adjoint operator }
		S_{\Xbf}^{*}: \R^m \mapto \H_{K},\;\;\; S_{\Xbf}^{*}\b = \sum_{i=1}^mb_iK_{x_i}.
	\end{align}
}
For $K^i$, $i=1,2$, denote the corresponding
sampling operators by $S_{i,\Xbf}:\H_{K^i}\mapto \R^m$, $i=1,2$. Together, they define the following empirical version
of $R_{ij}$
{
	\begin{align}
		R_{ij,\Xbf} & = \frac{1}{m}S_{i,\Xbf}^{*}S_{j, \Xbf} = \frac{1}{m}\sum_{k=1}^m(K^i_{x_k} \otimes K^j_{x_k}): \H_{K^j} \mapto \H_{K^i},
		\\
		R_{ij,\Xbf}f &= \frac{1}{m}\sum_{k=1}^mK^i_{x_k}\la f, K^j_{x_k}\ra_{\H_{K^j}} = \frac{1}{m}\sum_{k=1}^mf(x_k)K^i_{x_k}, \; f\in \H_{K^j}.
		%
	\end{align}
}
In particular, $R_{ii,\Xbf} = L_{K^i,\Xbf}$, with the empirical RKHS covariance operator $L_{K,\Xbf}:\H_K \mapto \H_K$ being self-adjoint and defined by
{
	\begin{align}
		&L_{K,\Xbf} = \frac{1}{m}S_{\Xbf}^{*}S_{\Xbf} =\frac{1}{m}\sum_{i=1}^m (K_{x_i} \otimes K_{x_i}):\H_K \mapto \H_K, 
		\\
		\;\;
		&
		L_{K,\Xbf}f = \frac{1}{m}S_{\Xbf}^{*}(f(x_i))_{i=1}^m = \frac{1}{m}\sum_{i=1}^mf(x_i)K_{x_i} = \frac{1}{m}\sum_{i=1}^m\la f, K_{x_i}\ra_{\H_K}K_{x_i}.
	\end{align}
}
Furthermore, the operator $S_{\Xbf}S_{\Xbf}^{*}:\R^m \mapto \R^m$ is given by
{
	\begin{align}
		S_{\Xbf}S_{\Xbf}^{*}\b = S_{\Xbf}\sum_{i=1}^mb_iK_{x_i} = (\sum_{i=1}^mb_iK(x_i, x_1), \ldots, \sum_{i=1}^mb_iK(x_i, x_m)) = K[\Xbf]\b.
	\end{align}
}
We note that while $S_{i,\Xbf}^{*}S_{j, \Xbf}:\H_{K^j} \mapto \H_{K^i}$ is always well-defined,
 $S_{j, \Xbf}S_{i,\Xbf}^{*}:\R^m \mapto \R^m$ is only well-defined when $\H_{K^i} \subseteq \H_{K^j}$.
In particular,
the nonzero eigenvalues of $L_{K,\Xbf} = \frac{1}{m}S_{\Xbf}^{*}S_{\Xbf}$ are the same as those of $\frac{1}{m}K[\Xbf] = \frac{1}{m}\S_{\Xbf}S_{\Xbf}^{*}$, corresponding to 
eigenvectors that must lie in $\H_{K,\Xbf}=\myspan\{K_{x_i}\}_{i=1}^m$.
Thus, the nonzero eigenvalues of $C_K:\Lcal^2(T, \nu) \mapto \Lcal^2(T,\nu)$, $\trace(C_K)$, $||C_K||_{\HS}$,
which are the same as those of $L_K:\H_K \mapto \H_K$, can be empirically estimated from those of the $m \times m$ matrix $\frac{1}{m}K[\Xbf]$ (see \citep{Rosasco:IntegralOperatorsJMLR2010}).

In terms of the RKHS covariance and cross-covariance operators, $D_{\Rrm,r}^{\gamma}$ and $\KL^{\gamma}$,
both the theoretical and empirical versions, admit the following representations.

\begin{proposition}
		[\textbf{Renyi divergence representation in RKHS operators}]
	\label{proposition:Renyi-RKHS-operator-representation}
	Let $\gamma \in \R$, $\gamma > 0$ be fixed. Let $0 < r < 1$ be fixed. Then
\begin{align}
&D^{\gamma}_{\Rrm,r}[\Ncal(0, C_{K^1})||\Ncal(0, C_{K^2})]
\nonumber
\\
& =  \frac{1}{2r(1-r)}\logdet
\left[\frac{1}{\gamma}\begin{pmatrix}
	(1-r)L_{K^1} & \sqrt{r(1-r)}R_{12}
	\\
	\sqrt{r(1-r)}R_{12}^{*} & rL_{K^2}
\end{pmatrix} + I_{\H_{K^1} \oplus \H_{K^2}}\right]
\nonumber
\\
& \quad -\frac{1}{2r}\logdet\left(\frac{L_{K^1}}{\gamma} + I_{\H_{K^1}}\right)
-\frac{1}{2(1-r)}\logdet\left(\frac{L_{K^2}}{\gamma} + I_{\H_{K^2}}\right)
\\
&=  \frac{1}{2r(1-r)}\log\dettwo
\left[\frac{1}{\gamma}\begin{pmatrix}
	(1-r)L_{K^1} & \sqrt{r(1-r)}R_{12}
	\\
	\sqrt{r(1-r)}R_{12}^{*} & rL_{K^2}
\end{pmatrix} + I_{\H_{K^1} \oplus \H_{K^2}}\right]
\nonumber
\\
& \quad -\frac{1}{2r}\log\dettwo\left(\frac{L_{K^1}}{\gamma} + I_{\H_{K^1}}\right)
-\frac{1}{2(1-r)}\log\dettwo\left(\frac{L_{K^2}}{\gamma} + I_{\H_{K^2}}\right). 
\end{align}
The empirical version is given by
\begin{align}
	&D^{\gamma}_{\Rrm,r}\left[\Ncal\left(0, \frac{1}{m}K^1[\Xbf]\right),\Ncal\left(0,\frac{1}{m}K^2[\Xbf]\right)\right]
	\nonumber
	\\
	&=\frac{1}{2r(1-r)}\logdet
	\left[\frac{1}{\gamma}\begin{pmatrix}
		(1-r)L_{K^1,\Xbf} & \sqrt{r(1-r)}R_{12,\Xbf}
		\\
		\sqrt{r(1-r)}R_{12,\Xbf}^{*} & rL_{K^2,\Xbf}
	\end{pmatrix} + I_{\H_{K^1} \oplus \H_{K^2}}\right]
\nonumber
	\\
	& \quad -\frac{1}{2r}\logdet\left(\frac{L_{K^1,\Xbf}}{\gamma} + I_{\H_{K^1}}\right)
	-\frac{1}{2(1-r)}\logdet\left(\frac{L_{K^2,\Xbf}}{\gamma} + I_{\H_{K^2}}\right)
	\\
	&=\frac{1}{2r(1-r)}\log\dettwo
	\left[\frac{1}{\gamma}\begin{pmatrix}
		(1-r)L_{K^1,\Xbf} & \sqrt{r(1-r)}R_{12,\Xbf}
		\\
		\sqrt{r(1-r)}R_{12,\Xbf}^{*} & rL_{K^2,\Xbf}
	\end{pmatrix} + I_{\H_{K^1} \oplus \H_{K^2}}\right]
	\nonumber
	\\
	& \quad -\frac{1}{2r}\log\dettwo\left(\frac{L_{K^1,\Xbf}}{\gamma} + I_{\H_{K^1}}\right)
	-\frac{1}{2(1-r)}\log\dettwo\left(\frac{L_{K^2,\Xbf}}{\gamma} + I_{\H_{K^2}}\right).
\end{align}
\end{proposition}

\begin{proposition}
	[\textbf{KL divergence representation in RKHS operators}]
	\label{proposition:KL-RKHS-operator-representation}
	Let $\gamma \in \R$, $\gamma > 0$ be fixed. Then
	{
		\begin{align}
			&{\KL}^{\gamma}[\Ncal(0, C_{K^1})||\Ncal(0, C_{K^2})] = 
			\frac{1}{2\gamma}\trace[L_{K^1}] 
			-\frac{1}{2}\trace\left[\left(\frac{L_{K^2}}{\gamma} + \frac{R_{12}^{*}R_{12}}{\gamma^2}\right)\left(I_{\H_{K^2}} + \frac{L_{K^2}}{\gamma}\right)^{-1}\right]
			\nonumber
			\\
			&\quad -\frac{1}{2}\log\det\left(\frac{L_{K^1}}{\gamma}+I_{\H_{K^1}}\right) + \frac{1}{2}\log\det\left(\frac{L_{K^2}}{\gamma} + I_{\H_{K^2}}\right)
			\\
			& =  - \frac{1}{2}\log\dettwo\left(I_{\H_{K^1}}+\frac{L_{K^1}}{\gamma}\right) - \frac{1}{2}\log\dettwo\left(I_{\H_{K^2}}+\frac{L_{K^2}}{\gamma}\right)^{-1} 
			\nonumber
			\\
			&\quad - \trace\left[\frac{R_{12}^{*}R_{12}}{2\gamma^2}\left(I_{\H_{K^2}}+\frac{L_{K^2}}{\gamma}\right)^{-1}\right].
		\end{align}
	}
	The empirical version is given by
	{
		\begin{align}
			&{\KL}^{\gamma}\left[\Ncal\left(0, \frac{1}{m}K^1[\Xbf]\right)\bigg\vert\bigg\vert \Ncal\left(0, \frac{1}{m}K^2[\Xbf]\right)\right]
			\nonumber
			\\
			&=\frac{1}{2\gamma}\trace[L_{K^1,\Xbf}] -\frac{1}{2}\trace\left[\left(\frac{L_{K^2,\Xbf}}{\gamma} + \frac{R_{12,\Xbf}^{*}R_{12,\Xbf}}{\gamma^2}\right)\left(I_{\H_{K^2}} + \frac{L_{K^2,\Xbf}}{\gamma}\right)^{-1}\right]
			\\
			&\quad -\frac{1}{2}\log\det\left(\frac{L_{K^1,\Xbf}}{\gamma}+I_{\H_{K^1}}\right) + \frac{1}{2}\log\det\left(\frac{L_{K^2,\Xbf}}{\gamma} + I_{\H_{K^2}}\right)
			\nonumber
			\\
			& =  - \frac{1}{2}\log\dettwo\left(I_{\H_{K^1}}+\frac{L_{K^1,\Xbf}}{\gamma}\right) - \frac{1}{2}\log\dettwo\left(I_{\H_{K^2}}+\frac{L_{K^2,\Xbf}}{\gamma}\right)^{-1}
			\\
			&\quad
			- \trace\left[\frac{R_{12,\Xbf}^{*}R_{12,\Xbf}}{2\gamma^2}\left(I_{\H_{K^2}}+\frac{L_{K^2,\Xbf}}{\gamma}\right)^{-1}\right].
			\nonumber
		\end{align}
	} 
\end{proposition}

The RKHS operator representations in Propositions \ref{proposition:Renyi-RKHS-operator-representation} and \ref{proposition:KL-RKHS-operator-representation} suggest that,
as $m \approach \infty$, 	$D^{\gamma}_{\Rrm,r}\left[\Ncal\left(0, \frac{1}{m}K^1[\Xbf]\right),\Ncal\left(0,\frac{1}{m}K^2[\Xbf]\right)\right]$
and $\KL^{\gamma}\left[\Ncal\left(0, \frac{1}{m}K^1[\Xbf]\right),\Ncal\left(0,\frac{1}{m}K^2[\Xbf]\right)\right]$
should converge, in some sense, to $D^{\gamma}_{\Rrm,r}[\Ncal(0, C_{K^1})||\Ncal(0, C_{K^2})]$ and ${\KL}^{\gamma}[\Ncal(0, C_{K^1})||\Ncal(0, C_{K^2})]$, respectively.
To obtain the rate of convergence, we first need to have rates of convergence of
$L_{K^1,\Xbf},L_{K^2,\Xbf},R_{12,\Xbf}$ to $L_{K^1},L_{K^2}, R_{12}$, respectively. For simplicity, 
in the following we use the following additional assumption
\begin{align}
\text{(Assumption \textbf{B5}) }\;\;\; \sup_{x \in T}K^i(x,x) \leq \kappa_i^2, \;\;\;i=1,2.
\end{align}
\begin{proposition}
	(\textbf{Convergence of RKHS empirical covariance and cross-covariance operators},\cite{Minh2021:FiniteEntropicGaussian})
	\label{proposition:concentration-TK2K1-empirical}
	Under Assumptions B1-B5,
	$||R_{ij,\Xbf}||_{\HS(\H_{K^j}, \H_{K^i})} \leq \kappa_i \kappa_j$, $i,j=1,2$, $\forall \Xbf \in T^m$.
	Let $\Xbf = (x_i)_{i=1}^m$ be independently sampled from $(T,\nu)$.
	$\forall 0 < \delta < 1$, with probability at least $1-\delta$,
	\begin{align}
		||R_{ij,\Xbf} - R_{ij}||_{\HS(\H_{K^j}, \H_{K^i})} \leq \kappa_i\kappa_j\left[ \frac{2\log\frac{2}{\delta}}{m} + \sqrt{\frac{2\log\frac{2}{\delta}}{m}}\right],
		\\
		||R_{ij,\Xbf}^{*}R_{ij,\Xbf} - R_{ij}^{*}R_{ij}||_{\tr(\H_{K^j})} \leq 2\kappa_i^2\kappa_j^2\left[ \frac{2\log\frac{2}{\delta}}{m} + \sqrt{\frac{2\log\frac{2}{\delta}}{m}}\right].
	\end{align}	
	In particular, 
	$||L_{K^i,\Xbf}||_{\HS(\H_{K^i})} \leq \kappa_i^2$, 
	and $\forall 0 <\delta <1$, 
	with probability at least $1-\delta$,
	\begin{align}
		\label{equation:LKX-LK-bounded}
		\left\|L_{K^i,\Xbf} - L_{K^i}\right\|_{\HS(\H_{K^i})} \leq \kappa_i^2\left(\frac{2\log\frac{2}{\delta}}{m} + \sqrt{\frac{2\log\frac{2}{\delta}}{m}}\right). 
	\end{align}
\end{proposition}

\begin{theorem}
	(\textbf{Estimation of Renyi divergences between centered Gaussian processes from finite covariance matrices})
	\label{theorem:Renyi-estimate-finite-covariance}
	Assume Assumptions B1-B5. 
	Let $\gamma \in \R, \gamma > 0$ be fixed. Let $0 < r < 1$ be fixed.
	Let $\Xbf = (x_j)_{j=1}^m$ be independently sampled from $(T,\nu)$.
	For any $0 < \delta < 1$, with probability at least $1-\delta$,
	\begin{align}
		&\left|D_{\Rrm,r}^{\gamma}\left[\Ncal\left(0, \frac{1}{m}K^1[\Xbf]\right)\bigg\vert\bigg\vert \Ncal\left(0, \frac{1}{m}K^2[\Xbf]\right)\right] 
		-D_{\Rrm,r}^{\gamma}[\Ncal(0, C_{K^1})||\Ncal(0, C_{K^2})]
		\right|
		\nonumber
		\\
	&\leq \frac{1}{2\gamma^2}\left[\frac{\kappa_1^4}{r} + \frac{\kappa_2^4}{1-r} + \frac{[(1-r)\kappa_1^2 + r\kappa_2^2]^2}{r(1-r)}\right]\left(\frac{2\log\frac{6}{\delta}}{m} + \sqrt{\frac{2\log\frac{6}{\delta}}{m}}\right).
	\end{align}
\end{theorem}

\begin{theorem}
	(\textbf{Estimation of KL divergence between centered Gaussian processes from finite covariance matrices})
	\label{theorem:KL-estimate-finite-covariance}
	Assume Assumptions B1-B5. 
	Let $\gamma \in \R, \gamma > 0$ be fixed.
	Let $\Xbf = (x_j)_{j=1}^m$ be independently sampled from $(T,\nu)$.
	For any $0 < \delta < 1$, with probability at least $1-\delta$,
	\begin{align}
		&\left|{\KL}^{\gamma}\left[\Ncal\left(0, \frac{1}{m}K^1[\Xbf]\right)\bigg\vert\bigg\vert \Ncal\left(0, \frac{1}{m}K^2[\Xbf]\right)\right] 
		-{\KL}^{\gamma}[\Ncal(0, C_{K^1})||\Ncal(0, C_{K^2})]
		\right|
		\nonumber
		\\
		& \leq \frac{1}{2\gamma^2}\left[\kappa_1^4 + \kappa_2^4 + \kappa_1^2\kappa_2^2\left(2+\frac{\kappa_2^2}{\gamma}\right)\right]\left(\frac{2\log\frac{6}{\delta}}{m} + \sqrt{\frac{2\log\frac{6}{\delta}}{m}}\right).
	\end{align}
\end{theorem}
We note that if both $\kappa_1,\kappa_2$ are absolute constants, e.g. with exponential kernels,
then the sample complexities in Theorems \ref{theorem:Renyi-estimate-finite-covariance} and \ref{theorem:KL-estimate-finite-covariance} are both completely {\it dimension-independent}.
Theorems \ref{theorem:Renyi-estimate-finite-covariance} and \ref{theorem:KL-estimate-finite-covariance}
provide the theoretical basis for the following Algorithm \ref{algorithm:estimate-divergence-finite-COV}, which consistently estimates the  R\'enyi and KL divergences between covariance operators/centered Gaussian processes
from finite-dimensional zero-mean Gaussian measures.
\begin{algorithm}
	\caption{Consistent estimates of R\'enyi and KL divergences between covariance operators/centered Gaussian processes
		from finite-dimensional zero-mean Gaussian measures}
	\label{algorithm:estimate-divergence-finite-COV}
	\begin{algorithmic}
		\REQUIRE Finite covariance matrices $K^i[\Xbf]$ from processes $\xi^i$, $i=1,2$, sampled at $m$ points
		$\Xbf = (x_j)_{j=1}^m$. 
		\REQUIRE Regularization parameter $\gamma > 0$, $0< r < 1$ (for the R\'enyi divergence)
		
		\textbf{Procedure}
		\STATE{Compute $D^{\gamma}_{\Rrm,r} = \frac{1}{2}d^{2r-1}_{\logdet}\left[\left(\gamma I_m+ \frac{1}{m}{K}^1[\Xbf]\right), \left(\gamma I_m+ \frac{1}{m}{K}^2[\Xbf]\right)\right] = D_{\Rrm,r}[\Ncal(0,K^1[\Xbf]+m\gamma I_m)||\Ncal(0,K^2[\Xbf]+m\gamma I_m)]$}
		\STATE{Compute $\KL^{\gamma} = \frac{1}{2}d^{1}_{\logdet}\left[\left(\gamma I_m+ \frac{1}{m}{K}^1[\Xbf]\right), \left(\gamma I_m+ \frac{1}{m}{K}^2[\Xbf]\right)\right] = \KL[\Ncal(0,K^1[\Xbf]+m\gamma I_m)||\Ncal(0,K^2[\Xbf]+m\gamma I_m)]$}
		\RETURN{$D^{\gamma}_{\Rrm,r}$ and $\KL^{\gamma}$}
	\end{algorithmic}
\end{algorithm}

{\bf Connection with previous work}. Let us connect the current results with those in \citep{Sun2019functionalKL} (see also discussion in \citep{burt2020understanding}).
Let $\R^T$ be the set of all real-valued functions $x = (x_t)_{t \in T}$. 
Let $\Bsc(\R^T)$ be the smallest $\sigma$-algebra corresponding to the cylinder sets
$\Ical_{t_1, \ldots, t_n}(B_1 \times \cdots \times B_n) = \{x: x_{t_1} \in B_1, \ldots, x_{t_n} \in B_n\}$
where $B_i$ is a Borel set on $\R$.
Let $\xi = (\xi_t)_{t \in T}$ be a real-valued stochastic process defined on a probability space $(\Omega, \Fcal,P)$, then its most important characteristic is the set of finite-dimensional distribution functions $\{F_{t_1,\ldots, t_n}\}$ on $\R^n$, which are defined by 
\begin{align}
	F_{t_1, \ldots, t_n}(B_1 \times \cdots \times B_n) = P(\{\omega: \xi_{t_1}(\omega) \in B_1, \ldots, \xi_{t_n}(\omega) \in B_n\}),
\end{align}
where $B_1, \ldots, B_n$ are Borel sets in $\R$.
The following fundamental theorem shows that the existence of a stochastic process $\xi = (\xi_t)_{t \in T}$ is fully guaranteed by the consistency of its finite-dimensional distribution functions
\begin{theorem}
	(\textbf{Kolmogorov Theorem on the Existence of a Stochastic Process}, see e.g. \cite[Theorem 2.1.5]{oksendal2013stochastic}, \citep[Theorem II.9.1]{Shiryaev:Probability1996})
	\label{theorem:Kolmogorov}
	Let $\{F_{t_1, \ldots, t_n}\}$, $t_i \in T$,
	be a given family of finite-dimensional distributions satisfying the following consistency conditions 
	\begin{align}
		\label{equation:consistency-1}
		F_{t_{\sigma(1)}, \ldots, t_{\sigma(n)}}(B_1 \times \cdots \times B_n) = F_{t_1, \ldots, t_n}(B_{\sigma^{-1}(1)} \times \cdots \times B_{\sigma^{-1}(n)}),
	\end{align}
	for all permutations $\sigma$ of $\{1, \ldots, n\}$, and
	\begin{align}
		\label{equation:consistency-2}
		F_{t_1, \ldots, t_n}(B_1 \times \cdots \times B_n) = F_{t_1, \ldots, t_n, t_{n+1}, \ldots, t_{n+k}}(B_1 \times \cdots \times B_n \times \R \times \cdots \times \R),
	\end{align}
	where on the right hand side, $\R$ appears $k$ times.
	Then there exist a probability space $(\Omega, \Fcal, P)$ and a stochastic process
	$\xi = (\xi_t)_{t \in T}$ such that
	\begin{align}
		P(\{\omega: \xi_{t_1}(\omega) \in B_1, \ldots, \xi_{t_n}(\omega) \in B_n\}) = F_{t_1, \ldots, t_n}(B_1\times \cdots \times B_n).
	\end{align}
	
\end{theorem}
In Theorem \ref{theorem:Kolmogorov}, one can take $\Omega = \R^T, \Fcal = \Bsc(\R^T)$. The space $(\R^T, \Bsc(\R^T), P)$
is called the {\it canonical} probability space associated with the stochastic process $\xi$.
Thus one can view the stochastic process $\xi$ as the probability measure $P$ on the space $(\R^T, \Bsc(\R^T))$.

Let $\xi^i = (\xi^i_t)_{t \in T}$ be two stochastic processes with corresponding canonical probability spaces
$(\R^T,\Bsc(\R^T), P_{\xi^i})$, we can thus define (\citep{Sun2019functionalKL})
\begin{align}
	\KL(\xi^1||\xi^2) = \KL(P_{\xi^1}||P_{\xi^2}).
\end{align}
If $\xi^i$ is a measurable process and the sample paths of $\xi$ belong to $\Lcal^p = \Lcal^p(T, \Bsc(T),\nu)$ almost surely, $1 \leq p < \infty$, then $P_{\xi^i}$ is concentrated on $\Bsc(\Lcal^p)$. In particular, if $\xi^i$ is a measurable Gaussian process, then $P_{\xi^i}$, viewed now as a probability measure on $\Lcal^p$, is a Gaussian measure on $\Lcal^p$ \citep{Rajput1972GaussianMeasuresLp,vakhaniya1979correspondence}, see also \citep{Bogachev:Gaussian}.

Let $\Xbf = (x_i)_{i=1}^m$ be sampled from $T$. Let $P_{\xi,\Xbf}$ be the finite-dimensional distribution function defined on the index set $\Xbf$. It was shown in  (\citep{Sun2019functionalKL}) that
\begin{align}
	\label{equation:KL-sup-def}
	\KL(P_{\xi^1}||P_{\xi^2}) = \sup_{\Xbf \in T^m, m \in \Nbb}\KL(P_{\xi^1,\Xbf}||P_{\xi^2,\Xbf}).
\end{align}

In the Gaussian setting, with $\xi^i \sim \GP(0, K^i)$, we have $P_{\xi^i,\Xbf} = \Ncal(0, K^i[\Xbf])$.
Assume that $K^2[\Xbf]$ is always strictly positive definite. Then the following is always finite
\begin{align}
	\label{equation:KL-finite-marginal}
	\KL(P_{\xi^1,\Xbf}||P_{\xi^2,\Xbf}) &= \KL[\Ncal(0,K^1[\Xbf])||\Ncal(0, K^2[\Xbf])]
	\\
	& = \lim_{\gamma \approach 0} \KL[\Ncal(0,K^1[\Xbf]+m \gamma I )||\Ncal(0, K^2[\Xbf] + m\gamma I)]
	\\
	& = \lim_{\gamma \approach 0}\KL^{\gamma}\left[\Ncal\left(0, \frac{1}{m}K^1[\Xbf]\right)||\Ncal\left(0, \frac{1}{m}K^2[\Xbf]\right)\right]
\end{align}
It is not clear, however, how to obtain a finite sample complexity using Eqs.\eqref{equation:KL-sup-def}
and \eqref{equation:KL-finite-marginal}.
By employing the regularization approach, we are able to obtain dimension-independent sample complexities for
$\KL^{\gamma}[\Ncal(0, C_{K^1})||\Ncal(0,C_{K^2})]$, $D_{\Rrm,r}^{\gamma}[\Ncal(0, C_{K^1})||\Ncal(0,C_{K^2})]$, for any fixed $\gamma > 0$, in the case $p=2$. 
Furthermore, by Theorem \ref{theorem:exact-regularized-correspondence}, for 
two equivalent Gaussian processes,
\begin{align}
\lim_{\gamma \approach 0^{+}}\KL^{\gamma}[\Ncal(0, C_{K^1})||\Ncal(0,C_{K^2})] = \KL[\Ncal(0, C_{K^1})||\Ncal(0,C_{K^2})].
\end{align}
Thus, for $\gamma >0$ sufficiently small, the regularized divergences approximate the exact divergences between equivalent Gaussian processes (with the mutually singular processes being infinitely far apart).

{\bf Comparison with the $2$-Wasserstein distance}.  In \citep{Minh2021:FiniteEntropicGaussian}, it was shown that for the $2$-Wasserstein distance $W_2$,
\begin{align}
	\label{equation:2-Wasserstein-Gaussian-process}
W_2^2(\Ncal(0,C_{K^1}),\Ncal(0,C_{K^2})) = \trace(C_{K^1}) + \trace(C_{K^2}) - 2\trace(C_{K^1}^{1/2}C_{K^2}C_{K^1}^{1/2})^{1/2}
\end{align}
which is called the squared Bures-Wasserstein distance,
is consistently estimated by
\begin{align}
	\label{equation:2-Wasserstein-Gaussian-process-empirical}
W_2^2\left[\Ncal\left(0, \frac{1}{m}K^1[\Xbf]\right),\Ncal\left(0, \frac{1}{m}K^1[\Xbf]\right)\right]=
\frac{1}{m}W_2^2[\Ncal(0,K^1[\Xbf]), \Ncal(0,K^2[\Xbf])],
\end{align}
with the {\it dimension-dependent} sample complexities given when $\dim(\H_{K^i}) < \infty$ for at least one $i$, $i=1,2$. This is due to the fact that the $2$-Wasserstein distance is continuous in the trace class norm,
which is a Banach space norm of type $1$,
in contrast to the continuity in the Hilbert-Schmidt norm, which is a Hilbertian norm, of $D_{\Rrm,r}^{\gamma}$ and $\KL^{\gamma}$, and for which more can be obtained with laws of large numbers. The entropic regularized $2$-Wasserstein distance, namely the Sinkhorn divergence, is continuous in the Hilbert-Schmidt norm, and hence possesses {\it dimension-independent} sample complexity \citep{Minh:2021EntropicConvergenceGaussianMeasures,Minh2021:FiniteEntropicGaussian}.

We also note the factor $\frac{1}{m}$ on the right hand side of Eq.\eqref{equation:2-Wasserstein-Gaussian-process-empirical}, which does {\it not} appear in Eq.\eqref{equation:KL-sup-def}. This is because the Alpha Log-Det divergences are {\it scale-invariant},
so that $\KL[\Ncal(0,\frac{1}{m}K^1[\Xbf])||\Ncal(0,\frac{1}{m}K^2[\Xbf])] = \KL[\Ncal(0, K^1[\Xbf])||\Ncal(0,K^2[\Xbf])]$, whereas $W_2$ is not scale-invariant and the factor $\frac{1}{m}$ is needed for convergence as $m \approach \infty$.

\subsection{Estimation of divergences from finite samples}
\label{section:estimation-divergence-finite-samples}

In practice, 
the covariance functions $K^1,K^2$ are generally {\it unknown}.
Assume now that 
we have access to
samples of the Gaussian processes $\xi^1,\xi^2$ on a finite set of points $\Xbf = (x_j)_{j=1}^m$
on $T$.
We can then estimate the finite covariance matrices $K^1[\Xbf], K^2[\Xbf]$
and use the regularized divergences between them
as the approximate version of the infinite-dimensional divergences. In this section,
we show that this approach leads to consistent estimate of
$\KL^{\gamma}[\Ncal(0,C_{K^1}),\Ncal(0,K^2)]$ and $D^{\gamma}_{\Rrm,r}[\Ncal(0,C_{K^1}),\Ncal(0,K^2)]$,
with {\it dimension-independent} sample complexities (the number of sample paths $N$ and sample points $m$).

For the Gaussian process $\xi = (\xi(\omega, t))$ defined on the probability space
$(\Omega, \Fcal, P)$, let $\Wbf = (\omega_1, \ldots, \omega_N)$ be independently sampled from $(\Omega,P)$, which corresponds to $N$ sample paths
$\xi_i(x) = \xi(\omega_i,x), 1 \leq i \leq N, x \in T$.
Let $\Xbf =(x_i)_{i=1}^m \in T^m$ be fixed. 
Consider the following $m \times N$ data matrix
	\begin{align}
		\Zbf = \begin{pmatrix}
			\xi(\omega_1, x_1), \ldots, \xi(\omega_N, x_1),
			\\
			\cdots  
			\\
			\xi(\omega_1, x_m), \ldots, \xi(\omega_N, x_m)
		\end{pmatrix}
		= [\zbf(\omega_1), \ldots \zbf(\omega_N)] \in \R^{m \times N}.
	\end{align}
Here
$\zbf(\omega) = (\z_i(\omega))_{i=1}^m = (\xi(\omega, x_i))_{i=1}^m \in \R^m$.
Since $(K[\Xbf])_{ij} = \bE[\xi(\omega,x_i)\xi(\omega, x_j)]$, $1\leq i,j\leq m$,
	\begin{align}
		\label{equation:K-exact-W}
		K[\Xbf] = \bE[\zbf(\omega)\zbf(\omega)^T] = \int_{\Omega}\zbf(\omega)\zbf(\omega)^TdP(\omega).
	\end{align}
The empirical version of $K[\Xbf]$, using the random sample $\Wbf = (\omega_i)_{i=1}^N$, is then
	\begin{align}
		\label{equation:K-hat-W}
		\hat{K}_{\Wbf}[\Xbf] = \frac{1}{N}\sum_{i=1}^N \zbf(\omega_i)\zbf(\omega_i)^T = \frac{1}{N}\Zbf\Zbf^T.
	\end{align}
The convergence of $\hat{K}_{\Wbf}[\Xbf]$ to $K[\Xbf]$ is given by the following.
\begin{proposition}
	[\citep{Minh2021:FiniteEntropicGaussian}]
	\label{proposition:concentration-empirical-covariance}
	Assume Assumptions B1-B5.
	Let $\xi \sim \GP(0,K)$ 
	on $(\Omega, \Fcal,P)$.
	Let $\Xbf \in T^m$ be fixed. Then
	$||K[\Xbf]||_F \leq m\kappa^2$.
	Let $\Wbf = (\omega_1, \ldots, \omega_N)$ be independently sampled from $(\Omega,P)$. 
	For any 
	$0 < \delta < 1$, with probability at least $1-\delta$, 
	{
		\begin{align}
			&||\hat{K}_{\Wbf}[\Xbf] - K[\Xbf]||_F \leq \frac{2\sqrt{3}m\kappa^2}{\sqrt{N}\delta},
			\;\;
			||\hat{K}_{\Wbf}[\Xbf]||_F \leq \frac{2m\kappa^2}{\delta}.
		\end{align}
	}
\end{proposition}
Let now $\xi^i \sim \GP(0, K^i)$, $i=1,2$, 
on the probability spaces $(\Omega_i, \Fcal_i, P_i)$, respectively.
Let $\Wbf^i = (\omega^i_j)_{j=1}^N$, be independently sampled from $(\Omega_i, P_i)$, corresponding to the sample paths
$\{\xi^i_j(t) = \xi^i(\omega_j,t)\}_{j=1}^N$, $t \in T$, from $\xi^i$, $i=1,2$.
Combining Proposition \ref{proposition:concentration-empirical-covariance} and
Theorem \ref{theorem:logdet-approx-infinite-sequence},
we obtain the following 
empirical estimate of 
$\KL^{\gamma}\left[\Ncal\left(0, \frac{1}{m}K^1[\Xbf]\right), \Ncal\left(0, \frac{1}{m}K^2[\Xbf]\right)\right]$
from two finite samples  of 
$\xi^1\sim \GP(0, K^1)$ and $\xi^2\sim \GP(0,K^2)$ given 
by $\Wbf^1, \Wbf^2$.

{\it For simplicity and without loss of generality, in the following theoretical analysis  we let $N_1 = N_2 = N$,
	where $N_1,N_2$ are the numbers of 
random sample paths drawn from the Gaussian processes $\xi^1,\xi^2$, respectively}.
\begin{theorem}
	\label{theorem:KL-estimate-unknown-1}
	Assume Assumptions B1-B5.
	Let $\Xbf = (x_i)_{i=1}^m \in T^m$, $m \in \Nbb$ be fixed.
	Let $\Wbf^1 = (\omega_j^1)_{j=1}^N$, $\Wbf^2 = (\omega_j^2)_{j=1}^N$ be independently sampled from
	$(\Omega_1, P_1)$ and $(\Omega_2, P_2)$, respectively.
	For any $0 < \delta<1$, with probability at least $1-\delta$,
		\begin{align}
			&\left|\KL^{\gamma}\left[\Ncal\left(0, \frac{1}{m}\hat{K}^1_{\Wbf^1}[\Xbf]\right), \Ncal\left(0, \frac{1}{m}\hat{K}^2_{\Wbf^2}[\Xbf]\right)\right]
			\right.
			\nonumber
			\\
			&\quad \quad\left.
			- \KL^{\gamma}\left[\Ncal\left(0, \frac{1}{m}K^1[\Xbf]\right), \Ncal\left(0, \frac{1}{m}K^2[\Xbf]\right)\right]
			\right| \leq \frac{D}{\gamma^2\sqrt{N}\delta},
			\end{align}
	where the constant factor $D$ is given by
	\begin{align}
	D =4\sqrt{3} \left(1+\frac{4\kappa_2^2}{\delta\gamma}\right)\left(1+\frac{\kappa_2^2}{\gamma}\right)\left[\kappa_1^2 + \left(1 + \frac{2\kappa_1^2}{\delta\gamma} + \frac{\kappa_1^2}{2\gamma}\right)\kappa_2^2\right]\left[
	(\kappa_1^2 + \kappa_2^2)\left(1+\frac{4}{\delta}\right) +\frac{4}{\delta\gamma}(\kappa_1^2 + \kappa_2^2)^2\right].	
	\end{align}
	Here the probability is with respect to the product space $(\Omega_1,P_1)^N \times (\Omega_2,P_2)^N$.
\end{theorem}

Combing Theorems \ref{theorem:KL-estimate-unknown-1} and \ref{theorem:KL-estimate-finite-covariance}, we obtain the following empirical estimate
of the
theoretical 
regularized KL divergence $\KL^{\gamma}[\Ncal(0,C_{K^1}), \Ncal(0,C_{K^2})]$
from two finite samples $\Zbf^1,\Zbf^2$ of $\xi^1 \sim \GP(0,K^1)$ and $\xi^2 \sim \GP(0,K^2)$.

\begin{theorem}
	(\textbf{Estimation of regularized KL divergence between Gaussian processes from finite samples})
	\label{theorem:KL-estimate-unknown-2}
	Assume Assumptions B1-B5.
	Let $\Xbf = (x_i)_{i=1}^m$
	be independently sampled from $(T, \nu)$.
	Let $\Wbf^1 = (\omega_j^1)_{j=1}^N$, $\Wbf^2 = (\omega_j^2)_{j=1}^N$ be independently sampled from
	$(\Omega_1, P_1)$ and $(\Omega_2, P_2)$, respectively.
	For any $0 < \delta<1$, with probability at least $1-\delta$,
		\begin{align}
			&\left|\KL^{\gamma}\left[\Ncal\left(0, \frac{1}{m}\hat{K}^1_{\Wbf^1}[\Xbf]\right), \Ncal\left(0, \frac{1}{m}\hat{K}^2_{\Wbf^2}[\Xbf]\right)\right]
			\right.
			\nonumber 
			\\
			&\quad\left.- \KL^{\gamma}[\Ncal(0,C_{K^1}), \Ncal(0,C_{K^2})]\right|
			\nonumber
			\\
			&
			\leq \frac{2D}{\gamma^2\sqrt{N}\delta} +
			\frac{1}{2\gamma^2}\left[\kappa_1^4 + \kappa_2^4 + \kappa_1^2\kappa_2^2\left(2+\frac{\kappa_2^2}{\gamma}\right)\right]\left(\frac{2\log\frac{12}{\delta}}{m} + \sqrt{\frac{2\log\frac{12}{\delta}}{m}}\right),
		\end{align}
	where the constant factor $D$ is given by
\begin{align}
D = 4\sqrt{3} \left(1+\frac{8\kappa_2^2}{\delta\gamma}\right)\left(1+\frac{\kappa_2^2}{\gamma}\right)\left[\kappa_1^2 + \left(1 + \frac{4\kappa_1^2}{\delta\gamma} + \frac{\kappa_1^2}{2\gamma}\right)\kappa_2^2\right]\left[
(\kappa_1^2 + \kappa_2^2)\left(1+\frac{8}{\delta}\right) +\frac{8}{\delta\gamma}(\kappa_1^2 + \kappa_2^2)^2\right].
\end{align}
	Here the probability is with respect to the space $(T,\nu)^m \times (\Omega_1,P_1)^N \times (\Omega_2,P_2)^N$.
\end{theorem}

%

Entirely similar results can be obtained for $D^{\gamma}_{\Rrm,r}$, $0 < r < 1$.
Theorem \ref{theorem:KL-estimate-unknown-2} and similar results for $D^{\gamma}_{\Rrm,r}$, $0 < r < 1$ provide the
theoretical basis for the following Algorithm \ref{algorithm:estimate-divergence-finite-sample} for estimating the regularized R\'enyi and KL divergences between two covariance operators using finite samples generated by two Gaussian processes.

\begin{algorithm}
	\caption{Consistent estimates of R\'enyi and KL divergences between covariance operators/centered Gaussian processes
		from finite samples}
	\label{algorithm:estimate-divergence-finite-sample}
	\begin{algorithmic}
		\REQUIRE Finite samples $\{\xi^i_k(x_j)\}$,
		from $N_i$ realizations $\xi^i_k$, $1 \leq k \leq N_i$, of processes $\xi^i$, $i=1,2$, sampled at $m$ points $x_j$, $1 \leq j \leq m$  (note that we set $N_1 = N_2 = N$ in the theoretical analysis only for simplicity) 
		\REQUIRE Regularization parameter $\gamma > 0$, $0< r < 1$ (for the R\'enyi divergence)
		
		\textbf{Procedure}
		\STATE{Form $m \times N_i$ data matrices $Z_i$ , with $(Z_i)_{jk} = \xi^i_k(x_j)$, $i=1,2$, $1\leq j \leq m, 1\leq k \leq N_i$}
		\STATE{Compute $m \times m$ empirical covariance matrices $\hat{K}^i = \frac{1}{N}Z_iZ_i^{T}$, $i=1,2$} 
		\STATE{Compute $D^{\gamma}_{\Rrm,r} = \frac{1}{2}d^{2r-1}_{\logdet}\left[\left(\gamma I_m+ \frac{1}{m}\hat{K}^1\right), \left(\gamma I_m+ \frac{1}{m}\hat{K}^2\right)\right]=D_{\Rrm,r}[\Ncal(0,\hat{K}^1+m\gamma I_m)||\Ncal(0, \hat{K^2}+m\gamma I_m)]$}
		\STATE{Compute $\KL^{\gamma} = \frac{1}{2}d^{1}_{\logdet}\left[\left(\gamma I_m+ \frac{1}{m}\hat{K}^1\right), \left(\gamma I_m+ \frac{1}{m}\hat{K}^2\right)\right]=\KL[\Ncal(0,\hat{K}^1+m\gamma I_m)||\Ncal(0, \hat{K^2}+m\gamma I_m)]$}
		\RETURN{$D^{\gamma}_{\Rrm,r}$ and $\KL^{\gamma}$}
	\end{algorithmic}
\end{algorithm}

\section{The general Gaussian process setting}
\label{section:Gaussian-process-full}

In Section \ref{section:estimate-Gaussian-process}, we focused on the 
setting of zero-mean Gaussian processes. We now extend it to the general Gaussian process setting,
that is together with the mean functions.
Generalizing Assumption B4, we assume the following

\begin{align}
\text{(\bf Assumption B6)}: \xi \sim \GP(\mu,K), \xi^i \sim \GP(\mu^i, K^i), \;\;\text{where $\mu, \mu^i \in \Lcal^2(T,\nu)$}, i=1,2.
\end{align}
Along with Assumption B3, this ensures that the sample paths of $\xi$ are in $\Lcal^2(T,\nu)$ almost surely.
Thus $\xi$ induces on $\Lcal^2(T,\nu)$ the Gaussian measure $\Ncal(\mu, C_K)$.

Let $\mu[\Xbf] = (\mu(x_j))_{j=1}^m \in \R^m$, then $(\xi(\cdot, x_j))_{j=1}^m$ is a random vector in $\R^m$, distributed according to the Gaussian measure $\Ncal(\mu[\Xbf], K[\Xbf])$. In the following, we show that
\begin{align*}
D^{\gamma}_{\Rrm,r}[\Ncal(\mu^1, C_{K^1})||\Ncal(\mu^2, C_{K^2})], \;
\KL^{\gamma}[\Ncal(\mu^1, C_{K^1})||\Ncal(\mu^2, C_{K^2})], 
\end{align*}
are consistently estimated, respectively, by
\begin{align*}
&D_{\Rrm,r}^{\gamma}\left[\Ncal\left(\frac{1}{\sqrt{m}}\mu^1[\Xbf], \frac{1}{m}K^1[\Xbf]\right)||\Ncal\left(\frac{1}{\sqrt{m}}\mu^2[\Xbf], \frac{1}{m}K^2[\Xbf]\right)\right]
\\
& \quad = D_{\Rrm,r}[\Ncal(\mu^1[\Xbf],K^1[\Xbf] + m\gamma I_m)||\Ncal(\mu^2[\Xbf],K^2[\Xbf] + m\gamma I_m)],
\\
&\KL^{\gamma}\left[\Ncal\left(\frac{1}{\sqrt{m}}\mu^1[\Xbf], \frac{1}{m}K^1[\Xbf]\right)||\Ncal\left(\frac{1}{\sqrt{m}}\mu^2[\Xbf], \frac{1}{m}K^2[\Xbf]\right)\right]
\\
&\quad = \KL[\Ncal(\mu^1[\Xbf],K^1[\Xbf] + m\gamma I_m)||\Ncal(\mu^2[\Xbf],K^2[\Xbf] + m\gamma I_m)].
\end{align*}
\begin{remark} We note that if $\dim(\H) < \infty$, then 
	\begin{align*}
	D_{\Rrm,r}^{\gamma}[\Ncal(\mu^1, C^1)||\Ncal(\mu^2, C^2)] &= D_{\Rrm,r}[\Ncal(\mu^1, C^1+\gamma I)||\Ncal(\mu^2, C^2+\gamma I)],
	\\
	\KL^{\gamma}[\Ncal(\mu^1, C^1)||\Ncal(\mu^2, C^2)] &= \KL[\Ncal(\mu^1, C^1+\gamma I)||\Ncal(\mu^2, C^2+\gamma I)].
	\end{align*}
The right hand side are not well-defined, however, when $\dim(\H) = \infty$, since the identity operator is not trace class.
\end{remark}
Having already estimated the divergences between the covariance operators, by Definition \ref{definition:renyi-regularized}, it remains to estimate the mean term. We now show that
the theoretical mean term
\begin{align}
\label{equation:mean-term-Gaussian}
	&\la \mu^1-\mu^2, [(1-r)(C_{K^1} + \gamma I) + r(C_{K^2}+\gamma I)]^{-1}(\mu^1 - \mu^2)\ra_{\Lcal^2(T,\nu)}
\end{align}
is consistently estimated by the empirical mean term
\begin{align}
\label{equation:mean-term-Gaussian-empirical}
&\frac{1}{m}\la \mu^1[\Xbf] - \mu^2[\Xbf], \left[(1-r)\left(\frac{1}{m}K^1[\Xbf]+\gamma I\right) + r\left(\frac{1}{m}K^2[\Xbf] + \gamma I\right)\right]^{-1}(\mu^1[\Xbf] - \mu^2[\Xbf])\ra_{\R^m}
\nonumber
\\
&=\la \mu^1[\Xbf] - \mu^2[\Xbf], [(1-r)(K^1[\Xbf]+m\gamma I) + r(K^2[\Xbf] + m\gamma I)]^{-1}(\mu^1[\Xbf] - \mu^2[\Xbf])\ra_{\R^m}.
\end{align}
The quantity in Eq.\eqref{equation:mean-term-Gaussian} is defined on $\Lcal^2(T,\nu)$ whereas 
that in Eq.\eqref{equation:mean-term-Gaussian-empirical} is defined on $\R^m$, thus a direct convergence analysis is not straightforward. As in Section \ref{section:estimate-Gaussian-process}, their representation using RKHS covariance and cross-covariance operators, defined on the same RKHSs, 
plays a crucial role for the analysis of consistency.

\begin{proposition}
	(\textbf{Representation of mean term in R\'enyi divergences via RKHS operators})
\label{proposition:Renyi-mean-RKHS-representation}
Assume Assumptions B1-B3 and B6. Let $0 \leq r \leq 1$ be fixed.
Let $\ubf_1 = \sqrt{1-r}R_{K^1}(\mu^1-\mu^2) \in \H_{K^1}$, $\ubf_2 = \sqrt{r}R_{K^2}(\mu^1-\mu^2)\in \H_{K^2}$. Then
\begin{align}
	&\la \mu^1-\mu^2, [(1-r)(C_{K^1} + \gamma I) + r(C_{K^2}+\gamma I)]^{-1}(\mu^1 - \mu^2)\ra_{\Lcal^2(T,\nu)}
	\nonumber
	\\
	& = \frac{1}{\gamma}||\mu^1-\mu^2||^2_{\Lcal^2(T,\nu)} 
\\
&\quad - \frac{1}{\gamma^2}\left\la
\begin{pmatrix}
	\ubf_1
	\\
	\ubf_2
\end{pmatrix},
\left[\frac{1}{\gamma}\begin{pmatrix} (1-r)L_{K^1} & \sqrt{r(1-r)}R_{12}
	\\
	\sqrt{r(1-r)}R_{12}^{*} & r L_{K^2}
\end{pmatrix}
+ I_{\H_{K^1} \oplus \H_{K^2}}\right]^{-1}\begin{pmatrix}
	\ubf_1
	\\
	\ubf_2
\end{pmatrix}\right\ra_{\H_{K^1}\oplus \H_{K^2}}.
\nonumber
\end{align} 
Let $\ubf_{1,\Xbf} = \frac{\sqrt{1-r}}{m}S_{1,\Xbf}^{*}(\mu^1[\Xbf] - \mu^2[\Xbf]) \in \H_{K^1}$,
$\ubf_{2,\Xbf} = \frac{\sqrt{r}}{m}S_{2,\Xbf}^{*}(\mu^1[\Xbf] - \mu^2[\Xbf]) \in \H_{K^2}$.
The empirical version is represented by
\begin{align}
	&\la \mu^1[\Xbf] - \mu^2[\Xbf], [(1-r)(K^1[\Xbf]+m\gamma I) + r(K^2[\Xbf] + m\gamma I)]^{-1}(\mu^1[\Xbf] - \mu^2[\Xbf])\ra_{\R^m}
	\nonumber
	\\
	&= \frac{1}{m\gamma}||\mu^1[\Xbf] -\mu^2[\Xbf]||^2_{\R^m} 
	\\
	& \quad - \frac{1}{\gamma^2}\left\la \begin{pmatrix}
		\ubf_{1,\Xbf}
		\\
		\ubf_{2,\Xbf}
	\end{pmatrix} \left[\frac{1}{\gamma}\begin{pmatrix}
		(1-r)L_{K^1,\Xbf} & \sqrt{r(1-r)}R_{12,\Xbf}
		\\
		\sqrt{r(1-r)}R_{12,\Xbf}^{*} & rL_{K^2,\Xbf}
	\end{pmatrix} + I_{\H_{K^1}\oplus \H_{K^2}}\right]^{-1} \begin{pmatrix}
		\ubf_{1,\Xbf}
		\\
		\ubf_{2,\Xbf}
	\end{pmatrix}\right\ra_{\H_{K^1} \oplus \H_{K^2}}.
\nonumber
\end{align}
\end{proposition}
In particular, for $r=1$, which corresponds to the KL divergence,
\begin{align}
	&\la \mu^1-\mu^2, (C_{K^2}+\gamma I_{\Lcal^2(T,\nu)})^{-1}(\mu^1 - \mu^2)\ra_{\Lcal^2(T,\nu)}
	\\ 
	&=\frac{1}{\gamma}||\mu^1-\mu^2||^2_{\Lcal^2(T,\nu)} - \frac{1}{\gamma^2}\left\la R_{K^2}(\mu^1-\mu^2), \left(\frac{1}{\gamma}L_{K^2}+ I_{\H_{K^2}}\right)^{-1}R_{K^2}(\mu^1-\mu^2)\right\ra_{\H_{K^2}}.
	\nonumber
\end{align}
The empirical version has the representation
\begin{align}
	&\la \mu^1[\Xbf] - \mu^2[\Xbf], (K^2[\Xbf] + m\gamma I_m)^{-1}(\mu^1[\Xbf]-\mu^2[\Xbf])\ra_{\R^m}
	\\
	&= \frac{1}{m\gamma}||\mu^1[\Xbf]-\mu^2[\Xbf]||^2_{\R^m} - \frac{1}{\gamma^2}\left\la \frac{1}{m}S_{2,\Xbf}^{*}(\mu^1[\Xbf] - \mu^2[\Xbf]), \left(\frac{1}{\gamma}L_{K^2,\Xbf} +  I_{\H_{K^2}}\right)^{-1}\frac{1}{m}S_{2,\Xbf}^{*}(\mu^1_{\Xbf}-\mu^2_{\Xbf})\right\ra_{\H_{K^2}}.
	\nonumber
\end{align}

With the RKHS representation, we now provide the sample complexity for the finite-dimensional estimation
of the infinite-dimensional mean term. For simplicity, we make the following additional assumption
\begin{align}
	\text{(\bf Assumption B7)}: \exists B, B_i > 0 \;\text{such that }||\mu||_{\infty} \leq B, ||\mu^{i}||_{\infty}\leq B_i, \; i=1,2. 
\end{align}
\begin{proposition}
	(\textbf{Estimation of mean term for R\'enyi divergences})
	\label{proposition:Renyi-mean-Gaussian-process-estimate}
	Assume Assumptions B1-B3 and B5-B7. Let $0 \leq r \leq 1$ be fixed. Let $\Xbf = (x_j)_{j=1}^m$ be independently sampled from $(T,\nu)$.
	For any $0 < \delta < 1$, with probability at least $1-\delta$,
	\begin{align}
		&\left|\la \mu^1[\Xbf] - \mu^2[\Xbf], [(1-r)(K^1[\Xbf]+m\gamma I) + r(K^2[\Xbf] + m\gamma I)]^{-1}(\mu^1[\Xbf] - \mu^2[\Xbf])\ra_{\R^m}\right.
		\nonumber
		\\
	&\quad \left.	-\la \mu^1-\mu^2, [(1-r)(C_{K^1} + \gamma I) + r(C_{K^2}+\gamma I)]^{-1}(\mu^1 - \mu^2)\ra_{\Lcal^2(T,\nu)}\right|
	\nonumber
	\\
	&	\leq \frac{1}{\gamma}(B_1+B_2)^2\left[1 + \frac{(1-r)\kappa_1^2 + r\kappa_2^2}{\gamma}\right]^2\left(\frac{2\log\frac{24}{\delta}}{m} + \sqrt{\frac{2\log\frac{24}{\delta}}{m}}\right).
	\end{align}
\end{proposition}

\begin{theorem}
	(\textbf{Estimation of R\'enyi divergences between general Gaussian processes})
	\label{theorem:Renyi-estimate-Gaussian-process-general-case}
	Assume Assumptions B1-B3 and B5-B7. Let $0 < r < 1$ be fixed. Let $\Xbf = (x_j)_{j=1}^m$ be independently sampled from $(T,\nu)$.
	For any $0 < \delta < 1$, with probability at least $1-\delta$,
	\begin{align}
		&\left| D_{\Rrm,r}[\Ncal(\mu^1[\Xbf], K^1[\Xbf] + m\gamma I)||\Ncal(\mu^1[\Xbf], K^1[\Xbf] + m\gamma I) ]\right.
		\nonumber
		\\
		&\quad -\left. D^{\gamma}_{\Rrm,r}[\Ncal(\mu^1, C_{K^1})||\Ncal(\mu^2, C_{K^2})]\right|
		\nonumber
		\\
		& \leq \frac{1}{2\gamma}(B_1+B_2)^2\left[1 + \frac{(1-r)\kappa_1^2 + r\kappa_2^2}{\gamma}\right]^2\left(\frac{2\log\frac{48}{\delta}}{m} + \sqrt{\frac{2\log\frac{48}{\delta}}{m}}\right)
		\nonumber
		\\
		&\quad +\frac{1}{2\gamma^2}\left[\frac{\kappa_1^4}{r} + \frac{\kappa_2^4}{1-r} + \frac{[(1-r)\kappa_1^2 + r\kappa_2^2]^2}{r(1-r)}\right]\left(\frac{2\log\frac{12}{\delta}}{m} + \sqrt{\frac{2\log\frac{12}{\delta}}{m}}\right).
	\end{align}
\end{theorem}

\begin{theorem}
	(\textbf{Estimation of KL divergence between general Gaussian processes})
	\label{theorem:KL-estimate-Gaussian-process-general-case}
	Assume Assumptions B1-B3 and B5-B7. Let $\Xbf = (x_j)_{j=1}^m$ be independently sampled from $(T,\nu)$.
	For any $0 < \delta < 1$, with probability at least $1-\delta$,
	\begin{align}
		&\left| \KL[\Ncal(\mu^1[\Xbf], K^1[\Xbf] + m\gamma I)||\Ncal(\mu^1[\Xbf], K^1[\Xbf] + m\gamma I) ]\right.
		\nonumber
		\\
		&\quad- \left. \KL^{\gamma}[\Ncal(\mu^1, C_{K^1})||\Ncal(\mu^2, C_{K^2})]\right|
		\nonumber
		\\
		& \leq \frac{1}{2\gamma}(B_1+B_2)^2\left[1 + \frac{\kappa_2^2}{\gamma}\right]^2\left(\frac{2\log\frac{48}{\delta}}{m} + \sqrt{\frac{2\log\frac{48}{\delta}}{m}}\right)
		\nonumber
		\\
		& \quad +  \frac{1}{2\gamma^2}\left[\kappa_1^4 + \kappa_2^4 + \kappa_1^2\kappa_2^2\left(2+\frac{\kappa_2^2}{\gamma}\right)\right]\left(\frac{2\log\frac{12}{\delta}}{m} + \sqrt{\frac{2\log\frac{12}{\delta}}{m}}\right).
	\end{align}
\end{theorem}
Similar to the centered Gaussian setting, if $B_1,B_2, \kappa_1,\kappa_2$ are absolute constants, then
the sample complexities in Theorems \ref{theorem:Renyi-estimate-Gaussian-process-general-case} and \ref{theorem:KL-estimate-Gaussian-process-general-case} are completely {\it dimension-independent}.
The following algorithms, which generalize Algorithms \ref{algorithm:estimate-divergence-finite-COV}
and \ref{algorithm:estimate-divergence-finite-sample}, respectively, consistently estimate the  R\'enyi and KL divergences between general Gaussian processes
from finite-dimensional Gaussian measures and finite samples generated by the processes.
\begin{algorithm}
	\caption{Consistent estimates of R\'enyi and KL divergences between general Gaussian processes
		from finite-dimensional Gaussian measures}
	\label{algorithm:estimate-divergence-finite-COV-full}
	\begin{algorithmic}
		\REQUIRE Finite covariance matrices $K^i[\Xbf]$ from processes $\xi^i$, $i=1,2$, sampled at $m$ points
		$\Xbf = (x_j)_{j=1}^m$
		\REQUIRE Mean vectors $\mu^i[\Xbf]$ on the points $\Xbf = (x_j)_{j=1}^m$, i=1,2
		\REQUIRE Regularization parameter $\gamma > 0$, $0< r < 1$ (for the R\'enyi divergence)
		
		\textbf{Procedure}
		\STATE{Compute $D^{\gamma}_{\Rrm,r} 
			 = D_{\Rrm,r}[\Ncal(\mu^1[\Xbf],K^1[\Xbf]+m\gamma I_m)||\Ncal(\mu^2[\Xbf],K^2[\Xbf]+m\gamma I_m)]$}
		\STATE{Compute $\KL^{\gamma}
			 = \KL[\Ncal(\mu^1[\Xbf],K^1[\Xbf]+m\gamma I_m)||\Ncal(\mu^2[\Xbf],K^2[\Xbf]+m\gamma I_m)]$}
		\RETURN{$D^{\gamma}_{\Rrm,r}$ and $\KL^{\gamma}$}
	\end{algorithmic}
\end{algorithm}

\begin{algorithm}
	\caption{Consistent estimates of R\'enyi and KL divergences between general Gaussian processes
		from finite samples}
	\label{algorithm:estimate-divergence-finite-sample-full}
	\begin{algorithmic}
		\REQUIRE Finite samples $\{\xi^i_k(x_j)\}$,
		from $N_i$ realizations $\xi^i_k$, $1 \leq k \leq N_i$, of processes $\xi^i$, $i=1,2$, sampled at $m$ points $x_j$, $1 \leq j \leq m$  (note that we set $N_1 = N_2 = N$ in the theoretical analysis only for simplicity) 
		\REQUIRE Regularization parameter $\gamma > 0$, $0< r < 1$ (for the R\'enyi divergence)
		
		\textbf{Procedure}
		\STATE{Form $m \times N_i$ data matrices $Z_i$ , with $(Z_i)_{jk} = \xi^i_k(x_j)$, $i=1,2$, $1\leq j \leq m, 1\leq k \leq N_i$}
		\STATE{Compute $m \times 1$ empirical mean vectors $\hat{\mu}^i = \frac{1}{N}Z_i\1_N$, $i=1,2$} 
		\STATE{Compute $m \times m$ empirical covariance matrices $\hat{K}^i = \frac{1}{N}Z_iJ_NZ_i^{T}$, $i=1,2$, where $J_N = I_N - \frac{1}{N}\1_N\1_N^T$} 
		\STATE{Compute $D^{\gamma}_{\Rrm,r}
			 =D_{\Rrm,r}[\Ncal(\hat{\mu}^1,\hat{K}^1+m\gamma I_m)||\Ncal(\hat{\mu}^2, \hat{K^2}+m\gamma I_m)]$}
		\STATE{Compute $\KL^{\gamma}
			=\KL[\Ncal(\hat{\mu}^1,\hat{K}^1+m\gamma I_m)||\Ncal(\hat{\mu}^2, \hat{K^2}+m\gamma I_m)]$}
		\RETURN{$D^{\gamma}_{\Rrm,r}$ and $\KL^{\gamma}$}
	\end{algorithmic}
\end{algorithm}
\begin{remark}
In Algorithms \ref{algorithm:estimate-divergence-finite-sample} and \ref{algorithm:estimate-divergence-finite-sample-full}, we have presented the standard empirical sample mean and sample covariance matrix. They can certainly be substituted by better estimations of the mean and covariance functions. We will explore this in a future work.
\end{remark}

\section{Numerical experiments}
\label{section-experiments}

In this section, we illustrate the above theoretical analysis with several experiments on Gaussian measures on RKHS and Gaussian processes.

Figure \ref{figure:RKHS-divergence} illustrates Algorithm \ref{algorithm:estimate-divergence-RKHS}.
Here we generated two data samples $\Xbf^1,\Xbf^2$ from two random Gaussian mixtures in $\R^n$ for two values of $n$, namely $n=5$ and $n=100$.
We then computed the regularized KL and R\'enyi divergences in RKHS between them using the Gaussian kernel
$K(x,y) = \exp(-||x-y||^2)$, with $\gamma = 10^{-3}$. The dimension-independent convergence is clearly shown here, in agreement with theory.
  
Figures \ref{figure:Gaussian-process-sample-1} and \ref{figure:Gaussian-process-sample-2} illustrate Algorithm \ref{algorithm:estimate-divergence-finite-COV}. In Figure \ref{figure:Gaussian-process-sample-1},
we show the samples generated from the Gaussian processes
$\GP(0,K^1)$, $\GP(0,K^2)$, where {$K^1(x,y) = \exp(-a||x-y||)$, $a=1$}, {$K^2(x,y) = \exp\left(-\frac{||x-y||^2}{\sigma^2}\right)$, $\sigma = 0.1$}
on {$T = [0,1]$} and the approximate regularized divergences $\KL^{\gamma}$ and $D^{\gamma}_{\Rrm,r}$, $r=0.5$, $r=0.75$, between them using finite-dimensional Gaussian measures.
Here the 
number of sample points is $m = 10, 20, \ldots, 1000$, and regularization parameter is {$\gamma = 10^{-6}$}.
Figure \ref{figure:Gaussian-process-sample-2} shows a similar set up, except that $K^2(x,y) = \exp(-a||x-y||)$ with $a = 1.2$.
The empirical values of the regularized divergences all converge as $m$ gets larger, in agreement with theory.
One can also see that on Figure \ref{figure:Gaussian-process-sample-2}, where the Gaussian processes are much closer to each other than in Figure \ref{figure:Gaussian-process-sample-1}, the divergences have much smaller values, as should be expected.
In Figure \ref{figure:Gaussian-process-sample-3}, we repeat this experiment, with $\xi^2 \sim \GP(\mu^2,K^2)$, where $\mu^2(x) =1+x$. The values of the divergences increase by small amounts from the mean term contributions, as expected.

\begin{figure}[!t]
	\begin{subfigure}{6cm}
		\includegraphics[width=6cm]{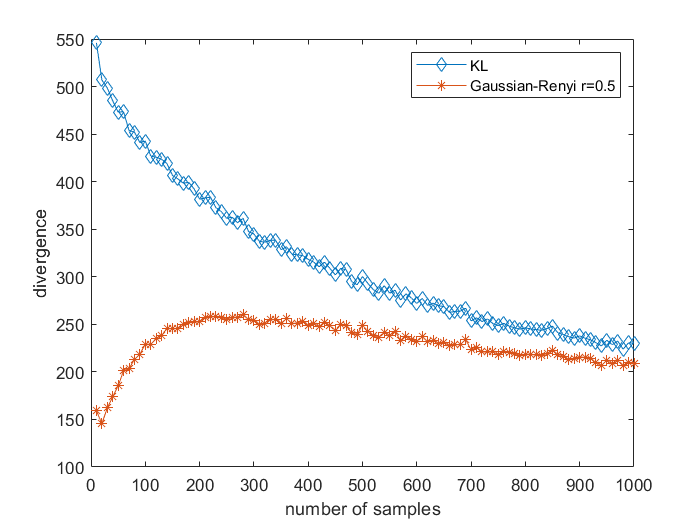}
	\end{subfigure}
	\begin{subfigure}{6cm}
		\includegraphics[width=6cm]{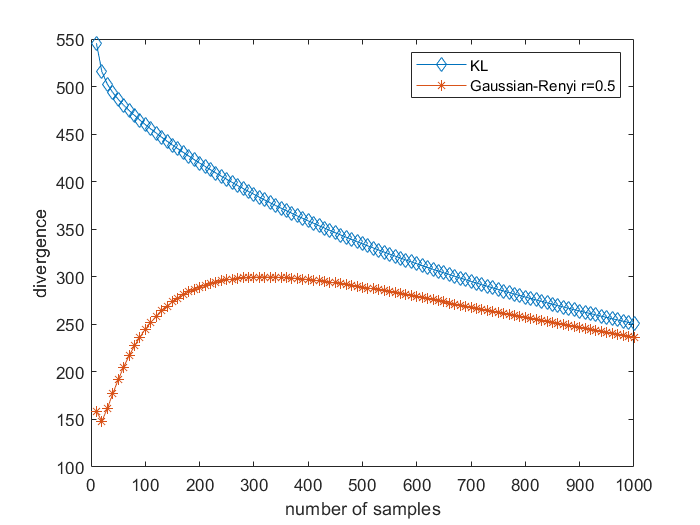}
	\end{subfigure}
	\caption{Regularized KL and R\'enyi divergences in RKHS between two random Gaussian mixtures in $\R^n$.
		Left: $n=5$, right: $n=100$. Here $K$ is the Gaussian kernel $K(x,y) = \exp(-||x-y||^2)$ and the regularization parameter is $\gamma = 10^{-3}$. The numbers of sample points are $m_1 = m_2 = m =10,20, \ldots, 1000$.}
	\label{figure:RKHS-divergence}
\end{figure}

\begin{figure}[!t]
	\begin{subfigure}{5cm}
		\includegraphics[width=5cm]{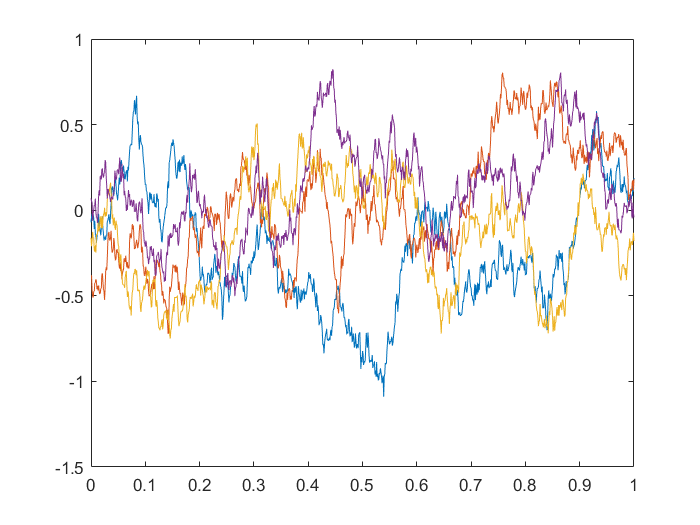}
	\end{subfigure}
	\begin{subfigure}{5cm}
		\includegraphics[width=5cm]{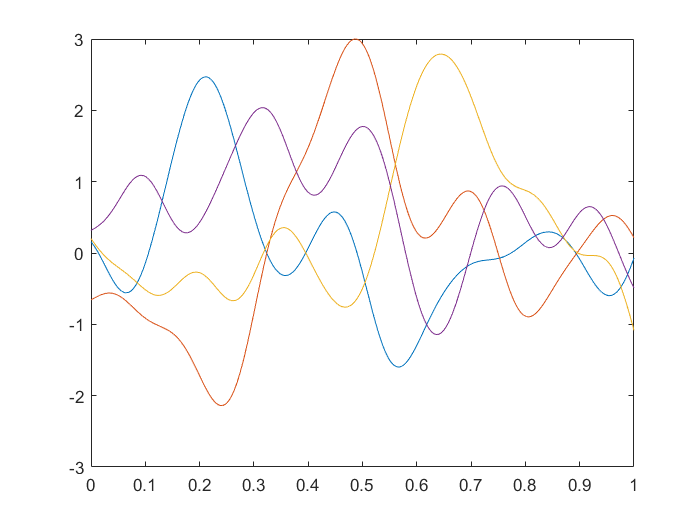}
	\end{subfigure}
	\begin{subfigure}{5cm}
		\includegraphics[width=5cm]{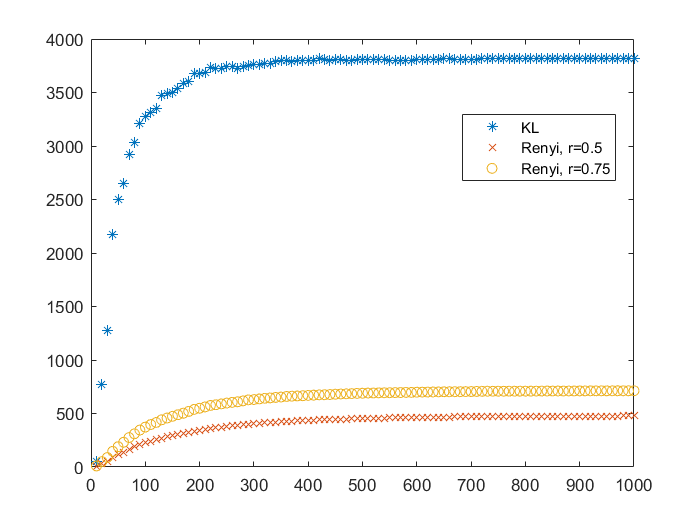}
	\end{subfigure}
	\caption{Samples of the centered Gaussian processes {$\GP(0,K^1)$, $\GP(0,K^2)$} 
		on {$T = [0,1]$} and approximate divergences between them using finite-dimensional Gaussian measures.
		Left: {$K^1(x,y) = \exp(-a||x-y||)$, $a=1$}. Middle: {$K^2(x,y) = \exp(-||x-y||^2/\sigma^2)$, $\sigma = 0.1$}.
		Here the 
		number of sample points is $m = 10, 20, \ldots, 1000$, and regularization parameter is {$\gamma = 10^{-6}$}}
	\label{figure:Gaussian-process-sample-1}
\end{figure}

\begin{figure}[!t]
	\begin{subfigure}{5cm}
		\includegraphics[width=5cm]{images/LaplacianKernel_GaussianProcess.png}
	\end{subfigure}
	\begin{subfigure}{5cm}
		\includegraphics[width=5cm]{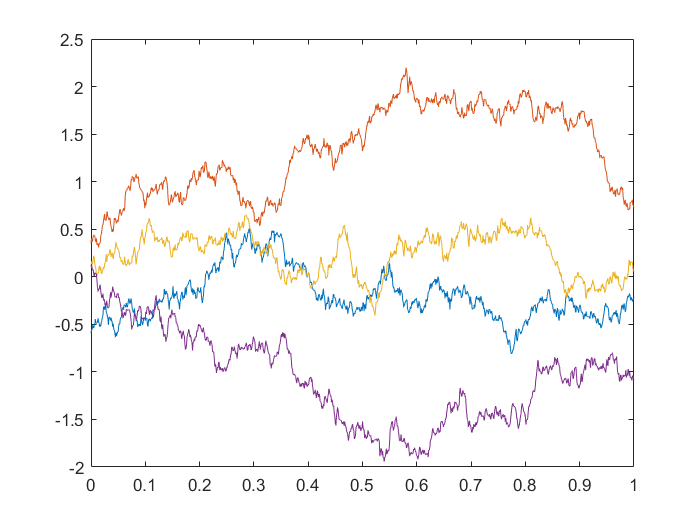}
	\end{subfigure}
	\begin{subfigure}{5cm}
		\includegraphics[width=5cm]{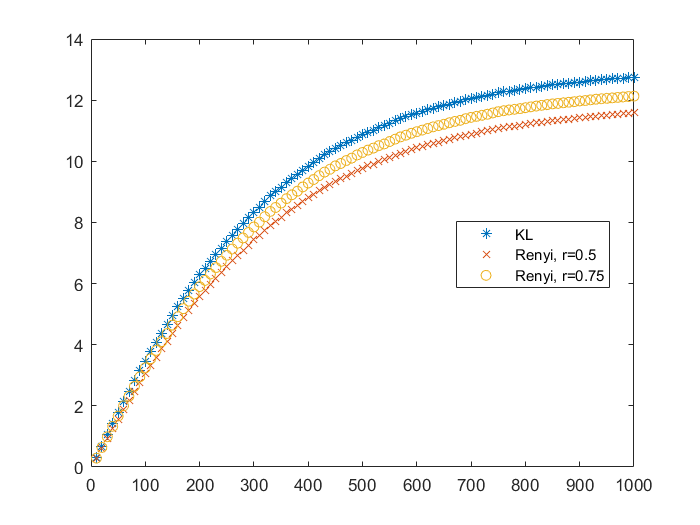}
	\end{subfigure}
	\caption{Samples of the centered Gaussian processes {$\GP(0,K^1)$, $\GP(0,K^2)$} 
		on {$T = [0,1]$} and approximate divergences between them, using finite-dimensional Gaussian measures.
		Left: {$K^1(x,y) = \exp(-a||x-y||)$, $a=1$}. Middle: {$K^2(x,y) = \exp(-a||x-y||)$, $a = 1.2$}.
		Here the 
		number of sample points is $m = 10,20, \ldots, 1000$, and regularization parameter is {$\gamma = 10^{-6}$}}
	\label{figure:Gaussian-process-sample-2}
\end{figure}

\begin{figure}[!t]
	\begin{subfigure}{5cm}
		\includegraphics[width=5cm]{images/LaplacianKernel_GaussianProcess.png}
	\end{subfigure}
	\begin{subfigure}{5cm}
		\includegraphics[width=5cm]{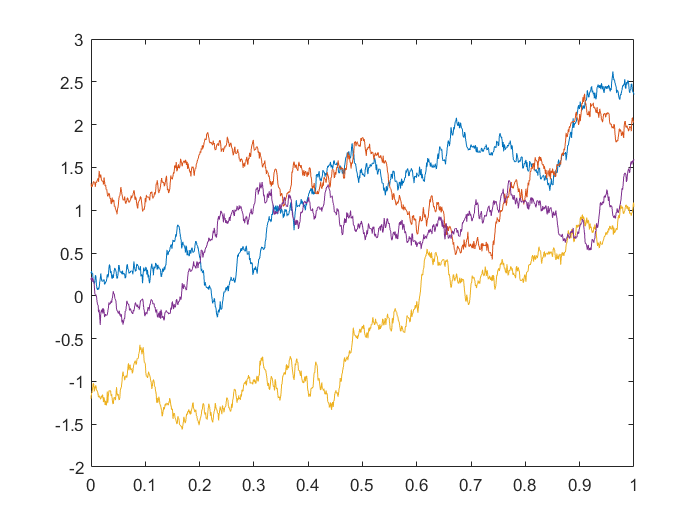}
	\end{subfigure}
	\begin{subfigure}{5cm}
		\includegraphics[width=5cm]{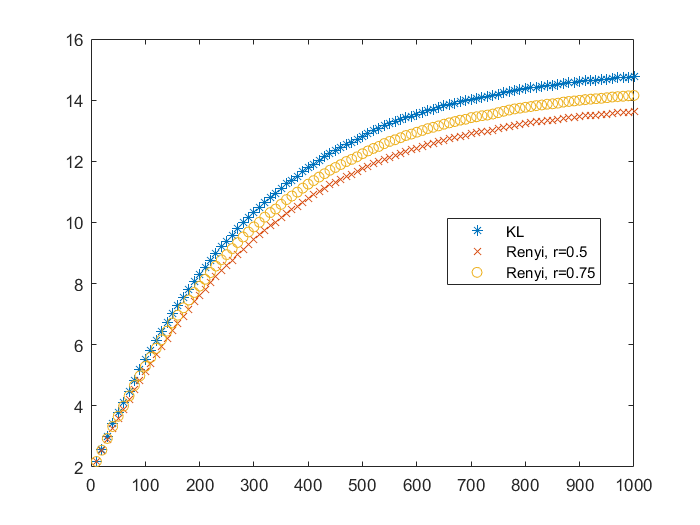}
	\end{subfigure}
	\caption{Samples of the Gaussian processes {$\GP(0,K^1)$, $\GP(\mu^2,K^2)$} 
		on {$T = [0,1]$} and approximate divergences between them, using finite-dimensional Gaussian measures.
		Left: {$K^1(x,y) = \exp(-a||x-y||)$, $a=1$}. Middle: {$\mu^2(x) = 1+x,K^2(x,y) = \exp(-a||x-y||)$, $a = 1.2$}.
		Here the 
		number of sample points is $m = 10,20, \ldots, 1000$, and regularization parameter is {$\gamma = 10^{-6}$}}
	\label{figure:Gaussian-process-sample-3}
\end{figure}

\section{Conclusion and future work}
We have presented the following results in the current work
\begin{enumerate}
 \item The approximation properties of the regularized R\'enyi and KL divergences 
on general, separable Hilbert spaces.
\item Consistency and sample complexities for the setting of RKHS covariance operators and Gaussian measures.
\item Consistency and sample complexities for the divergences between Gaussian processes, based on the corresponding finite-dimensional Gaussian measures.
\end{enumerate}
Potential applications of these results to problems in machine learning, statistics, in particular functional data analysis, etc, will be treated in subsequent future work.

\newpage

\appendix

\section{Proofs of main results}
\label{section-proofs}

\subsection{Proofs for the finite-rank/finite-dimensional approximations}
\label{section:proofs-sequence-approx}

\begin{lemma}
	\label{lemma:I+A-power-HS}
	Let $\alpha \in \R$ be fixed. 
	Let $I+A \in \PC_2(\H)$. Then $(I+A)^{\alpha}-I \in \Sym(\H)\cap\HS(\H)$.
	In other words, $(I+A)^{\alpha} = I+B$, where $B \in \Sym(\H) \cap \HS(\H)$.
\end{lemma}
\begin{proof}
	For $I+A \in \PC_2(\H)$, we have $\log(I+A)\in \Sym(\H) \cap \HS(\H)$ \citep[Lemma 12]{Minh:LogDetIII2018}. Thus
	$(I+A)^{\alpha}$ is well-defined $\forall \alpha \in \R$ via the expansion
	\begin{align*}
		(I+A)^{\alpha} = \exp(\alpha\log(I+A)) = I+ \sum_{k=1}^{\infty}\frac{\alpha^k}{k!}[\log(I+A)]^{k}.
	\end{align*}
	Since $\HS(\H)$ is a Banach algebra, it follows that
	\begin{align*}
		||(I+A)^{\alpha} - I||_{\HS} \leq \sum_{k=1}^{\infty}\frac{|\alpha|^k}{k!}||\log(I+A)||^{k}_{\HS} = \exp(||\log(I+A)||_{\HS}) -1 < \infty.
	\end{align*}
\end{proof}

\begin{proof}
	\textbf{of Proposition \ref{proposition:logdet-F-beta-representation}}
	For $-1< \alpha < 1$, by Definition \ref{def:logdet-HS}
	\begin{align*}
		d^{\alpha}_{\logdet}[(A+\gamma I), (B+\gamma I)] = \frac{4}{1-\alpha^2}\log\detX\left[\frac{1-\alpha}{2}(I+\Lambda)^{\frac{1+\alpha}{2}} + \frac{1+\alpha}{2}(I+\Lambda)^{-\frac{1-\alpha}{2}}\right].
	\end{align*}
Here $\log\detX$ is well-defined since $\frac{1-\alpha}{2}(I+\Lambda)^{\frac{1+\alpha}{2}} + \frac{1+\alpha}{2}(I+\Lambda)^{-\frac{1-\alpha}{2}} \in \PC_1(\H)$ by Proposition 2 in \citep{Minh:Positivity2020}. 
	By Lemma \ref{lemma:I+A-power-HS}, $(I+\Lambda)^{\frac{1+\alpha}{2}} = I+B$, $(I+\Lambda)^{-\frac{1-\alpha}{2}}
	= I+C$, for some $B,C \in \Sym(\H) \cap \HS(\H)$. Thus
	$\frac{1-\alpha}{2}(I+\Lambda)^{\frac{1+\alpha}{2}} + \frac{1+\alpha}{2}(I+\Lambda)^{-\frac{1-\alpha}{2}}
	= I + \frac{1-\alpha}{2}B + \frac{1+\alpha}{2}C$, 
	with $\frac{1-\alpha}{2}B + \frac{1+\alpha}{2}C\in \Tr(\H)$,
	so that $\detX$ reduces to $\det$, i.e.
	{\small
		\begin{align*}
			\detX\left[\frac{1-\alpha}{2}(I+\Lambda)^{\frac{1+\alpha}{2}} + \frac{1+\alpha}{2}(I+\Lambda)^{-\frac{1-\alpha}{2}}\right] = \det\left[\frac{1-\alpha}{2}(I+\Lambda)^{\frac{1+\alpha}{2}} + \frac{1+\alpha}{2}(I+\Lambda)^{-\frac{1-\alpha}{2}}\right].
		\end{align*}
	}
	For $\alpha = 1$, by Definition \ref{def:logdet-HS}
	\begin{align*}
		d^1_{\logdet}[(A+\gamma I), (B+\gamma I)] &= -\log\dettwoX[(B+\gamma I)^{-1}(A+\gamma I)] 
		\\
		&= -\log\dettwoX\left[\left(\frac{B}{\gamma}+I\right)^{-1}\left(\frac{A}{\gamma} + I\right)\right].
	\end{align*}
	This expression is defined entirely in terms of the nonzero eigenvalues
	of $\left(\frac{B}{\gamma}+I\right)^{-1}\left(\frac{A}{\gamma} + I\right)-I$, which
	are the same as those of $\left(\frac{B}{\gamma}+I\right)^{-1/2}\left(\frac{A}{\gamma} + I\right)\left(\frac{B}{\gamma}+I\right)^{-1/2}-I = \Lambda$. Thus
	\begin{align*}
		d^{1}_{\logdet}[(A+\gamma I),(B+\gamma I)] = -\log\dettwoX(I+\Lambda) = -\log\dettwo(I+\Lambda) = F_{0}(\Lambda). 
	\end{align*}
	The case $\alpha=-1$ is entirely similar.
\end{proof}

To prove Theorem \ref{theorem:F-A-continuity-HS}, 
we apply the following result from \citep{Kato1987VariationDiscreteSpectra},
specialized to the compact setting.
Let $A \in \Lcal(\H)$ be a self-adjoint, compact operator.
An eigenvalue of $A$ is called a {\it discrete eigenvalue}
if it has finite multiplicity. In particular, all nonzero eigenvalues of $A$ are discrete eigenvalues.
An {\it extended enumeration of discrete eigenvalues of $A$} is a sequence $\{\alpha_j\}_{j \in \Nbb}$
consisting of (i) every discrete eigenvalue with finite multiplicity $m$, appearing exactly $m$ times {\it (proper values)}, (ii)
all other values, if any, are zero {\it (improper values)}, which appears infinitely many times (if zero is an eigenvalue with
infinite multiplicity) or one time (if zero is not an eigenvalue but just the limiting point of the eigenvalue sequence).
If there are no improper values, then we simply have {\it enumeration}.

\begin{theorem}
	[\citep{Kato1987VariationDiscreteSpectra}, Theorem II, compact setting]\label{theorem:Kato-spectra}
	Let $A,B \in \Sym(\H)$ be compact.
	Let $\{\gamma_k\}_{k\in \Nbb}$ be an enumeration of nonzero eigenvalues of $B-A$.
	Then there exist extended enumerations $\{\alpha_j\}_{j \in \Nbb}, \{\beta_j\}_{j \in \Nbb}$ of discrete eigenvalues for $A,B$, respectively, such that
	\begin{align}
		\sum_{j=1}^{\infty}\Phi(\beta_j - \alpha_j) \leq \sum_{k=1}^{\infty}\Phi(\gamma_k)
	\end{align}
	for any nonnegative convex function $\Phi$ on $\R$, with $\Phi(0) = 0$. In particular,
	\begin{align}
		\label{equation:Kato-p-norm-difference}
		\left(\sum_{j=1}^{\infty}|\beta_j - \alpha_j|^{p}\right)^{1/p} \leq ||B-A||_p = \left(\sum_{k=1}^{\infty}|\gamma_k|^p\right)^{1/p}, \;\;\; 1 \leq p \leq \infty.
	\end{align}
\end{theorem}

The following is a consequence of Corollary 3.2 in \citep{Kitta:InequalitiesV} (see \citep{Minh:2021RiemannianEstimation})
\begin{lemma}
	\label{lemma:I+A-power-r-Schatten-norm}
	Let $1 \leq p \leq \infty$ be fixed.
	For two operators $A,B \in \Sym^{+}(\H)\cap \Csc_p(\H)$,
	\begin{align*}
		||(I+A)^{r} - (I+B)^{r}||_p \leq |r|\; ||A-B||_p, \;\;\; |r| \leq 1.
	\end{align*} 
\end{lemma}

\begin{lemma}
	\label{lemma:I+Delta-Schatten-norm}
	Let $1\leq p \leq \infty$ be fixed. 
	Let $B \in \Sym^{+}(\H) \cap \Csc_p(\H)$. Then
	for any $A \in \Csc_p(\H)$, $(I+B)^{-1/2}(I+A)(I+B)^{-1/2}-I \in \Csc_p(\H)$, with
	\begin{align}
		||(I+B)^{-1/2}(I+A)(I+B)^{-1/2}-I||_p \leq ||A-B||_p.
	\end{align}
	Furthermore, let $B_n \in \Sym^{+}(\H) \cap \Csc_p(\H)$, $A_n \in \Csc_p(\H)$, then
	\begin{align}
		&||(I+B_n)^{-1/2}(I+A_n)(I+B_n)^{-1/2} - (I+B)^{-1/2}(I+A)(I+B)^{-1/2}||_{p}
		\nonumber
		\\
		& \leq
		||A_n - A||_p + \left(1+\frac{1}{2}||A_n|| + \frac{1}{2}||A||\right)||B_n - B||_p.
	\end{align}
	Assume further that $A_n, A \in \Sym^{+}(\H) \cap \Csc_p(\H)$, then
	\begin{align}
		&||(I+B_n)^{1/2}(I+A_n)^{-1}(I+B_n)^{1/2} - (I+B)^{1/2}(I+A)^{-1}(I+B)^{1/2}||_{p}
		\nonumber
		\\
		&\leq (1+||B_n||^{1/2})(1+||B||^{1/2})||A_n - A||_p + \left(1+\frac{1}{2}||B_n||^{1/2} + \frac{1}{2}||B||^{1/2}\right)||B_n - B||_p.
	\end{align}
\end{lemma}
\begin{proof}
	Since $(I+B)^{-1} = I - B(I+B)^{-1}$, we have for $B \in \Sym^{+}(\H)$,
	\begin{align*}
		||(I+B)^{-1/2}(I+A)(I+B)^{-1/2} - I ||_p &= ||(I+B)^{-1/2}(A-B)(I+B)^{-1/2}||_p 
		\\
		&\leq ||(I+B)^{-1/2}||^2||A-B||_p \leq ||A-B||_p.
	\end{align*}
	For the second inequality, by Lemma \ref{lemma:I+A-power-r-Schatten-norm},
	\begin{align*}
		&\Delta = ||(I+B_n)^{-1/2}(I+A_n)(I+B_n)^{-1/2}- (I+B)^{-1/2}(I+A)(I+B)^{-1/2}||_p
		\\
		&\leq ||(I+B_n)^{-1/2}(I+A_n)(I+B_n)^{-1/2}- (I+B_n)^{-1/2}(I+A_n)(I+B)^{-1/2}||_p
		\\
		&\quad + ||(I+B_n)^{-1/2}(I+A_n)(I+B)^{-1/2} - (I+B_n)^{-1/2}(I+A)(I+B)^{-1/2}||_p
		\\
		& \quad + ||(I+B_n)^{-1/2}(I+A)(I+B)^{-1/2} - (I+B)^{-1/2}(I+A)(I+B)^{-1/2}||_p
		\\
		& \leq ||(I+B_n)^{-1/2}||\;||I+A_n||\;||(I+B_n)^{-1/2} - (I+B)^{-1/2}||_p
		\\
		& \quad + ||(I+B_n)^{-1/2}||\;||A_n - A||_p||(I+B)^{-1/2}||
		\\
		&\quad + ||(I+B_n)^{-1/2} - (I+B)^{-1/2}||_p ||I+A||\;||(I+B)^{-1/2}||
		\\
		& \leq \frac{1}{2}(1+||A_n||)||B_n - B||_p + ||A_n - A||_p+ \frac{1}{2}(1+||A||)||B_n - B||_p
		\\
		& = ||A_n - A||_p + (1+\frac{1}{2}||A_n|| + \frac{1}{2}||A||)||B_n - B||_p.
	\end{align*}

	For the third inequality, with the additional assumption $A_n, A \in \Sym^{+}(\H)$,
	\begin{align*}
		&\Delta_2 = ||(I+B_n)^{1/2}(I+A_n)^{-1}(I+B_n)^{1/2}- (I+B)^{1/2}(I+A)^{-1}(I+B)^{1/2}||_p
		\\
		&\leq ||(I+B_n)^{1/2}(I+A_n)^{-1}(I+B_n)^{1/2}- (I+B_n)^{1/2}(I+A_n)^{-1}(I+B)^{1/2}||_p
		\\
		&\quad + ||(I+B_n)^{1/2}(I+A_n)^{-1}(I+B)^{1/2} - (I+B_n)^{1/2}(I+A)^{-1}(I+B)^{1/2}||_p
		\\
		& \quad + ||(I+B_n)^{1/2}(I+A)^{-1}(I+B)^{1/2} - (I+B)^{1/2}(I+A)^{-1}(I+B)^{1/2}||_p
		\\
		& \leq ||(I+B_n)^{1/2}||\;||(I+A_n)^{-1}||\;||(I+B_n)^{1/2}- (I+B)^{1/2}||_p
		\\
		&\quad + ||(I+B_n)^{1/2}||\;||(I+A_n)^{-1} - (I+A)^{-1}||_p||(I+B)^{1/2}||
		\\
		& \quad + ||(I+B_n)^{1/2} - (I+B)^{1/2}||_p||(I+A)^{-1}||\;||(I+B)^{1/2}||
		\\
		& \leq \frac{1}{2}(1+||B_n||^{1/2})||B_n - B||_p + (1+||B_n||^{1/2})(1+||B||^{1/2})||A_n - A||_p
		\\
		&\quad + \frac{1}{2}(1+||B||^{1/2})||B_n - B||_p
		\\
		& = (1+||B_n||^{1/2})(1+||B||^{1/2})||A_n - A||_p + (1+\frac{1}{2}||B_n||^{1/2} + \frac{1}{2}||B||^{1/2})||B_n - B||_p.
	\end{align*}
\end{proof}

\begin{lemma}
	\label{lemma:A2-inverseI+A-inequality}
	For $A,B \in \HS(\H)$, with $I+A, I+B$ invertible,
	\begin{align}
		&||(I+A)^{-1}A^2 - (I+B)^{-1}B^2||_{\tr}
		\nonumber
		\\
		&\leq ||(I+A)^{-1}||\;||(I+B)^{-1}||[||A||_{\HS} + ||B||_{\HS} + ||A||_{\HS}||B||_{\HS}]||A-B||_{\HS}. 
	\end{align}
\end{lemma}
\begin{proof}
	Applying the property $||AB||_{\tr} \leq ||A||_{\HS}||B||_{\HS}$ for $A,B \in \HS(\H)$ (see e.g. \citep{ReedSimon:Functional})
	\begin{align*}
		&||(I+A)^{-1}A^2 - (I+B)^{-1}B^2||_{\tr} = ||(I+A)^{-1}[A^2(I+B) - (I+A)B^2](I+B)^{-1}||_{\tr}
		\\
		& = ||(I+A)^{-1}[A^2-B^2 + A(A-B)B](I+B)^{-1}||_{\tr} 
		\\
		&\leq ||(I+A)^{-1}||\;[||A(A-B) + (A-B)B||_{\tr} + ||A(A-B)B||_{\tr}]||(I+B)^{-1}||
		\\
		& \leq ||(I+A)^{-1}||\;||(I+B)^{-1}||[||A||_{\HS} + ||B||_{\HS} + ||A||_{\HS}||B||_{\HS}]||A-B||_{\HS}.
	\end{align*}
\end{proof}

\begin{proof}
		\textbf{of Theorem \ref{theorem:F-A-continuity-HS}}
	We prove here the case $0 < \beta < 1$ and prove the cases $\beta =0,1$ separately in Theorem 
	\ref{theorem:logdet-Hilbert-Carleman-bound}.
	
	(i) Consider the general case $I+A > 0, I+B > 0$.
	Let $\{\lambda_k\}_{k \in \Nbb}$ be any extended enumeration
	of discrete eigenvalues of $A$. 
	By Lemma \ref{lemma:log-sum-inequality-1},
	\begin{align*}
		&0 \leq \log\det[\beta (I+A)^{1-\beta} + (1-\beta)(I+A)^{-\beta}] = \sum_{k=1}^{\infty}
		\log[\beta(1+\lambda_k)^{1-\beta} + (1-\beta)(1+\lambda_k)^{-\beta}]
		\\
		&\leq \sum_{k=1}^{\infty}\frac{\lambda_k^2}{1+\lambda_k} \leq \max_{k\in \Nbb}\frac{1}{1+\lambda_k}\sum_{k=1}^{\infty}\lambda_k^2 = ||(I+A)^{-1}||\;||A||^2_{\HS}.
	\end{align*}
	Since the above series all converge absolutely,  the value of $\log\det[\beta (I+A)^{1-\beta} + (1-\beta)(I+A)^{-\beta}]$ is independent of any rearrangement of 
	the nonzero eigenvalues of $A$. This also shows that $F_{\beta}(A) \geq 0$.
	
	By Eq.\eqref{equation:Kato-p-norm-difference} in Theorem \ref{theorem:Kato-spectra}, applied to
	$(I+A)^{-1}A^2$ and $(I+B)^{-1}B^2$ with $p=1$,
	there are 
	extended enumerations $\{\lambda_k\}_{k \in \Nbb}, \{\mu_k\}_{k \in \Nbb}$
	of discrete eigenvalues of $A,B$, respectively, with
	\begin{align*}
		\sum_{k=1}^{\infty}\left|\frac{\lambda_k^2}{1+\lambda_k} - \frac{\mu_k^2}{1+\mu_k}\right| \leq ||(I+A)^{-1}A^2 - (I+B)^{-1}B^2||_{\tr}.
	\end{align*}
	Using these extended enumerations $\{\lambda_k\}_{k \in \Nbb}$, $\{\mu_k\}_{k \in \Nbb}$, we have
	\begin{align*}
		&|\log\det[\beta (I+A)^{1-\beta} + (1-\beta)(I+A)^{-\beta}] -
		\log\det[\beta (I+B)^{1-\beta} + (1-\beta)(I+B)^{-\beta}]|
		\\ 
		&= \left|\sum_{k=1}^{\infty}\log[\beta(1+\lambda_k)^{1-\beta} + (1-\beta)(1+\lambda_k)^{-\beta}] - 
		\sum_{k=1}^{\infty}\log[\beta(1+\mu_k)^{1-\beta} + (1-\beta)(1+\mu_k)^{-\beta}]\right|
		\\
		&\leq \sum_{k=1}^{\infty}|\log[\beta(1+\lambda_k)^{1-\beta} + (1-\beta)(1+\lambda_k)^{-\beta}]
		- \log[\beta(1+\mu_k)^{1-\beta} + (1-\beta)(1+\mu_k)^{-\beta}]|
		\\
		& \leq \beta(1-\beta)\sum_{k=1}^{\infty}\left|\frac{\lambda_k^2}{1+\lambda_k} - \frac{\mu_k^2}{1+\mu_k}\right|
		\;\;\;\text{by Lemma \ref{lemma:log-x2-inverse-inequality-1}}
		\\
		& \leq \beta(1-\beta)||(I+A)^{-1}A^2 - (I+B)^{-1}B^2||_{\trace} \;\;\; \text{by Theorem \ref{theorem:Kato-spectra}, applied to $p=1$}
		\\
		& \leq \beta(1-\beta)||(I+A)^{-1}||\;||(I+B)^{-1}||[||A||_{\HS} + ||B||_{\HS} + ||A||_{\HS}||B||_{\HS}]||A-B||_{\HS}
	\end{align*}
	by Lemma \ref{lemma:A2-inverseI+A-inequality}. This gives $|F_{\beta}(A) - F_{\beta}(B)| 
	\leq ||(I+A)^{-1}||\;||(I+B)^{-1}||[||A||_{\HS} + ||B||_{\HS} + ||A||_{\HS}||B||_{\HS}]||A-B||_{\HS}$.
	Letting $B =0$ gives $0 \leq F_{\beta}(A) \leq ||(I+A)^{-1}||\;||A||^2_{\HS}$.
	
	(ii) For $A,B \in \Sym^{+}(\H) \cap \HS(\H)$, we choose the extended enumerations
	$\{\lambda_k\}_{k \in \Nbb}, \{\mu_k\}_{k \in \Nbb}$ of discrete eigenvalues of $A,B$, respectively, such that
	\begin{align*}
		\sum_{k=1}^{\infty}|\lambda_k^2 - \mu_k^2|\leq ||A^2 - B^2||_{\tr}.
	\end{align*}
	Using these extended enumerations $\{\lambda_k\}_{k \in \Nbb}$, $\{\mu_k\}_{k \in \Nbb}$, 
	with $\lambda_k \geq 0, \mu_k \geq 0$ $\forall k \in \Nbb$,
	\begin{align*}
		&|\log\det[\beta (I+A)^{1-\beta} + (1-\beta)(I+A)^{-\beta}] -
		\log\det[\beta (I+B)^{1-\beta} + (1-\beta)(I+B)^{-\beta}]|
		\\ 
		&= \left|\sum_{k=1}^{\infty}\log[\beta(1+\lambda_k)^{1-\beta} + (1-\beta)(1+\lambda_k)^{-\beta}] - 
		\sum_{k=1}^{\infty}\log[\beta(1+\mu_k)^{1-\beta} + (1-\beta)(1+\mu_k)^{-\beta}]\right|
		\\
		&\leq \sum_{k=1}^{\infty}|\log[\beta(1+\lambda_k)^{1-\beta} + (1-\beta)(1+\lambda_k)^{-\beta}]
		- \log[\beta(1+\mu_k)^{1-\beta} + (1-\beta)(1+\mu_k)^{-\beta}]|
		\\
		& \leq \frac{\beta(1-\beta)}{2}\sum_{k=1}^{\infty}\left|\lambda_k^2 - \mu_k^2\right|
		\;\;\;\text{by Lemma \ref{lemma:log-x2-inverse-inequality-1}}
		\\
		&\leq \frac{\beta(1-\beta)}{2}||A^2-B^2||_{\tr} = \frac{\beta(1-\beta)}{2}||A(A-B)+(A-B)B||_{\tr}
		\\
		&\leq \frac{\beta(1-\beta)}{2}[||A||_{\HS} + ||B||_{\HS}]||A-B||_{\HS}.
	\end{align*}
This gives $|F_{\beta}(A)- F_{\beta}(B)| \leq \frac{1}{2}[||A||_{\HS} + ||B||_{\HS}]||A-B||_{\HS}$.
Letting $B = 0$ gives $0 \leq F_{\beta}(A) \leq \frac{1}{2}||A||^2_{\HS}$.
\end{proof}

\begin{proof}
	\textbf{of Theorem \ref{theorem:logdet-Hilbert-Carleman-bound}}
	(i) Consider first the function $\log\dettwo(I+A)$.
	Let $\{\lambda_k\}_{k\in \Nbb}$ denote the eigenvalues of $A$. Then
	\begin{align*}
		\log\dettwo(I+A) = \log\det[(I+A)\exp(-A)] = \log\prod_{k=1}^{\infty}(1+\lambda_k)e^{-\lambda_k} = 
		\sum_{k=1}^{\infty}[\log(1+\lambda_k)-\lambda_k].
	\end{align*}
	By Lemma \ref{lemma:log-bound-1}, $\log\dettwo(I+A) \leq 0$ always and thus 
	\begin{align*}
		|\log\dettwo(I+A)| & = -\log\dettwo(I+A)= \sum_{k=1}^{\infty}[\lambda_k-\log(1+\lambda_k)] \leq \sum_{k=1}^{\infty}\frac{\lambda_k^2}{1+\lambda_k} 
		\\
		&\leq \max_{k \in \Nbb}\left\{\frac{1}{1+\lambda_k} \right\}\sum_{k=1}^{\infty}\lambda_k^2 = ||(I+A)^{-1}||\;||A||^2_{\HS}.
	\end{align*}
	Furthermore, for $A \in \Sym^{+}(\H) \cap \HS(\H)$, 
	$\lambda_k \geq 0$ $\forall k \in \Nbb$ and thus by Lemma \ref{lemma:log-bound-1}
	\begin{align*}
		|\log\dettwo(I+A)| & = -\log\dettwo(I+A)= \sum_{k=1}^{\infty}[\lambda_k-\log(1+\lambda_k)] \leq \frac{1}{2}\sum_{k=1}^{\infty}{\lambda_k^2} = \frac{1}{2}||A||^2_{\HS}.
	\end{align*}
	(ii) Consider now the function $\log\dettwo[(I+A)^{-1}]$. We have
	\begin{align*}
		\log\dettwo[(I+A)^{-1}] &= \log\det[(I+A)^{-1}\exp(A(I+A)^{-1})] 
		\\
		&= \log\prod_{k=1}^{\infty}(1+\lambda_k)^{-1}\exp(\lambda_k/(1+\lambda_k)) = \sum_{k=1}^{\infty}\left[-\log(1+\lambda_k) + \frac{\lambda_k}{1+\lambda_k}\right]. 
	\end{align*}
	Once again by Lemma \ref{lemma:log-bound-1}, $\log\dettwo[(I+A)^{-1}] \leq 0$ and
	\begin{align*}
		|\log\dettwo[(I+A)^{-1}]| &= - \log\dettwo[(I+A)^{-1}] = \sum_{k=1}^{\infty}\left[\log(1+\lambda_k) - \frac{\lambda_k}{1+\lambda_k}\right]
		\\
		& \leq \sum_{k=1}^{\infty}\frac{\lambda_k^2}{1+\lambda_k} \leq ||(I+A)^{-1}||\;||A||^2_{\HS}.
	\end{align*}
	The other inequalities follow as in Theorem \ref{theorem:F-A-continuity-HS}, by applying Lemmas \ref{lemma:log-bound-1} and
	\ref{lemma:log-bound-3}.
\end{proof}

\begin{proof}
	\textbf{of Theorem \ref{theorem:convergence-logdet-infinite}}
	(i) Consider the case $\gamma = 1$.
	Write $(I+B)^{-1/2}(I+A)(I+B)^{-1/2} = 
	I+\Lambda$,
	where $\Lambda = (I+B)^{-1/2}(A-B)(I+B)^{-1/2} \in \Sym(\H) \cap \HS(\H)$. 
	Then
	\begin{align*}
		&d^{\alpha}_{\logdet}[(I+A),(I+B)]
		= F_{\frac{1-\alpha}{2}}(\Lambda)\;\;\; \text{by Proposition \ref{proposition:logdet-F-beta-representation}, with $F$ as defined in Eq.\eqref{equation:F-beta}}
		\\
		& \leq ||(I+\Lambda)^{-1}||\;||\Lambda||^2_{\HS} \;\;\text{by Theorem \ref{theorem:F-A-continuity-HS}}
		\\
		& \leq ||(I+B)^{1/2}(I+A)^{-1}(I+B)^{1/2}||\;||A-B||^2_{\HS}||(I+B)^{-1}||^2
		\\
		& \leq (1+||B||)||(I+A)^{-1}||\;||(I+B)^{-1}||^2\;||A-B||^2_{\HS}.
	\end{align*}
	(ii) For $\gamma > 0$,
	$(\gamma I + B)^{-1/2}(\gamma I + A)(\gamma I + B)^{-1/2} = (I+\frac{B}{\gamma})^{-1/2}(I+\frac{A}{\gamma})(I+\frac{B}{\gamma})^{-1/2}$. Thus
	\begin{align*}
		&d^{\alpha}_{\logdet}[(\gamma I +A), (\gamma I + B)] = d^{\alpha}_{\logdet}\left[\left(I + \frac{A}{\gamma}\right), \left(I+ \frac{B}{\gamma}\right)\right]
		\\
		&\leq \left(1+\frac{||B||}{\gamma}\right)\left\|\left(I+ \frac{A}{\gamma}\right)^{-1}\right\|\left\|\left(I+ \frac{B}{\gamma}\right)^{-1}\right\|^2\left\|\frac{A}{\gamma} - \frac{B}{\gamma}\right\|^2_{\HS}
		\\
		&= (\gamma +||B||)||(\gamma I + A)^{-1}||\;||(\gamma I + B)^{-1}||\;||A-B||^2_{\HS}.
	\end{align*} 
	
\end{proof}

\begin{proof}
	\textbf{of Theorem \ref{theorem:logdet-approx-infinite-sequence}}
	(i) Consider the case $\gamma = 1$. 
	Write $(I+B)^{-1/2}(I+A)(I+B)^{-1/2} = I+\Lambda$,
	where $\Lambda = (I+B)^{-1/2}(A-B)(I+B)^{-1/2} \in \Sym(\H) \cap \HS(\H)$. Similarly,
	write  $(I+B_n)^{-1/2}(I+A_n)(I+B_n)^{-1/2}  = I+\Lambda_n$,
	where $\Lambda_n = (I+B_n)^{-1/2}(A_n-B_n)(I+B_n)^{-1/2} \in \Sym(\H) \cap \HS(\H)$. 
	Let $\beta = \frac{1-\alpha}{2}$, then, with
	$F_{\beta}$ as defined in Eq.\eqref{equation:F-beta},
	by Proposition \ref{proposition:logdet-F-beta-representation},
	$d^{\alpha}_{\logdet}[(I+A), (I+B)] = F_{\beta}(\Lambda)$, $d^{\alpha}_{\logdet}[(I+A_n), (I+B_n)] = F_{\beta}(\Lambda_n)$.
	Thus
	by Theorem \ref{theorem:F-A-continuity-HS},
	\begin{align*}
		&\Delta = |d^{\alpha}_{\logdet}[(I+A_n), (I+B_n)] - d^{\alpha}_{\logdet}[(I+A), (I+B)]|
		= |F_{\beta}(\Lambda_n) - F_{\beta}(\Lambda)|
		\\
		&\leq ||(I+\Lambda_n)^{-1}||\;||(I+\Lambda)^{-1}||\;[||\Lambda_n||_{\HS} + ||\Lambda||_{\HS} + ||\Lambda_n||_{\HS}||\Lambda||_{\HS}]||\Lambda_n - \Lambda||_{\HS}.
	\end{align*}
	Since $B_n,B \in \Sym^{+}(\H)$, we have $||\Lambda_n||_{\HS} \leq ||A_n - B_n||_{\HS} \leq ||A_n||_{\HS} + ||B_n||_{\HS}$,
	$||\Lambda||_{\HS} \leq ||A-B||_{\HS} \leq ||A||_{\HS} + ||B||_{\HS}$.
	By Lemma \ref{lemma:I+Delta-Schatten-norm},
	\begin{align*}
		||\Lambda_n - \Lambda||_{\HS} &=||(I+B_n)^{-1/2}(I+A_n)(I+B_n)^{-1/2} - (I+B)^{-1/2}(I+A)(I+B)^{-1/2}||_{\HS}
		\\
		&\leq ||A_n -A||_{\HS} + \left(1+\frac{1}{2}||A_n|| + \frac{1}{2}||A||\right)||B_n - B||_{\HS}.
	\end{align*}
	Furthermore, since $A_n,A,B_n,B \in \Sym^{+}(\H) \cap \HS(\H)$, we have
	\begin{align*}
		||(I+\Lambda_n)^{-1}|| &= ||(I+B_n)^{1/2}(I+A_n)^{-1}(I+B_n)^{1/2}|| \leq ||(I+B_n)^{1/2}||^2 = 1 +||B_n||,
		\\
		||(I+\Lambda)^{-1}|| &= ||(I+B)^{1/2}(I+A)^{-1}(I+B)^{1/2}|| \leq ||(I+B)^{1/2}||^2 = 1 +||B||.
	\end{align*}
	It follows that
	\begin{align*}
		\Delta &\leq (1+||B_n||)(1+||B||)
		\\
		&\quad \times \left[||A_n||_{\HS}+||B_n||_{\HS} + ||A||_{\HS} + ||B||_{\HS} + (||A_n||_{\HS} + ||B_n||_{\HS})(||A||_{\HS} + ||B||_{\HS})\right]
		\\
		& \quad \times \left[||A_n -A||_{\HS} + \left(1+\frac{1}{2}||A_n|| + \frac{1}{2}||A||\right)||B_n - B||_{\HS}\right].
	\end{align*}
	(ii) The general case $\gamma > 0$ follows from
	$d^{\alpha}_{\logdet}[(A+\gamma I), (B+\gamma I)] = d^{\alpha}_{\logdet}[(I+\frac{A}{\gamma}), (I+ \frac{B}{\gamma})]$.
\end{proof}


\subsection{Proofs for the RKHS setting}
\label{section:proofs-RKHS}

\begin{lemma}
	\label{lemma:Sylvester}
	Let {$\H_1$} and {$\H_2$} be two separable Hilbert spaces. Let 
	{$A: \H_1 \mapto \H_2$} be compact, such that $A^{*}A \in \Tr(\H_1)$. Then 
	$AA^{*} \in \Tr(\H_2)$ and
	\begin{equation}
		\det(I_{\H_1} + A^{*}A) = \det(I_{\H_2} + AA^{*}).
	\end{equation}
\end{lemma}
\begin{proof}
	This follows from the fact that the nonzero eigenvalues of $AA^{*}:\H_2 \mapto \H_2$ are the same as those of
	$A^{*}A: \H_1 \mapto \H_1$.
\end{proof}

\begin{lemma}\label{lemma:det-1} 
	Let $\H_1, \H_2, \H$ be separable Hilbert spaces. Let $A: \H_1 \mapto \H$, $B:\H_2 \mapto \H$ be 
	compact operators
	such that $A^{*}A \in \Tr(\H_1)$, $B^{*}B\in \Tr(\H_2)$.
	Then $AA^{*}, BB^{*} \in \Tr(\H)$
	and
	\begin{equation}
		\det(AA^{*}+ BB^{*}+ I_{\H}) = \det
		\left[
		\begin{pmatrix}
			A^{*}A & A^{*}B\\
			B^{*}A & B^{*}B
		\end{pmatrix}
		+ I_{\H_1 \oplus \H_2}
		\right].
	\end{equation}
\end{lemma}

\begin{proof}
	Consider the operator {$(A \; B): \H_1 \oplus \H_2 \mapto \H$}, defined by $(A \; B)\begin{pmatrix}v_1 \\ v_2\end{pmatrix} 
	= Av_1 + Bv_2$, where $v_1 \in \H_1, v_2\in \H_2$.
	Here {$\H_1 \oplus \H_2$} denotes the direct sum of {$\H_1$} with $\H_2$, that is
	$\H_1 \oplus \H_2 = \{(v_1, v_2) \; : \; v_1 \in \H_1, v_2 \in \H_2\}$,
	equipped with the inner product
	$\la (v_1, v_2) , (w_1, w_2)\ra_{\H_1 \oplus \H_2} = \la v_1, w_1\ra_{\H_1} + \la v_2, w_2\ra_{\H_2}$.
	The adjoint of $(A \; B): \H_1 \oplus \H_2 \mapto \H$ is
	$(A \; B)^{*} = \begin{pmatrix}
		A^{*}\\
		B^{*}
	\end{pmatrix}: \H \mapto \H_1 \oplus \H_2$, defined by
	$\begin{pmatrix}
		A^{*}\\
		B^{*}
	\end{pmatrix}v = \begin{pmatrix}
		A^{*}v\\ B^{*}v
	\end{pmatrix} \in \H_1 \oplus \H_2$, for $v \in \H$. 
	By Lemma \ref{lemma:Sylvester}, both $AA^{*}, BB^{*}:
	\H \mapto \H$ are trace class operators and we have
	\begin{align*}
		&\det(AA^{*}+ BB^{*}+ I_{\H}) =   
		\det\left[
		\begin{pmatrix}
			A & B
		\end{pmatrix}
		\begin{pmatrix}
			A^{*}\\
			B^{*}
		\end{pmatrix}
		+ I_{\H}\right]
		\\
		&= \det\left[
		\begin{pmatrix}
			A^{*}\\
			B^{*}
		\end{pmatrix}\begin{pmatrix}
			A & B
		\end{pmatrix}
		+ \begin{pmatrix}
			I_{\H_1} & 0\\
			0 & I_{\H_2}
		\end{pmatrix}  \right]
		=
		\det
		\left[
		\begin{pmatrix}
			A^{*}A & A^{*}B\\
			B^{*}A & B^{*}B
		\end{pmatrix}
		+ I_{\H_1 \oplus \H_2}
		\right].
	\end{align*}
	Here $\begin{pmatrix}
		A^{*}A & A^{*}B\\
		B^{*}A & B^{*}B
	\end{pmatrix}=\begin{pmatrix}
		A^{*}\\
		B^{*}
	\end{pmatrix}\begin{pmatrix}
		A & B
	\end{pmatrix} \in \Tr(\H_1 \oplus \H_2)$ also by Lemma \ref{lemma:Sylvester}.
\end{proof}

\begin{theorem}
	[\citep{Minh:LogDet2016}]
	\label{theorem:LogDet-infinite}
	Assume that {$\dim(\H) = \infty$} and that {$(A+\gamma I), (B+ \mu I) \in \PC_1(\H)$}.
	The divergence {$d^{\alpha}_{\logdet}[(A+\gamma I), (B+\mu I)]$} for {$-1 < \alpha < 1$} is given by
	{%
		\begin{align}\label{equation:LogDet-infinite}
			&d^{\alpha}_{\logdet}[(A+\gamma I), (B+\mu I)] 
			= \frac{4}{1-\alpha^2}\logdet\left[\frac{(1-\alpha)A+(1+\alpha)B}
			{s(\gamma, \mu, \alpha)}
			+ I\right]
			\\
			&
			\quad \quad -\frac{4\beta}{1-\alpha^2}\logdet\left(\frac{A}{\gamma} + I\right)
			- \frac{4(1-\beta)}{1-\alpha^2}\logdet\left(\frac{B}{\mu} + I\right) 
			+ c(\alpha, \gamma, \mu).
			\nonumber
		\end{align}
	}
	where $c(\alpha, \gamma, \mu) \geq 0$, with equality if and only if $\gamma = \mu$.
	%
\end{theorem}
In the finite-dimensional case, on the other hand,
there is an explicit dependence on 
{$\dim(\H)$}.

\begin{theorem}[\citep{Minh:LogDet2016}]\label{theorem:LogDet-finite} Assume that {$\dim(\H) < \infty$}. Let {$\gamma, \mu > 0$}. The divergence {$d^{\alpha}_{\logdet}[(A+\gamma I), (B+\mu I)]$} for {$-1 < \alpha < 1$} is given by
	{
		\begin{align}\label{equation:LogDet-finite}
			&d^{\alpha}_{\logdet}[(A+\gamma I), (B+\mu I)] 
			= \frac{4}{1-\alpha^2}\logdet\left(\frac{(1-\alpha)A+(1+\alpha)B}
			{s(\gamma, \mu, \alpha)}
			+ I\right)
			\\
			&
			\quad\quad -\frac{2}{1+\alpha}\logdet\left(\frac{A}{\gamma} + I\right)
			- \frac{2}{1-\alpha}\logdet\left(\frac{B}{\mu} + I\right) 
			+ c(\alpha, \gamma, \mu)\dim(\H).
			\nonumber
		\end{align}
	}
\end{theorem}

\begin{proof}
	\textbf{of Theorem \ref{theorem:LogDet-infinite-2}}
	By Lemma \ref{lemma:det-1}, we have
	{
		\begin{align*}
			&\log\det\left[\frac{(1-\alpha)AA^{*}+(1+\alpha)BB^{*}}{(1-\alpha)\gamma + (1+\alpha)\mu} + I_{\H}\right] 
			\\
			&= \logdet
			\left[
			\frac{1}{(1-\alpha)\gamma + (1+\alpha)\mu}
			\begin{pmatrix}
				(1-\alpha)A^{*}A & \sqrt{1-\alpha^2}A^{*}B\\
				\sqrt{1-\alpha^2}B^{*}A & (1+\alpha)B^{*}B
			\end{pmatrix}
			+ I_{\H_1 \oplus \H_2}
			\right],
			\\
			%
			&\log\det\left(\frac{AA^{*}}{\gamma} + I_{\H}\right) = \log\det\left(\frac{A^{*}A}{\gamma} + I_{\H_1}\right),
			\\ 
			&\log\det\left(\frac{BB^{*}}{\mu} + I_{\H}\right) = \log\det\left(\frac{B^{*}B}{\mu} + I_{\H_2}\right).
		\end{align*}
	}
	Combining all of these expressions with Theorem \ref{theorem:LogDet-infinite}, we obtain the first expression.
	
	By definition of the Hilbert-Carleman determinant,
	\begin{align*}
	&\log\det\left(\frac{A^{*}A}{\gamma} + I_{\H_1}\right) =\log\dettwo\left(\frac{A^{*}A}{\gamma} + I_{\H_1}\right) + \trace\left(\frac{A^{*}A}{\gamma}\right),
	\\
	&\log\det\left(\frac{B^{*}B}{\mu} + I_{\H_2}\right) = \log\dettwo\left(\frac{B^{*}B}{\mu} + I_{\H_2}\right)
	+\trace\left(\frac{B^{*}B}{\mu}\right),
	\\
	&\logdet
	\left[
	\frac{1}{(1-\alpha)\gamma + (1+\alpha)\mu}
	\begin{pmatrix}
		(1-\alpha)A^{*}A & \sqrt{1-\alpha^2}A^{*}B\\
		\sqrt{1-\alpha^2}B^{*}A & (1+\alpha)B^{*}B
	\end{pmatrix}
	+ I_{\H_1 \oplus \H_2}
	\right]
	\\
	& = \log\dettwo
	\left[
	\frac{1}{(1-\alpha)\gamma + (1+\alpha)\mu}
	\begin{pmatrix}
		(1-\alpha)A^{*}A & \sqrt{1-\alpha^2}A^{*}B\\
		\sqrt{1-\alpha^2}B^{*}A & (1+\alpha)B^{*}B
	\end{pmatrix}
	+ I_{\H_1 \oplus \H_2}
	\right]
	\\
	&\quad +\frac{1-\alpha}{(1-\alpha)\gamma + (1+\alpha)\mu}\trace(A^{*}A) + \frac{1+\alpha}{(1-\alpha)\gamma + (1+\alpha)\mu}\trace(B^{*}B).
	\end{align*}
	Noting that $\beta = \frac{(1-\alpha)\gamma}{(1-\alpha)\gamma + (1+\alpha)\mu}$, $1-\beta = \frac{(1+\alpha)\mu}{(1-\alpha)\gamma + (1+\alpha)\mu}$, we obtain the final result.
\end{proof}

\begin{proof}
	\textbf{of Theorem \ref{theorem:LogDet-finite-2}}
	This is similar to the proof of Theorem \ref{theorem:LogDet-infinite-2}, except that we combine Lemma \ref{lemma:det-1} with Theorem \ref{theorem:LogDet-finite}.
\end{proof}

\begin{proof}
	\textbf{of Theorem \ref{theorem:LogDet-RKHS-infinite}}
	Let {$A = \frac{1}{\sqrt{m_1}}\Phi(\Xbf^1)J_{m_1}:\R^{m_1} \mapto \H_K$}, {$B = \frac{1}{\sqrt{m_2}}\Phi(\Xbf^2)J_{m_2}:\R^{m_2} \mapto \H_K$}, then 
	$AA^{*} = C_{\Phi(\Xbf^1)}, \;\; BB^{*} = C_{\Phi(\Xbf^2)},
	A^{*}A = \frac{1}{m_1}J_{m_1}K[\Xbf^1]J_{m_1}, \; B^{*}B = \frac{1}{m_2}J_{m_1}K[\Xbf^2]J_{m_2}, 
	A^{*}B = \frac{1}{\sqrt{m_1m_2}}J_{m_1}K[\Xbf^1,\Xbf^2]J_{m_2}, \; B^{*}A = \frac{1}{\sqrt{m_1m_2}}J_{m_2}K[\Xbf^2,\Xbf^1]J_{m_1}$.
	We then obtained the desired result by applying Theorem \ref{theorem:LogDet-infinite-2}.
\end{proof}

\begin{proof}
	\textbf{of Theorem \ref{theorem:LogDet-RKHS-finite}}
	 This is similar to the proof of Theorem \ref{theorem:LogDet-RKHS-infinite}, except that we apply the formulas obtained from Theorem \ref{theorem:LogDet-finite-2}.
\end{proof}

\begin{proof}
	\textbf{of Theorem \ref{theorem:logdet-alpha-1-AAstar-BBstar-representation-switch}}
	Consider first the case $\gamma = 1, \mu =1$. 
	Using the identity $(I_{\H}+BB^{*})^{-1} = I_{\H}- B(I_{\H_2}+B^{*}B)^{-1}B^{*}$, we have for the first term,
	\begin{align*}
		&	\trX[(BB^{*}+I_{\H})^{-1}(AA^{*} +I_{\H})-I_{\H}] = \trace[(BB^{*}+I_{\H})^{-1} + (BB^{*}+I_{\H})^{-1}AA^{*}-I_{\H}]
		\\
		&= \trace[-B(I_{\H_2}+B^{*}B)^{-1}B^{*} +(I_{\H}-B(I_{\H_2}+B^{*}B)^{-1}B^{*})AA^{*}] 
		\\
		&= \trace[A^{*}A] + \trace[-B^{*}B(I_{\H_2}+B^{*}B)^{-1}  -B^{*}AA^{*}B(I_{\H_2}+B^{*}B)^{-1}]
		\\
		& = \trace[A^{*}A] - \trace(B^{*}B + B^{*}AA^{*}B)(I_{\H_2}+B^{*}B)^{-1}].
	\end{align*}
	For the second term,
	\begin{align*}
		&\log\detX[(BB^{*}+I_{\H})^{-1}(AA^{*}+I_{\H})] = \log\det[(BB^{*}+I_{\H})^{-1}(AA^{*}+I_{\H})]
		\\
		& = -\log\det(BB^{*}+I_{\H}) + \log\det(AA^{*}+I_{\H})
		= \log\det(A^{*}A+I_{\H_1}) - \log\det(B^{*}B+I_{\H_2}).
	\end{align*}
	Combining these two terms, we get
	\begin{align*}
		&d^1_{\logdet}[(I_{\H}+AA^{*}), (I_{\H}+BB^{*})]
		\\
		& 
		= \trX[(BB^{*}+I_{\H})^{-1}(AA^{*}+I_{\H})] - \log\detX[(BB^{*}+I_{\H})^{-1}(AA^{*}+I_{\H})]
		\\
		&=  \trace[A^{*}A] - \trace[(B^{*}B+B^{*}AA^{*}B)(I_{\H_2}+B^{*}B)^{-1}]
		\\
		&\quad -\log\det(A^{*}A+I_{\H_1}) + \log\det(B^{*}B+I_{\H_2})
		\\
		& = - \log\dettwo(I_{\H_1}+A^{*}A) - \log\dettwo[(I_{\H_2}+B^{*}B)^{-1}] - \trace[B^{*}AA^{*}B(I_{\H_2}+B^{*}B)^{-1}].
	\end{align*}
	Here we have applied the following formulas, for $A \in \Tr(\H)$,
	\begin{align}
		\log\dettwo(I+A) &= \log\det[(I+A)\exp(-A)] = \logdet(I+A)-\trace(A),
		\\
		\log\dettwo[(I+A)^{-1}] &= \log\det[(I+A)^{-1}\exp(A(I+A)^{-1})]
		\nonumber
		\\
		& 
		= -\log\det(I+A) + \trace[A(I+A)^{-1}].
	\end{align}
	It follows that for any $\gamma,\mu \in \R, \gamma > 0, \mu > 0$,
	{\small
		\begin{align*}
			&d^1_{\logdet}[(AA^{*}+\gamma I_{\H}), (BB^{*}+\mu I_{\H})]
			\\
			&= (\frac{\gamma}{\mu}-1)\log\frac{\gamma}{\mu}+\trX[(BB^{*}+\mu I_{\H})^{-1}(AA^{*}+\gamma I_{\H})-I_{\H}] 
			- \frac{\gamma}{\mu}\log\detX[(BB^{*}+\mu I_{\H})^{-1}(AA^{*} + \gamma I_{\H})]
			\\
			&= (\frac{\gamma}{\mu}-1 - \log\frac{\gamma}{\mu}) + \frac{\gamma}{\mu}\trace\left[\left(\frac{BB^{*}}{\mu} +I_{\H}\right)^{-1}\left(\frac{AA^{*}}{\gamma} + I_{\H}\right)-I_{\H}\right]
			\\
			& \quad -\frac{\gamma}{\mu}\log\det\left[\left(\frac{BB^{*}}{\mu}+I_{\H}\right)^{-1}\left(\frac{AA^{*}}{\gamma} + I_{\H}\right)\right]
			\\
			&=(\frac{\gamma}{\mu}-1 - \log\frac{\gamma}{\mu}) +\frac{\gamma}{\mu}\trace\left[\frac{A^{*}A}{\gamma}\right] - \frac{\gamma}{\mu}\trace\left[\left(\frac{B^{*}B}{\mu} + \frac{B^{*}AA^{*}B}{\gamma\mu}\right)\left(I_{\H_2}+\frac{B^{*}B}{\mu}\right)^{-1}\right]
			\\
			&\quad +\frac{\gamma}{\mu}\log\det\left(\frac{B^{*}B}{\mu}+I_{\H_2}\right) 
			-\frac{\gamma}{\mu} \log\det\left(\frac{A^{*}A}{\gamma}+I_{\H_1}\right)
			\\
			& = (\frac{\gamma}{\mu}-1 - \log\frac{\gamma}{\mu}) - \frac{\gamma}{\mu}\log\dettwo\left(I_{\H_1}+\frac{A^{*}A}{\gamma}\right) - \frac{\gamma}{\mu}\log\dettwo\left[\left(I_{\H_2}+\frac{B^{*}B}{\mu}\right)^{-1}\right] 
			\\
			&\quad - \trace\left[\frac{B^{*}AA^{*}B}{\gamma \mu}\left(I_{\H_2}+\frac{B^{*}B}{\mu}\right)^{-1}\right].
		\end{align*}
	}
\end{proof}

\begin{proof}
	\textbf{of Theorem \ref{theorem:logdet-alpha-1-AAstar-BBstar-representation-switch-finite}}
	Since $\dim(\H) < \infty$, we utilize the finite-dimensional formula
	\begin{align*}
		&d^1_{\logdet}[(AA^{*} + \gamma I_{\H}), (BB^{*} + \mu I_{\H})]
		\\
		& = \trace[(BB^{*}+\mu I_{\H})^{-1}(AA^{*}+\gamma I_{\H}) -I_{\H}]
		- \log\det[(BB^{*}+\mu I_{\H})^{-1}(AA^{*} + \gamma I_{\H})]
		\\
		& = \frac{\gamma}{\mu}\trace\left[\left(\frac{BB^{*}}{\mu}+ I_{\H}\right)^{-1}\left(\frac{AA^{*}}{\gamma}+ I_{\H}\right)\right] - \dim(\H) 
		\\
		& \quad + \log\det(BB^{*}+\mu I_{\H}) - \log\det(AA^{*} + \gamma I_{\H})
		\\
		&= \frac{\gamma}{\mu}\trace\left[I_{\H}+\frac{AA^{*}}{\gamma}\right]-\frac{\gamma}{\mu}\trace\left[\left(\frac{B^{*}B}{\mu} + \frac{B^{*}AA^{*}B}{\gamma \mu}\right)\left(I_{\H_2}+\frac{B^{*}B}{\mu}\right)^{-1}\right]-\dim(\H) 
		\\
		&\quad + \log\det\left(\frac{BB^{*}}{\mu} + I_{\H}\right) + \dim(\H)\log(\mu) - \log\det\left(\frac{AA^{*}}{\gamma} + I_{\H}\right) - \dim(\H)\log(\gamma)
		\\
		& = \trace\left[\frac{A^{*}A}{\mu}\right] - \frac{\gamma}{\mu}\trace\left[\left(\frac{B^{*}B}{\mu} + \frac{B^{*}AA^{*}B}{\gamma \mu}\right)\left(I_{\H_2}+\frac{B^{*}B}{\mu}\right)^{-1}\right]
		\\
		&\quad+ \log\det\left(\frac{B^{*}B}{\mu} + I_{\H_2}\right)
		- \log\det\left(\frac{A^{*}A}{\gamma} + I_{\H_1}\right)
		+ (\frac{\gamma}{\mu}-1 - \log\frac{\gamma}{\mu})\dim(\H).
	\end{align*}
\end{proof}

\begin{proof}
	\textbf{of Theorem \ref{theorem:logdet-approximate-RKHS-infinite}}
	By Theorem \ref{theorem:CPhi-concentration},
	$||C_{\Phi(\Xbf^i)}||_{\HS(\H_K)} \leq 2\kappa^2$, $||C_{\Phi,\rho_i}||_{\HS(\H_K)} \leq 2\kappa^2$,
	$i=1,2$. Furthermore, with probability at least $1-\delta$, the following hold simultaneously
	{\small
		\begin{align*}
			||C_{\Phi(\Xbf^1)} - C_{\Phi, \rho_1}||_{\HS(\H_K)} &\leq 3\kappa^2\left(\frac{2\log\frac{8}{\delta}}{m_1} + \sqrt{\frac{2\log\frac{8}{\delta}}{m_1}}\right),
			\\
			\text{and }||C_{\Phi(\Xbf^2)} - C_{\Phi,\rho_2}||_{\HS(\H_K)} &\leq 3\kappa^2\left(\frac{2\log\frac{8}{\delta}}{m_2} + \sqrt{\frac{2\log\frac{8}{\delta}}{m_2}}\right).
		\end{align*} 
	}
	Applying Theorem \ref{theorem:logdet-approx-infinite-sequence}, we have with probability at least $1-\delta$,
	\begin{align*}
		&\Delta = \left|d^{\alpha}_{\logdet}[(C_{\Phi(\Xbf^1)} + \gamma I_{\H_K}), (C_{\Phi(\Xbf^2)} + \gamma I_{\H_K}) ] 
		-d^{\alpha}_{\logdet}[(C_{\Phi,\rho_1} + \gamma I_{\H_K}), (C_{\Phi, \rho_2} + \gamma I_{\H_K})]
		\right|
		\\
		&\leq
		\frac{1}{\gamma^2}\left(1+\frac{1}{\gamma}||C_{\Phi(\Xbf^2)}||\right)\left(1+\frac{1}{\gamma}||C_{\Phi,\rho_2}||
		\right)
		\\
		&\quad \times \left[||C_{\Phi(\Xbf^1)} -C_{\Phi,\rho_1}||_{\HS} + \left(1+\frac{1}{2\gamma}||C_{\Phi(\Xbf^1)}|| + \frac{1}{2\gamma}||C_{\Phi,\rho_1}||\right)||C_{\Phi(\Xbf^2)} - C_{\Phi,\rho_2}||_{\HS}\right]
		\\
		&\quad \times \left[||C_{\Phi(\Xbf^1)}||_{\HS}+||C_{\Phi(\Xbf^2)}||_{\HS} + ||C_{\Phi,\rho_1}||_{\HS} + ||C_{\Phi,\rho_2}||_{\HS} \right.
		\\
		&\quad \quad \left.+ \frac{1}{\gamma}(||C_{\Phi(\Xbf^1)}||_{\HS} + ||C_{\Xbf^2}||_{\HS})(||C_{\Phi,\rho_1}||_{\HS} + ||C_{\Phi,\rho_2}||_{\HS})\right]
		\\
		& \leq \frac{1}{\gamma^2}\left(1+ \frac{2\kappa^2}{\gamma}\right)^2
		\left[3\kappa^2\left(\frac{2\log\frac{8}{\delta}}{m_1} + \sqrt{\frac{2\log\frac{8}{\delta}}{m_1}}\right)
		+ \left(1+\frac{2\kappa^2}{\gamma}\right)3\kappa^2\left(\frac{2\log\frac{8}{\delta}}{m_2} + \sqrt{\frac{2\log\frac{8}{\delta}}{m_2}}\right)
		\right]
		\\
		&\quad \times \left(8\kappa^2 + \frac{4\kappa^4}{\gamma}\right)
		\\
		& = \frac{12\kappa^4}{\gamma^2}\left(1+ \frac{2\kappa^2}{\gamma}\right)^2\left(2 +  \frac{\kappa^2}{\gamma}\right)
		\left[\left(\frac{2\log\frac{8}{\delta}}{m_1} + \sqrt{\frac{2\log\frac{8}{\delta}}{m_1}}\right)
		+ \left(1+\frac{2\kappa^2}{\gamma}\right)\left(\frac{2\log\frac{8}{\delta}}{m_2} + \sqrt{\frac{2\log\frac{8}{\delta}}{m_2}}\right)
		\right].
	\end{align*}
\end{proof}

\begin{proof}
	\textbf{of Theorem \ref{theorem:divergence-RKHS-characteristic}}
	From the divergence properties of $D^{\gamma}_{\Rrm,r}$ (\citep{Minh:LogDet2016}), 
	we have 
	\begin{align*}
		D^{\gamma}_{\Rrm,r}[\Ncal(\mu_{\Phi,\rho_1}, C_{\Phi,\rho_1})||\Ncal(\mu_{\Phi,\rho_2}, C_{\Phi,\rho_2})]\geq 0
	\end{align*}
	for any pair $\rho_1,\rho_2$.
	Furthermore,
	\begin{align*}
		D^{\gamma}_{\Rrm,r}[\Ncal(\mu_{\Phi,\rho_1}, C_{\Phi,\rho_1})||\Ncal(\mu_{\Phi,\rho_2}, C_{\Phi,\rho_2})]=0 
		\equivalent \mu_{\Phi,\rho_1} = \mu_{\Phi,\rho_2}, C_{\Phi,\rho_1} = C_{\Phi,\rho_2}.
	\end{align*}
	Since $K$ is characteristic, $\mu_{\Phi,\rho_1} = \mu_{\Phi,\rho_2} \equivalent \rho_1 = \rho_2$.
\end{proof}

\begin{lemma}
	\label{lemma:inverse-AAstar-plus-BBstar}
	Let $A \in \Lcal(\H_1,\H)$, $B \in \Lcal(\H_2, \H)$, $m \in \H$. Then
	{\small
		\begin{align}
			(I_{\H} + AA^{*} + BB^{*})^{-1} =
			I_{\H} - (A \;\; B)\begin{bmatrix}
				I_{\H_1 \oplus \H_2} + 	\begin{pmatrix}A^{*}A & A^{*}B\\B^{*}A & B^{*}B\end{pmatrix}
			\end{bmatrix}^{-1}
			\begin{pmatrix}
				A^{*}\\B^{*}
			\end{pmatrix}. 
		\end{align}
		\begin{align}
			&\la m, (I_{\H}+AA^{*})^{-1}m\ra = ||m||^2 - \la A^{*}m, (I_{\H_1}+A^{*}A)^{-1}A^{*}m\ra,
			\\
			&\la m, (I_{\H}+AA^{*}+BB^{*})^{-1}m\ra = 
			||m||^2 - \left\la \begin{pmatrix}
				A^{*}m\\B^{*}m
			\end{pmatrix},
			\begin{bmatrix}
				I_{\H_1 \oplus \H_2} + 	\begin{pmatrix}A^{*}A & A^{*}B\\B^{*}A & B^{*}B\end{pmatrix}
			\end{bmatrix}^{-1}
			\begin{pmatrix}
				A^{*}m\\B^{*}m
			\end{pmatrix} \right\ra .
		\end{align}
	}
\end{lemma}
\begin{proof}
	We have $AA^{*} + BB^{*} = (A \;\; B)\begin{pmatrix}A^{*}\\B^{*}\end{pmatrix}$, with $(A \;\; B): \H_1 \oplus \H_2 \mapto \H$ defined by $(A \; B)\begin{pmatrix}v_1 \\v_2\end{pmatrix} = Av_1 + Bv_2$, 
	$v_1 \in \H_1, v_2 \in \H_2$. 
	By the Sherman–Morrison–Woodbury formula, we have
	$(I_{\H} + AA^{*})^{-1} = I_{\H} - A(I_{\H_1} + A^{*}A)^{-1}A^{*}$.
	Applying this expression gives
	{\small
		\begin{align*}
			(I_{\H} + AA^{*} + BB^{*})^{-1} &= I_{\H} - (A \;\; B)\begin{bmatrix}
				I_{\H_1 \oplus \H_2} + 	\begin{pmatrix}A^{*}\\B^{*}\end{pmatrix}(A \;\; B)
			\end{bmatrix}^{-1}
			\begin{pmatrix}
				A^{*}\\B^{*}
			\end{pmatrix}
			\\
			& = I_{\H} - (A \;\; B)\begin{bmatrix}
				I_{\H_1 \oplus \H_2} + 	\begin{pmatrix}A^{*}A & A^{*}B\\B^{*}A & B^{*}B\end{pmatrix}
			\end{bmatrix}^{-1}
			\begin{pmatrix}
				A^{*}\\B^{*}
			\end{pmatrix}.
		\end{align*}
	}
	It thus follows that for $m \in \H$,
	{\small
		\begin{align*}
			&\la m, (I_{\H} +AA^{*})^{-1}m\ra = ||m||^2 - \la m, A(I_{\H_1}+A^{*}A)^{-1}A^{*}m\ra = ||m||^2 - \la A^{*}m, (I_{\H_1} + A^{*}A)^{-1}A^{*}m\ra,
			\\
			&\la m, (I_{\H}+AA^{*}+BB^{*})^{-1}m\ra = 
			||m||^2 - \left\la m, (A \;\; B)\begin{bmatrix}
				I_{\H_1 \oplus \H_2} + 	\begin{pmatrix}A^{*}A & A^{*}B\\B^{*}A & B^{*}B\end{pmatrix}
			\end{bmatrix}^{-1}
			\begin{pmatrix}
				A^{*}\\B^{*}
			\end{pmatrix}m\right\ra
			\\
			&= ||m||^2 - \left\la \begin{pmatrix}
				A^{*}m\\B^{*}m
			\end{pmatrix},
			\begin{bmatrix}
				I_{\H_1 \oplus \H_2} + 	\begin{pmatrix}A^{*}A & A^{*}B\\B^{*}A & B^{*}B\end{pmatrix}
			\end{bmatrix}^{-1}
			\begin{pmatrix}
				A^{*}m\\B^{*}m
			\end{pmatrix} \right\ra. 
		\end{align*}
	}
\end{proof}

\begin{lemma}
	\label{lemma:muPhi-norm-HK}
	Let $\Xbf^i = (x^i_j)_{j=1}^{m_i}$, $i=1,2$, be finite samples drawn from $\Xcal$. Then
	\begin{align}
		\Phi(\Xbf^i)^{*}\mu_{\Phi(\Xbf^j)} &= \frac{1}{m_j}K[\Xbf^i,\Xbf^j]\1_{m_j} \in \R^{m_i},
		\\
		||\mu_{\Phi(\Xbf^1)} - \mu_{\Phi(\Xbf^2)}||^2_{\H_K} &= \frac{1}{m_1^2}\1_{m_1}^TK[\Xbf^1]\1_{m_1} + \frac{1}{m_2^2}\1_{m^2}^TK[\Xbf^2]\1_{m_2}
		\\
		&\quad-\frac{2}{m_1m_2}\1_{m_1}^TK[\Xbf^1,\Xbf^2]\1_{m_2}.
		\nonumber
	\end{align}
\end{lemma}
\begin{proof}
	Since $\mu_{\Phi(\Xbf^j)} = \frac{1}{m_j}\Phi(\Xbf^j)\1_{m_j}$, we have
	$\Phi(\Xbf^i)^{*}\mu_{\Phi(\Xbf^j)} 
	=\frac{1}{m_j}\Phi(\Xbf^i)^{*}\Phi(\Xbf^j)\1_{m_j} = \frac{1}{m_j}K[\Xbf^i,\Xbf^j]\1_{m_j}$, giving the first expression. The second expression follows from
	\begin{align*}
		\la \mu_{\Phi(\Xbf^i)}, \mu_{\Phi(\Xbf^j)}\ra_{\H_K} &= \frac{1}{m_i m_j}\la \Phi(\Xbf^i)\1_{m_i}, \Phi(\Xbf^j)\1_{m_j}\ra_{\H_K} 
		\\
		&= \frac{1}{m_im_j}\la \1_{m_i}, \Phi(\Xbf^i)^{*}\Phi(\Xbf^j)\1_{m_j}\ra_{\R^{m_i}} 
		= \frac{1}{m_im_j}\1^T_{m_i}K[\Xbf^i,\Xbf^j]\1_{m_j}.
	\end{align*}
	
\end{proof}

\begin{proof}
	\textbf{of Theorem \ref{theorem:Renyi-RKHS}}
	 By definition of $D^{\gamma}_{\Rrm,r}$,
	{\small
		\begin{align*}
			&D_{\Rrm,r}^{\gamma}[\Ncal(\mu_{\Phi(\Xbf^1)}, C_{\Phi(\Xbf^1)})|| \Ncal(\mu_{\Phi(\Xbf^2)}, C_{\Phi(\Xbf^2)})]
			\\
			&= \frac{1}{2}\la (\mu_{\Phi(\Xbf^1)} - \mu_{\Phi(\Xbf^2)}), [(1-r)(C_{\Phi(\Xbf^1)} +\gamma I_{\H_K}) + r(C_{\Phi(\Xbf^2)}+\gamma I_{\H_K})]^{-1} 
			(\mu_{\Phi(\Xbf^1)} - \mu_{\Phi(\Xbf^2)})\ra_{\H_K}
			\\
			&+ \frac{1}{2}d^{2r-1}_{\logdet}[(C_{\Phi(\Xbf^1)} + \gamma I_{H_K}), (C_{\Phi(\Xbf^2)} + \gamma I_{H_K})].
		\end{align*}
	}
	Define
	$A = \sqrt{\frac{(1-r)}{m_1\gamma}}\Phi(\Xbf^1)J_{m_1}:\R^{m_1} \mapto \H_K$, $B=\sqrt{\frac{r}{m_2\gamma }}\Phi(\Xbf^2)J_{m_2}:\R^{m_2} \mapto \H_K$,
	then $AA^{*} = \frac{1-r}{\gamma}C_{\Phi(\Xbf^1)}$, $BB^{*} = \frac{r}{\gamma}C_{\Phi(\Xbf^2)}$,
	$A^{*}A = \frac{(1-r)}{m_1\gamma}J_{m_1}K[\Xbf^1]J_{m_1}$, $B^{*}B = \frac{r}{m_2\gamma}J_{m_2}K[\Xbf^2]J_{m_2}$,
	$A^{*}B = \frac{1}{\gamma}\sqrt{\frac{r(1-r)}{m_1m_2}}J_{m_1}K[\Xbf^1,\Xbf^2]J_{m_2}$,
	$B^{*}A = \frac{1}{\gamma}\sqrt{\frac{r(1-r)}{m_1m_2}}J_{m_2}K[\Xbf^2,\Xbf^1]J_{m_1}$.
	By Lemma \ref{lemma:muPhi-norm-HK},
	{\small
		\begin{align*}
			A^{*}(\mu_{\Phi(\Xbf^1)} - \mu_{\Phi(\Xbf^2)}) = \sqrt{\frac{1-r}{m_1\gamma}}J_{m_1}
			\left(\frac{1}{m_1}K[\Xbf^1]\1_{m_1}  - \frac{1}{m_2}K[\Xbf^1,\Xbf^2]\1_{m_2}\right) = \frac{1}{\sqrt{\gamma}}\v_1,
			\\
			B^{*}(\mu_{\Phi(\Xbf^1)} - \mu_{\Phi(\Xbf^2)}) = \sqrt{\frac{r}{m_2\gamma}}J_{m_2}
			\left(\frac{1}{m_1}K[\Xbf^2,\Xbf^1]\1_{m_1}  - \frac{1}{m_2}K[\Xbf^2]\1_{m_2}\right) = \frac{1}{\sqrt{\gamma}}\v_2,
		\end{align*}
	}
	By Lemma \ref{lemma:inverse-AAstar-plus-BBstar}, 
	\begin{align*}
		&\la (\mu_{\Phi(\Xbf^1)} - \mu_{\Phi(\Xbf^2)}), [(1-r)(C_{\Phi(\Xbf^1)} +\gamma I_{\H_K}) + r(C_{\Phi(\Xbf^2)}+\gamma I_{\H_K})]^{-1} 
		(\mu_{\Phi(\Xbf^1)} - \mu_{\Phi(\Xbf^2)})\ra_{\H_K}
		\\
		& =\frac{1}{\gamma}\la (\mu_{\Phi(\Xbf^1)} - \mu_{\Phi(\Xbf^2)}), \left[\frac{(1-r)}{\gamma}C_{\Phi(\Xbf^1)} + \frac{r}{\gamma}C_{\Phi(\Xbf^2)}+ I_{\H_K}\right]^{-1} 
		(\mu_{\Phi(\Xbf^1)} - \mu_{\Phi(\Xbf^2)})\ra_{\H_K}
		\\
		& = \frac{1}{\gamma}\la (\mu_{\Phi(\Xbf^1)} - \mu_{\Phi(\Xbf^2)}), (AA^{*}+BB^{*}+I_{\H_K})^{-1}(\mu_{\Phi(\Xbf^1)} - \mu_{\Phi(\Xbf^2)})\ra_{\H_K}
		= \frac{1}{\gamma}(\Delta_1 - \Delta_2).
	\end{align*}
	The quantities $\Delta_1$ and $\Delta_2$ are given by
	{\small
		\begin{align*}
			\Delta_1 =||\mu_{\Phi(\Xbf^1)} - \mu_{\Phi(\Xbf^2)}||^2_{\H_K}=\frac{1}{m_1^2}\1_{m_1}^TK[\Xbf^1]\1_{m_1} + \frac{1}{m_2^2}\1_{m^2}^TK[\Xbf^2]\1_{m_2}
			-\frac{2}{m_1m_2}\1_{m_1}^TK[\Xbf^1,\Xbf^2]\1_{m_2}. 
		\end{align*}
	}
	{\small
		\begin{align*}
			&\Delta_2 =\left\la \begin{pmatrix}
				A^{*}(\mu_{\Phi(\Xbf^1)} - \mu_{\Phi(\Xbf^2)})\\B^{*}(\mu_{\Phi(\Xbf^1)} - \mu_{\Phi(\Xbf^2)})
			\end{pmatrix},
			\begin{bmatrix}
				I_{m_1 \oplus m_2} + 	\begin{pmatrix}A^{*}A & A^{*}B\\B^{*}A & B^{*}B\end{pmatrix}
			\end{bmatrix}^{-1}
			\begin{pmatrix}
				A^{*}(\mu_{\Phi(\Xbf^1)} - \mu_{\Phi(\Xbf^2)})\\B^{*}(\mu_{\Phi(\Xbf^1)} - \mu_{\Phi(\Xbf^2)})
			\end{pmatrix} \right\ra
			\\
			& = \left\la 
			\begin{pmatrix}
				\frac{\v_1}{\sqrt{\gamma}}
				\\
				\frac{\v2}{\sqrt{\gamma}}
			\end{pmatrix},
			\left[I_{m_1+m_2} + 
			\begin{pmatrix}
				\frac{1-r}{m_1\gamma}J_{m_1}K[\Xbf^1]J_{m_1} & \frac{1}{\gamma}\sqrt{\frac{r(1-r)}{m_1m_2}}J_{m_1}K[\Xbf^1,\Xbf^2]J_{m_2}
				\\
				\frac{1}{\gamma}\sqrt{\frac{r(1-r)}{m_1m_2}}J_{m_2}K[\Xbf^2,\Xbf^1]J_{m_1} &\frac{r}{m_2\gamma}J_{m_2}K[\Xbf^2]J_{m_2}
			\end{pmatrix}
			\right]^{-1}
			\begin{pmatrix}
				\frac{\v1}{\sqrt{\gamma}}
				\\
				\frac{\v2}{\sqrt{\gamma}}
			\end{pmatrix}
			\right\ra
			\\
			&
			=\left\la 
			\begin{pmatrix}
				\v_1
				\\
				\v_2
			\end{pmatrix},
			\left[\gamma I_{m_1+m_2} + 
			\begin{pmatrix}
				\frac{1-r}{m_1}J_{m_1}K[\Xbf^1]J_{m_1} & \sqrt{\frac{r(1-r)}{m_1m_2}}J_{m_1}K[\Xbf^1,\Xbf^2]J_{m_2}
				\\
				\sqrt{\frac{r(1-r)}{m_1m_2}}J_{m_2}K[\Xbf^2,\Xbf^1]J_{m_1} &\frac{r}{m_2}J_{m_2}K[\Xbf^2]J_{m_2}
			\end{pmatrix}
			\right]^{-1}
			\begin{pmatrix}
				\v_1
				\\
				\v_2
			\end{pmatrix}
			\right\ra.
		\end{align*}
	}
	
\end{proof}

\begin{proof}
	\textbf{of Theorem \ref{theorem:KL-RKHS}}
	By definition of $\KL^{\gamma}$
	{\small
		\begin{align*}
			&\KL^{\gamma}[\Ncal(\mu_{\Phi(\Xbf^1)}, C_{\Phi(\Xbf^1)})|| \Ncal(\mu_{\Phi(\Xbf^2)}, C_{\Phi(\Xbf^2)})]=\frac{1}{2}d^{1}_{\logdet}[(C_{\Phi(\Xbf^1)} + \gamma I_{H_K}), (C_{\Phi(\Xbf^2)} + \gamma I_{H_K})]
			\\
			&\quad +\frac{1}{2}\la (\mu_{\Phi(\Xbf^1)} - \mu_{\Phi(\Xbf^2)}), [C_{\Phi(\Xbf^2)}+\gamma I_{\H_K}]^{-1} 
			(\mu_{\Phi(\Xbf^1)} - \mu_{\Phi(\Xbf^2)})\ra_{\H_K}.
		\end{align*}
	}
	Let $\Delta_1$ be as in the proof of Theorem \ref{theorem:Renyi-RKHS}, let $B = \sqrt{\frac{1}{m_2\gamma}}\Phi(\Xbf^2)J_{m_2}$, 
	$B^{*}(\mu_{\Phi(\Xbf^1)} - \mu_{\Phi(\Xbf^2)}) = \sqrt{\frac{1}{m_2\gamma}}J_{m_2}
	\left(\frac{1}{m_1}K[\Xbf^2,\Xbf^1]\1_{m_1}  - \frac{1}{m_2}K[\Xbf^2]\1_{m_2}\right) = \frac{1}{\sqrt{\gamma}}\v_2$,
	$B^{*}B = \frac{1}{m_2\gamma}J_{m_2}K[\Xbf^2]J_{m_2}$, and
	\begin{align*}
		&\la (\mu_{\Phi(\Xbf^1)} - \mu_{\Phi(\Xbf^2)}), [C_{\Phi(\Xbf^2)}+\gamma I_{\H_K}]^{-1} 
		(\mu_{\Phi(\Xbf^1)} - \mu_{\Phi(\Xbf^2)})\ra_{\H_K} 
		\\
		& = \frac{1}{\gamma}\la (\mu_{\Phi(\Xbf^1)} - \mu_{\Phi(\Xbf^2)}), [BB^{*}+ I_{\H_K}]^{-1} 
		(\mu_{\Phi(\Xbf^1)} - \mu_{\Phi(\Xbf^2)})\ra_{\H_K} 
		= \frac{1}{\gamma}(\Delta_1 - \Delta_2),
	\end{align*}
	where the quantity $\Delta_2$ is given by
	\begin{align*}
		\Delta_ 2&= \la B^{*}(\mu_{\Phi(\Xbf^1)} - \mu_{\Phi(\Xbf^2)}, (I_{m_2} + B^{*}B)^{-1} B^{*}(\mu_{\Phi(\Xbf^1)}-\mu_{\Phi(\Xbf^2)})\ra
		\\
		&= \v_2^T\left(\gamma I_{m_2} + \frac{1}{m_2}J_{m_2}K[\Xbf^2]J_{m_2}\right)^{-1}\v_2.
	\end{align*}
\end{proof}

\begin{lemma}
	\label{lemma:norm-square-difference}
	Let $A_1, A_2 \in \Lcal(\H)$ and $m_1, m_2 \in \H$. Then
	\begin{align}
		\left|||A_1m_1||^2 - ||A_2m_2||^2\right| \leq &[||A_1||||m_1-m_2|| + ||A_1 - A_2||||m_2||]
		\\
		&\times [||A_1||||m_1|| + ||A_2||||m_2||].
		\nonumber
		\\
		|\la m_1, A_1 m_1\ra - \la m_2, A_2m_2\ra|
		&\leq ||A_1-A_2||||m_1||||m_2|| 
		\\
		&\quad + ||m_1-m_2||[||A_1||||m_1|| + ||A_2||||m_2||].
		\nonumber
	\end{align}
\end{lemma}
\begin{proof}
	For the first inequality,
	\begin{align*}
		&\left|||A_1m_1||^2 - ||A_2m_2||^2\right| \leq 
		\left|||A_1m_1|| - ||A_2m_2||\right|[||A_1 m_1|| + ||A_2m_2||]
		\\
		&\leq ||A_1m_1 - A_2m_2||[||A_1||||m_1|| + ||A_2||||m_2||]  
		\\
		&= ||A_1(m_1-m_2) + (A_1-A_2)m_2||[||A_1||||m_1|| + ||A_2||||m_2||]
		\\
		& \leq [||A_1||||m_1-m_2|| + ||A_1 - A_2||||m_2||][||A_1||||m_1|| + ||A_2||||m_2||].
	\end{align*}
	For the second inequality,
	\begin{align*}
		&|\la m_1, A_1 m_1\ra - \la m_2, A_2m_2\ra| 
		\\
		&\leq |\la m_1, A_1m_1\ra - \la m_1, A_1m_2\ra| +
		|\la m_1, A_1 m_2\ra - \la m_1, A_2m_2\ra| + |\la m_1, A_2 m_2\ra - \la m_2, A_2m_2\ra|  
		\\
		& \leq ||m_1||||A_1||||m_1-m_2|| + ||m_1||||A_1-A_2||||m_2|| + ||m_1-m_2||||A_2||||m_2||
		\\
		& = ||A_1-A_2||||m_1||||m_2|| + ||m_1-m_2||[||A_1||||m_1|| + ||A_2||||m_2||].
	\end{align*}
\end{proof}

\begin{lemma}
	\label{lemma:mean-term-difference}
	Let $\gamma \in \R, \gamma > 0$ be fixed.
	Let $C_1, C_2 \in \Sym^{+}(\H)$, $m_1, m_2 \in \H$. Then
	\begin{align}
		|\la m_1, (C_1+\gamma I)^{-1}m_1\ra - \la m_2, (C_2 + \gamma I)^{-1}m_2\ra|
		&\leq \frac{1}{\gamma^2}||C_1-C_2||||m_1||||m_2||
		\\
		&\quad + \frac{1}{\gamma}||m_1-m_2||[||m_1|| + ||m_2||].
		\nonumber
	\end{align}
\end{lemma}
\begin{proof} Applying Lemma \ref{lemma:norm-square-difference}, 
	we have
	\begin{align*}
		&|\la m_1, (C_1+\gamma I)^{-1}m_1\ra - \la m_2, (C_2 + \gamma I)^{-1}m_2\ra|
		\\
		& \leq ||(C_1+\gamma I)^{-1} - (C_2+\gamma I)^{-1}||||m_1||||m_2|| 
		\\
		& \quad +
		||m_1-m_2||[||(C_1+\gamma I)^{-1}||m_1|| + ||(C_2+\gamma I)^{-1}||||m_2||] 
		\\
		& \leq ||[(C_1+\gamma I)^{-1}(C_1-C_2)(C_2+\gamma I)^{-1}]|||m_1||||m_2|| + \frac{1}{\gamma}||m_1-m_2||[||m_1|| + ||m_2||]
		\\
		& \leq \frac{1}{\gamma^2}||C_1-C_2||||m_1||||m_2|| + \frac{1}{\gamma}||m_1-m_2||[||m_1|| + ||m_2||].
	\end{align*}
\end{proof}

\begin{proposition}
	\label{proposition:Renyi-term-RKHS-approximate}
	Assume Assumptions A1-A4.
	Let $0 \leq r \leq 1$ be fixed.
	Let $\Xbf^i = (x^i_j)_{j=1}^{m_i}$ be independently sampled from 
	$(\Xcal, \rho)$.
	For any $0 < \delta < 1$, with probability at least $1-\delta$,
	\begin{align}
		&|\la (\mu_{\Phi(\Xbf^1)} - \mu_{\Phi(\Xbf^2)}), [(1-r)(C_{\Phi(\Xbf^1)} +\gamma I) + r(C_{\Phi(\Xbf^2)}+\gamma I)]^{-1} 
		(\mu_{\Phi(\Xbf^1)} - \mu_{\Phi(\Xbf^2)})\ra_{\H_K}
		\nonumber
		\\
		&\quad - \la (\mu_{\Phi,\rho_1} - \mu_{\Phi,\rho_2}), [(1-r)(C_{\Phi,\rho_1}+\gamma I) +r(C_{\Phi,\rho_2}+\gamma I)]^{-1}
		(\mu_{\Phi,\rho_1} - \mu_{\Phi, \rho_2})\ra_{\H_K}|
		\nonumber
		\\
		&
		 \leq \left(\frac{12\kappa^4(1-r)}{\gamma^2} + \frac{4\kappa^2}{\gamma}\right)
		\left(\frac{2\log\frac{8}{\delta}}{m_1} + \sqrt{\frac{2\log\frac{8}{\delta}}{m_1}}\right)
		+ \left(\frac{12\kappa^4r}{\gamma^2} + \frac{4\kappa^2}{\gamma}\right)\left(\frac{2\log\frac{8}{\delta}}{m_2} + \sqrt{\frac{2\log\frac{8}{\delta}}{m_2}}\right).
	\end{align}
\end{proposition}
\begin{proof}
	Applying Lemma \ref{lemma:mean-term-difference}, we have
	\begin{align*}
		&\Delta = |\la(\mu_{\Phi(\Xbf^1)} - \mu_{\Phi(\Xbf^2)}), [(1-r)(C_{\Phi(\Xbf^1)} +\gamma I) + r(C_{\Phi(\Xbf^2)}+\gamma I)]^{-1} 
		(\mu_{\Phi(\Xbf^1)} - \mu_{\Phi(\Xbf^2)})\ra_{\H_K}
		\\
		&\quad - \la (\mu_{\Phi,\rho_1} - \mu_{\Phi,\rho_2}), [(1-r)(C_{\Phi,\rho_1}+\gamma I) +r(C_{\Phi,\rho_2}+\gamma I)]^{-1}
		(\mu_{\Phi,\rho_1} - \mu_{\Phi, \rho_2})\ra_{\H_K}|
		\\
		& \leq \frac{1}{\gamma^2}||(1-r)(C_{\Phi(\Xbf^1)} - C_{\Phi,\rho_1}) + r(C_{\Phi(\Xbf^2)}-C_{\Phi,\rho_2})||
		||\mu_{\Phi(\Xbf^1)} - \mu_{\Phi(\Xbf^2)}||||\mu_{\Phi,\rho_1} - \mu_{\Phi, \rho_2}||
		\\
		&+ \frac{1}{\gamma}||(\mu_{\Phi(\Xbf^1)} - \mu_{\Phi(\Xbf^2)}) - (\mu_{\Phi,\rho_1} - \mu_{\Phi, \rho_2})||[||\mu_{\Phi(\Xbf^1)} - \mu_{\Phi(\Xbf^2)}|| + ||\mu_{\Phi,\rho_1} - \mu_{\Phi, \rho_2}||]
		\\
		& \leq \frac{[(1-r)||C_{\Phi(\Xbf^1)} - C_{\Phi,\rho_1}|| + r||C_{\Phi(\Xbf^2)}-C_{\Phi,\rho_2}||]}{\gamma^2}
		[||\mu_{\Phi(\Xbf^1)}|| +||\mu_{\Phi(\Xbf^2)}||][||\mu_{\Phi,\rho_1}||+ ||\mu_{\Phi, \rho_2}||]
		\\
		& \quad + \frac{1}{\gamma}[||\mu_{\Phi(\Xbf^1)} - \mu_{\Phi,\rho_1}||+||\mu_{\Phi(\Xbf^2)} - \mu_{\Phi, \rho_2}||][||\mu_{\Phi(\Xbf^1)}|| + ||\mu_{\Phi(\Xbf^2)}||+||\mu_{\Phi,\rho_1}||+ ||\mu_{\Phi, \rho_2}||]
	\end{align*}
	By Theorem \ref{theorem:CPhi-concentration}, for $i=1,2$,
	$||\mu_{\Phi(\Xbf^i)}||_{\H_K} \leq \kappa$, $||\mu_{\Phi,\rho_i}||_{\H_K} \leq \kappa$, 
	$||C_{\Phi(\Xbf^i)}||_{\HS(\H_K)} \leq 2 \kappa^2$, $||C_{\Phi,\rho_i}||_{\HS(\H_K)} \leq 2\kappa^2$.
	Furthermore, for any $0 < \delta < 1$, with probability $1-\delta$,
	{\small
		\begin{align*}
			||\mu_{\Phi(\Xbf^i)} - \mu_{\Phi,\rho_i}||_{\H_K}  \leq \kappa\left(\frac{2\log\frac{4}{\delta}}{m_i} + \sqrt{\frac{2\log\frac{4}{\delta}}{m_i}}\right),
			||C_{\Phi(\Xbf^i)} - C_{\Phi,\rho_i}||_{\HS} \leq 3\kappa^2\left(\frac{2\log\frac{4}{\delta}}{m_i} + \sqrt{\frac{2\log\frac{4}{\delta}}{m_i}}\right).
		\end{align*}
	}
	It follows that, with probability at least $1-\delta$,
	\begin{align*}
		\Delta &\leq \frac{12\kappa^4}{\gamma^2} \left[(1-r)\left(\frac{2\log\frac{8}{\delta}}{m_1} + \sqrt{\frac{2\log\frac{8}{\delta}}{m_1}}\right)
		+ r\left(\frac{2\log\frac{8}{\delta}}{m_2} + \sqrt{\frac{2\log\frac{8}{\delta}}{m_2}}\right)\right]
		\\
		& \quad + \frac{4\kappa^2}{\gamma}\left[
		\left(\frac{2\log\frac{8}{\delta}}{m_1} + \sqrt{\frac{2\log\frac{8}{\delta}}{m_1}}\right)
		+
		\left(\frac{2\log\frac{8}{\delta}}{m_2} + \sqrt{\frac{2\log\frac{8}{\delta}}{m_2}}\right)\right]
		\\
		& = \left(\frac{12\kappa^4(1-r)}{\gamma^2} + \frac{4\kappa^2}{\gamma}\right)
		\left(\frac{2\log\frac{8}{\delta}}{m_1} + \sqrt{\frac{2\log\frac{8}{\delta}}{m_1}}\right)
		+ \left(\frac{12\kappa^4r}{\gamma^2} + \frac{4\kappa^2}{\gamma}\right)\left(\frac{2\log\frac{8}{\delta}}{m_2} + \sqrt{\frac{2\log\frac{8}{\delta}}{m_2}}\right).
	\end{align*}
\end{proof}

\begin{proof}
	\textbf{of Theorem \ref{theorem:Renyi-approximate-RKHS-infinite}}
	\begin{align*}
		&\Delta = \left|D_{\Rrm,r}^{\gamma}[\Ncal(\mu_{\Phi(\Xbf^1)}, C_{\Phi(\Xbf^1)})|| \Ncal(\mu_{\Phi(\Xbf^2)}, C_{\Phi(\Xbf^2)})] - D_{\Rrm,r}^{\gamma}[\Ncal(\mu_{\Phi,\rho_1}, C_{\Phi, \rho_1})|| \Ncal(\mu_{\Phi,\rho_2}, C_{\Phi, \rho_2})]\right|
		\\
		&\leq\frac{1}{2}|\la (\mu_{\Phi(\Xbf^1)} - \mu_{\Phi(\Xbf^2)}), [(1-r)(C_{\Phi(\Xbf^1)} + \gamma I) + r(C_{\Phi(\Xbf^2)} + \gamma I)]^{-1}(\mu_{\Phi(\Xbf^1)} - \mu_{\Phi(\Xbf^2)})\ra_{\H_K}
		\\
		&\quad - \la (\mu_{\Phi,\rho_1} - \mu_{\Phi,\rho_2}),[(1-r)(C_{\Phi,\rho_1} + \gamma I) + r(C_{\Phi,\rho_2}+\gamma I)]^{-1}(\mu_{\Phi,\rho_1} - \mu_{\Phi,\rho_2})\ra_{\H_K}|
		\\
		&\quad+\frac{1}{2}|d^{2r-1}_{\logdet}[(C_{\Phi(\Xbf^1)} + \gamma I), (C_{\Phi(\Xbf^2)}+\gamma I)]
		-d^{2r-1}_{\logdet}[(C_{\Phi,\rho_1}+\gamma I), (C_{\Phi, \rho_2}+ \gamma I)]|
		= \Delta_1 + \Delta_2.
	\end{align*}
	By Proposition \ref{proposition:Renyi-term-RKHS-approximate}, for any $0 < \delta < 1$, with probability at least $1-\delta$,
	\begin{align*}
	\Delta_1 \leq \left(\frac{12\kappa^4(1-r)}{\gamma^2} + \frac{4\kappa^2}{\gamma}\right)
	\left(\frac{2\log\frac{8}{\delta}}{m_1} + \sqrt{\frac{2\log\frac{8}{\delta}}{m_1}}\right)
	+ \left(\frac{12\kappa^4r}{\gamma^2} + \frac{4\kappa^2}{\gamma}\right)\left(\frac{2\log\frac{8}{\delta}}{m_2} + \sqrt{\frac{2\log\frac{8}{\delta}}{m_2}}\right).
	\end{align*}
	By Theorem \ref{theorem:logdet-approximate-RKHS-infinite}, for any $0 < \delta < 1$, with probability at least 
	$1-\delta$,
	\begin{align*}
		\Delta_2 
		& \leq \frac{6\kappa^4}{\gamma^2}\left(1+ \frac{2\kappa^2}{\gamma}\right)^2\left(2 +  \frac{\kappa^2}{\gamma}\right)
		\left[\left(\frac{2\log\frac{8}{\delta}}{m_1} + \sqrt{\frac{2\log\frac{8}{\delta}}{m_1}}\right)
		+ \left(1+\frac{2\kappa^2}{\gamma}\right)\left(\frac{2\log\frac{8}{\delta}}{m_2} + \sqrt{\frac{2\log\frac{8}{\delta}}{m_2}}\right)
		\right].
	\end{align*}
	It follows that for any $0 < \delta < 1$, with probability at least $1-\delta$,
	\begin{align*}
		&\Delta 
		\leq
		\left(\frac{6\kappa^4(1-r)}{\gamma^2} + \frac{2\kappa^2}{\gamma}\right)
		\left(\frac{2\log\frac{16}{\delta}}{m_1} + \sqrt{\frac{2\log\frac{16}{\delta}}{m_1}}\right)
		+ \left(\frac{6\kappa^4r}{\gamma^2} + \frac{2\kappa^2}{\gamma}\right)\left(\frac{2\log\frac{16}{\delta}}{m_2} + \sqrt{\frac{2\log\frac{16}{\delta}}{m_2}}\right)
		\\
		&\quad + \frac{6\kappa^4}{\gamma^2}\left(1+ \frac{2\kappa^2}{\gamma}\right)^2\left(2 +  \frac{\kappa^2}{\gamma}\right)
		\left[\left(\frac{2\log\frac{16}{\delta}}{m_1} + \sqrt{\frac{2\log\frac{16}{\delta}}{m_1}}\right)
		+ \left(1+\frac{2\kappa^2}{\gamma}\right)\left(\frac{2\log\frac{16}{\delta}}{m_2} + \sqrt{\frac{2\log\frac{16}{\delta}}{m_2}}\right)
		\right]
		\\
		& = C_1(\kappa,\gamma,r)\left(\frac{2\log\frac{16}{\delta}}{m_1} + \sqrt{\frac{2\log\frac{16}{\delta}}{m_1}}\right) + C_2(\kappa,\gamma,r)\left(\frac{2\log\frac{16}{\delta}}{m_2} + \sqrt{\frac{2\log\frac{16}{\delta}}{m_2}}\right).
	\end{align*}
where $C_1(\kappa, \gamma, r) = \frac{6\kappa^4(1-r)}{\gamma^2} + \frac{2\kappa^2}{\gamma} + \frac{6\kappa^4}{\gamma^2}\left(1+ \frac{2\kappa^2}{\gamma}\right)^2\left(2 +  \frac{\kappa^2}{\gamma}\right)$, 
$C_2(\kappa, \gamma, r) = \frac{6\kappa^4r}{\gamma^2} + \frac{2\kappa^2}{\gamma} + \frac{6\kappa^4}{\gamma^2}\left(1+ \frac{2\kappa^2}{\gamma}\right)^2\left(2 +  \frac{\kappa^2}{\gamma}\right)
\left(1+\frac{2\kappa^2}{\gamma}\right)$.
\end{proof}

\subsection{Proofs for the centered Gaussian process setting}
\label{section:proofs-Gaussian-process}

\begin{lemma}
	\label{lemma:A-I+B-inverse-trace-norm}
	Let $A_1, A_2 \in \Tr(\H)$, $B_1, B_2 \in \Sym^{+}(\H) \cap \HS(\H)$. Then
	\begin{align}
		||A_1(I+B_1)^{-1} - A_2(I+B_2)^{-1}||_{\tr} \leq ||A_1 - A_2||_{\tr} + ||A_2||_{\HS}||B_1 - B_2||_{\HS}.
	\end{align}
\end{lemma}
\begin{proof}
	Using the identity $(I+A)^{-1}-(I+B)^{-1} = (I+A)^{-1}(B-A)(I+B)^{-1}$,
	\begin{align*}
		&||A_1(I+B_1)^{-1} - A_2(I+B_2)^{-1}||_{\tr} 
		\\
		&\leq ||(A_1-A_2)(I+B_1)^{-1}||_{\tr} + ||A_2[(I+B_1)^{-1}-(I+B_2)^{-1}]||_{\tr}
		\\
		&\leq ||A_1 - A_2||_{\tr}||(I+B_1)^{-1}|| + ||A_2(I+B_1)^{-1}(B_1 - B_2)(I+B_2)^{-1}||_{\tr}
		\\
		& \leq ||A_1-A_2||_{\tr} + ||A_2||_{\HS}||B_1 - B_2||_{\HS}.
	\end{align*}
\end{proof}

\begin{proof}
	\textbf{of Proposition \ref{proposition:Renyi-RKHS-operator-representation}}
	We apply Theorems \ref{theorem:LogDet-infinite-2} and \ref{theorem:LogDet-finite-2} in the case $\gamma = \mu$.
	Note that $s(\gamma,\gamma,\alpha) =2 \gamma$ and $c(\alpha,\gamma,\gamma) = 0$, so that the formulas
	in Theorems \ref{theorem:LogDet-infinite-2} and \ref{theorem:LogDet-finite-2} are identical.
	For the first expression, we let $A = R_{K^1}^{*}:\H_{K^1} \mapto \Lcal^2(T,\nu)$,
	$B = R_{K^2}^{*}: \Lcal^2(T, \nu) \mapto \H_{K^2}$, with $AA^{*} = C_{K^1}: \Lcal^2(T,\nu)
	\mapto \Lcal^2(T,\nu)$, $A^{*}A = L_{K^1}: \H_{K^1} \mapto \H_{K^1}$,
	$BB^{*} = C_{K^2}:\Lcal^2(T,\nu) \mapto \Lcal^2(T,\nu)$, $B^{*}B = L_{K^2}:\H_{K^2}\mapto \H_{K^2}$,
	$A^{*}B = R_{K^1}R_{K^2}^{*}: \H_{K^2} \mapto \H_{K^1}$. 	With $\alpha = 2r-1$, so that $1+\alpha = 2r$, $1-\alpha = 2(1-r)$,
	\begin{align*}
		&D^{\gamma}_{\Rrm,r}[\Ncal(0,C_{K^1})||\Ncal(0,C_{K^2})] = \frac{1}{2}d^{2r-1}_{\logdet}[(C_{K^1}+\gamma I_{\Lcal^2(T,\nu)}, (C_{K^2}+\gamma I_{\Lcal^2(T,\nu)})]
		\\
		& = \frac{1}{2r(1-r)}\logdet
		\left[\frac{1}{\gamma}\begin{pmatrix}
			(1-r)A^{*}A & \sqrt{r(1-r)}A^{*}B
			\\
			\sqrt{r(1-r)}B^{*}A & rB^{*}B
		\end{pmatrix} + I_{\H_{K^1} \oplus \H_{K^2}}\right]
		\\
		& \quad -\frac{1}{2r}\logdet\left(\frac{A^{*}A}{\gamma} + I_{\H_{K^1}}\right)
		-\frac{1}{2(1-r)}\logdet\left(\frac{B^{*}B}{\gamma} + I_{\H_{K^2}}\right)
		\\
		& =  \frac{1}{2r(1-r)}\logdet
		\left[\frac{1}{\gamma}\begin{pmatrix}
			(1-r)L_{K^1} & \sqrt{r(1-r)}R_{12}
			\\
			\sqrt{r(1-r)}R_{12}^{*} & rL_{K^2}
		\end{pmatrix} + I_{\H_{K^1} \oplus \H_{K^2}}\right]
		\\
		& \quad -\frac{1}{2r}\logdet\left(\frac{L_{K^1}}{\gamma} + I_{\H_{K^1}}\right)
		-\frac{1}{2(1-r)}\logdet\left(\frac{L_{K^2}}{\gamma} + I_{\H_{K^2}}\right)
		\\
		&=  \frac{1}{2r(1-r)}\log\dettwo
		\left[\frac{1}{\gamma}\begin{pmatrix}
			(1-r)L_{K^1} & \sqrt{r(1-r)}R_{12}
			\\
			\sqrt{r(1-r)}R_{12}^{*} & rL_{K^2}
		\end{pmatrix} + I_{\H_{K^1} \oplus \H_{K^2}}\right]
		\\
		& \quad -\frac{1}{2r}\log\dettwo\left(\frac{L_{K^1}}{\gamma} + I_{\H_{K^1}}\right)
		-\frac{1}{2(1-r)}\log\dettwo\left(\frac{L_{K^2}}{\gamma} + I_{\H_{K^2}}\right).
	\end{align*}
For the second expression, we let $A = \frac{1}{\sqrt{m}}S_{K^1,\Xbf}: \H_{K^1} \mapto \R^m$,
$B = \frac{1}{\sqrt{m}}S_{K^2,\Xbf}: \H_{K^2} \mapto \R^m$, 
with $AA^{*}= \frac{1}{m}K^1[\Xbf]:\R^m \mapto \R^m$, $A^{*}A = L_{K^1,\Xbf}:\H_{K^1} \mapto \H_{K^1}$, $BB^{*}= \frac{1}{m}K^2[\Xbf]:\R^m \mapto \R^m$, $B^{*}B = L_{K^2,\Xbf}:\H_{K^2} \mapto \H_{K^2}$,
$A^{*}B = \frac{1}{m}S_{K^1,\Xbf}^{*}S_{K^2,\Xbf} = R_{12,\Xbf}:\H_{K^2} \mapto \H_{K^1}$. Then
\begin{align*}
&D^{\gamma}_{\Rrm,r}\left[\Ncal\left(0, \frac{1}{m}K^1[\Xbf]\right),\Ncal\left(0,\frac{1}{m}K^2[\Xbf]\right)\right]
\\
&= \frac{1}{2}d^{2r-1}_{\logdet}\left[\left(\frac{1}{m}K^1[\Xbf] + \gamma I_{\R^m}\right), \left(\frac{1}{m}K^2[\Xbf] + \gamma I_{\R^m}\right)\right]
\\
&=\frac{1}{2r(1-r)}\logdet
\left[\frac{1}{\gamma}\begin{pmatrix}
	(1-r)L_{K^1,\Xbf} & \sqrt{r(1-r)}R_{12,\Xbf}
	\\
	\sqrt{r(1-r)}R_{12,\Xbf}^{*} & rL_{K^2,\Xbf}
\end{pmatrix} + I_{\H_{K^1} \oplus \H_{K^2}}\right]
\\
& \quad -\frac{1}{2r}\logdet\left(\frac{L_{K^1,\Xbf}}{\gamma} + I_{\H_{K^1}}\right)
-\frac{1}{2(1-r)}\logdet\left(\frac{L_{K^2,\Xbf}}{\gamma} + I_{\H_{K^2}}\right)
\\
&=\frac{1}{2r(1-r)}\log\dettwo
\left[\frac{1}{\gamma}\begin{pmatrix}
	(1-r)L_{K^1,\Xbf} & \sqrt{r(1-r)}R_{12,\Xbf}
	\\
	\sqrt{r(1-r)}R_{12,\Xbf}^{*} & rL_{K^2,\Xbf}
\end{pmatrix} + I_{\H_{K^1} \oplus \H_{K^2}}\right]
\\
& \quad -\frac{1}{2r}\log\dettwo\left(\frac{L_{K^1,\Xbf}}{\gamma} + I_{\H_{K^1}}\right)
-\frac{1}{2(1-r)}\log\dettwo\left(\frac{L_{K^2,\Xbf}}{\gamma} + I_{\H_{K^2}}\right).
\end{align*}
\end{proof}

\begin{proof}
	\textbf{of Proposition \ref{proposition:KL-RKHS-operator-representation}}
	These formulas are obtained by applying Theorems \ref{theorem:logdet-alpha-1-AAstar-BBstar-representation-switch} and 
	\ref{theorem:logdet-alpha-1-AAstar-BBstar-representation-switch-finite}
	in the case
	$\gamma = \mu$ (note that the formulas in these theorems are identical in this case).
	For the first expression, we let $A = R_{K^1}^{*}:\H_{K^1} \mapto \Lcal^2(T,\nu)$,
	$B = R_{K^2}^{*}: \Lcal^2(T, \nu) \mapto \H_{K^2}$, with $AA^{*} = C_{K^1}: \Lcal^2(T,\nu)
	\mapto \Lcal^2(T,\nu)$, $A^{*}A = L_{K^1}: \H_{K^1} \mapto \H_{K^1}$,
	$BB^{*} = C_{K^2}:\Lcal^2(T,\nu) \mapto \Lcal^2(T,\nu)$, $B^{*}B = L_{K^2}:\H_{K^2}\mapto \H_{K^2}$,
	$A^{*}B = R_{K^1}R_{K^2}^{*}: \H_{K^2} \mapto \H_{K^1}$.
	
	For the second expression, we let $A = \frac{1}{\sqrt{m}}S_{K^1,\Xbf}: \H_{K^1} \mapto \R^m$,
	$B = \frac{1}{\sqrt{m}}S_{K^2,\Xbf}: \H_{K^2} \mapto \R^m$, 
	with $AA^{*}= \frac{1}{m}K^1[\Xbf]:\R^m \mapto \R^m$, $A^{*}A = L_{K^1,\Xbf}:\H_{K^1} \mapto \H_{K^1}$, $BB^{*}= \frac{1}{m}K^2[\Xbf]:\R^m \mapto \R^m$, $B^{*}B = L_{K^2,\Xbf}:\H_{K^2} \mapto \H_{K^2}$,
	$A^{*}B = \frac{1}{m}S_{K^1,\Xbf}^{*}S_{K^2,\Xbf} = R_{12,\Xbf}:\H_{K^2} \mapto \H_{K^1}$.
\end{proof}

The following lemma follows from the definition of the Hilbert-Schmidt norm.
\begin{lemma}
	\label{lemma:HS-norm-block-operators}
	Let $A_{11}\in \HS(\H_1), A_{12} \in \HS(\H_2,\H_1), A_{21}\in \HS(\H_1,\H_2), A_{22} \in \HS(\H_2)$. 
	Consider the operator 
	$A = \begin{pmatrix}
		A_{11} & A_{12}
		\\
		A_{21} & A_{22}
	\end{pmatrix}: \H_1 \oplus \H_2 \mapto \H_1 \oplus \H_2$. Then
	$||A||^2_{\HS} = \sum_{i,j=1}^2||A_{ij}||^2_{\HS}$ and $||A||_{\HS} \leq \sum_{i,j=1}^2||A_{ij}||_{\HS}$.
	%
\end{lemma}

\begin{proof}
\textbf{of Theorem \ref{theorem:Renyi-estimate-finite-covariance}}
By Proposition \ref{proposition:Renyi-RKHS-operator-representation},
\begin{align*}
	&\Delta = \left|D_{\Rrm,r}^{\gamma}\left[\Ncal\left(0, \frac{1}{m}K^1[\Xbf]\right)\bigg\vert\bigg\vert \Ncal\left(0, \frac{1}{m}K^2[\Xbf]\right)\right] 
-D_{\Rrm,r}^{\gamma}[\Ncal(0, C_{K^1})||\Ncal(0, C_{K^2})]
\right|
\\
& \leq \frac{1}{2r}\left|\log\dettwo\left(\frac{L_{K^1,\Xbf}}{\gamma} + I_{\H_{K^1}}\right) - \log\dettwo\left(\frac{L_{K^1}}{\gamma} + I_{\H_{K^1}}\right)\right|
\\
&\quad +\frac{1}{2(1-r)}\left|\log\dettwo\left(\frac{L_{K^2,\Xbf}}{\gamma} + I_{\H_{K^2}}\right)
-\log\dettwo\left(\frac{L_{K^2}}{\gamma} + I_{\H_{K^2}}\right)\right|
\\
&\quad + \frac{1}{2r(1-r)}\left|\log\dettwo
\left[\frac{1}{\gamma}\begin{pmatrix}
	(1-r)L_{K^1,\Xbf} & \sqrt{r(1-r)}R_{12,\Xbf}
	\\
	\sqrt{r(1-r)}R_{12,\Xbf}^{*} & rL_{K^2,\Xbf}
\end{pmatrix} + I_{\H_{K^1} \oplus \H_{K^2}}\right]\right.
\\
&\quad \quad \quad \quad \quad \left. - \log\dettwo
\left[\frac{1}{\gamma}\begin{pmatrix}
	(1-r)L_{K^1} & \sqrt{r(1-r)}R_{12}
	\\
	\sqrt{r(1-r)}R_{12}^{*} & rL_{K^2}
\end{pmatrix} + I_{\H_{K^1} \oplus \H_{K^2}}\right]\right|
\\
& = \frac{1}{2r}\Delta_1 + \frac{1}{2(1-r)}\Delta_2 + \frac{1}{2r(1-r)}\Delta_3.
\end{align*}

Define $D = \begin{pmatrix}
	(1-r)L_{K^1} & \sqrt{r(1-r)}R_{12}
	\\
	\sqrt{r(1-r)}R_{12}^{*} & rL_{K^2}
\end{pmatrix}$, $D_{\Xbf} = \begin{pmatrix}
(1-r)L_{K^1,\Xbf} & \sqrt{r(1-r)}R_{12,\Xbf}
\\
\sqrt{r(1-r)}R_{12,\Xbf}^{*} & rL_{K^2,\Xbf}
\end{pmatrix}$. By Theorem \ref{theorem:logdet-Hilbert-Carleman-bound},
\begin{align*}
&\Delta_3 = \left|\log\dettwo\left(I_{\H_{K^1} \oplus \H_{K^2}}+\frac{1}{\gamma}D_{\Xbf}\right) 
- \log\dettwo\left(I_{\H_{K^1}\oplus \H_{K^2}}+\frac{1}{\gamma}D\right)\right|
\\
& \leq
\frac{1}{2\gamma^2}||D_{\Xbf} - D||_{\HS}[||D_{\Xbf}|| + ||D||_{\HS}]. 
\end{align*}
By Lemma \ref{lemma:HS-norm-block-operators} and Proposition \ref{proposition:concentration-TK2K1-empirical}, 
\begin{align*}
&||D_{\Xbf}||_{\HS}^2 = (1-r)^2||L_{K^1}||^2_{\HS} + 2r(1-r)||R_{12}||^2_{\HS} + r^2||L_{K^2}||^2_{\HS}
\\
& \leq  (1-r)^2\kappa_1^4 + 2r(1-r)\kappa_1^1\kappa_2^2 + r^2\kappa_2^4 = [(1-r)\kappa_1^2 + r\kappa_2^2]^2
\imply ||D_{\Xbf}||_{\HS} \leq (1-r)\kappa_1^2 + r\kappa_2^2,
\\
&||D||_{\HS} \leq (1-r)\kappa_1^2 + r\kappa_2^2.
\end{align*}
By Lemma \ref{lemma:HS-norm-block-operators},
\begin{align*}
&||D_{\Xbf}- D||_{\HS}^2 \leq (1-r)^2||L_{K^1,\Xbf} - L_{K^1}||^2_{\HS} + 2r(1-r)||R_{12,\Xbf}-R_{12}||^2_{\HS}
+r^2||L_{K^2,\Xbf} - L_{K^2}||^2_{\HS}.
\end{align*}
By Proposition \ref{proposition:concentration-TK2K1-empirical},  for any $0 < \delta < 1$, with probability at least
$1-\delta$, the following hold simultaneously,
\begin{align*}
&||L_{K^1,\Xbf} - L_{K^1}||_{\HS} \leq \kappa_1^2\left(\frac{2\log\frac{6}{\delta}}{m} + \sqrt{\frac{2\log\frac{6}{\delta}}{m}}\right), 
\\
&||R_{12,\Xbf}-R_{12}||_{\HS} \leq \kappa_1\kappa_2\left(\frac{2\log\frac{6}{\delta}}{m} + \sqrt{\frac{2\log\frac{6}{\delta}}{m}}\right), 
\\
&||L_{K^2,\Xbf} - L_{K^2}||_{\HS} \leq \kappa_2^2\left(\frac{2\log\frac{6}{\delta}}{m} + \sqrt{\frac{2\log\frac{6}{\delta}}{m}}\right).
\end{align*}
It follows that, with probability at least $1-\delta$,
\begin{align*}
||D_{\Xbf}-D||_{\HS} \leq [(1-r)\kappa_1^2 + r\kappa_2^2]\left(\frac{2\log\frac{6}{\delta}}{m} + \sqrt{\frac{2\log\frac{6}{\delta}}{m}}\right).
\end{align*}
Consequently, with probability at least $1-\delta$,
\begin{align*}
\Delta_3 \leq \frac{1}{\gamma^2}[(1-r)\kappa_1^2 + r\kappa_2^2]^2\left(\frac{2\log\frac{6}{\delta}}{m} + \sqrt{\frac{2\log\frac{6}{\delta}}{m}}\right).
\end{align*}
	Similarly, by Theorem \ref{theorem:logdet-Hilbert-Carleman-bound} and Proposition \ref{proposition:concentration-TK2K1-empirical},
\begin{align*}
	&\Delta_1 = \left|\log\dettwo\left(I_{\H_{K^1}}+\frac{L_{K^1,\Xbf}}{\gamma}\right) 
	- \log\dettwo\left(I_{\H_{K^1}}+ \frac{L_{K^1}}{\gamma}\right)\right| 
	\\
	&\leq \frac{1}{2\gamma^2}[||L_{K^1,\Xbf}||_{\HS} + ||L_{K^1}||_{\HS}]||L_{K^1,\Xbf} - L_{K^1}||_{\HS}
	\leq \frac{\kappa_1^4}{\gamma^2}\left(\frac{2\log\frac{6}{\delta}}{m} + \sqrt{\frac{2\log\frac{6}{\delta}}{m}}\right). 
\end{align*}
\begin{align*}
	&\Delta_2 = \left|\log\dettwo\left(I_{\H_{K^2}}+\frac{L_{K^2,\Xbf}}{\gamma}\right) 
	- \log\dettwo\left(I_{\H_{K^2}}+ \frac{L_{K^2}}{\gamma}\right)\right| 
	\leq \frac{\kappa_2^4}{\gamma^2}\left(\frac{2\log\frac{6}{\delta}}{m} + \sqrt{\frac{2\log\frac{6}{\delta}}{m}}\right). 
\end{align*}
It follows that, with probability at least $1-\delta$,
\begin{align*}
\Delta \leq \frac{1}{2\gamma^2}\left[\frac{\kappa_1^4}{r} + \frac{\kappa_2^4}{1-r} + \frac{[(1-r)\kappa_1^2 + r\kappa_2^2]^2}{r(1-r)}\right]\left(\frac{2\log\frac{6}{\delta}}{m} + \sqrt{\frac{2\log\frac{6}{\delta}}{m}}\right).
\end{align*}
\end{proof}

\begin{proof}
	\textbf{of Theorem \ref{theorem:KL-estimate-finite-covariance}}
	By Proposition 
	\ref{proposition:KL-RKHS-operator-representation},
	\begin{align*}
		\Delta &= \left|D_{\KL}^{\gamma}\left[\Ncal\left(0, \frac{1}{m}K^1[\Xbf]\right)\bigg\vert\bigg\vert \Ncal\left(0, \frac{1}{m}K^2[\Xbf]\right)\right] 
		-D_{\KL}^{\gamma}[\Ncal(0, C_{K^1})||\Ncal(0, C_{K^2})]
		\right|
		\\
		& \leq \frac{1}{2}\left|\log\dettwo\left(I+\frac{L_{K^1,\Xbf}}{\gamma}\right) - \log\dettwo\left(I+ \frac{L_{K^1}}{\gamma}\right)\right|
		\\
		&\quad +\frac{1}{2}\left|\log\dettwo\left(I+\frac{L_{K^2,\Xbf}}{\gamma}\right)^{-1} - \log\dettwo\left(I+ \frac{L_{K^2}}{\gamma}\right)^{-1}\right|
		\\
		& \quad+ \frac{1}{2\gamma^2}\trace\left[{R_{12,\Xbf}^{*}R_{12,\Xbf}}\left(I+\frac{L_{K^2,\Xbf}}{\gamma}\right)^{-1}\right]- \trace\left[{R_{12}^{*}R_{12}}\left(I+\frac{L_{K^2}}{\gamma}\right)^{-1}\right]
		\\
		& = \frac{1}{2}\Delta_1 + \frac{1}{2}\Delta_2 + \frac{1}{2\gamma^2}\Delta_3.
	\end{align*}
	By Lemma \ref{lemma:A-I+B-inverse-trace-norm},
	\begin{align*}
		&\Delta_3 = \left|\trace\left[{R_{12,\Xbf}^{*}R_{12,\Xbf}}\left(I+\frac{L_{K^2,\Xbf}}{\gamma}\right)^{-1}\right]- \trace\left[{R_{12}^{*}R_{12}}\left(I+\frac{L_{K^2}}{\gamma}\right)^{-1}\right]\right|
		\\
		& \leq \left\|{R_{12,\Xbf}^{*}R_{12,\Xbf}}\left(I+\frac{L_{K^2,\Xbf}}{\gamma}\right)^{-1}
		- {R_{12}^{*}R_{12}}\left(I+\frac{L_{K^2}}{\gamma}\right)^{-1}\right\|_{\tr}
		\\
		& \leq ||R_{12,\Xbf}^{*}R_{12,\Xbf} - R_{12}^{*}R_{12}||_{\tr} + \frac{1}{\gamma}||R_{12}^{*}R_{12}||_{\HS}
		||L_{K^2,\Xbf} - L_{K^2}||_{\HS}.
	\end{align*}
	We have $||R_{12}^{*}R_{12}||_{\HS} = ||R_{12}||^2_{\HS} \leq \kappa_1^2\kappa_2^2$.
	By Proposition \ref{proposition:concentration-TK2K1-empirical}, for any $0 < \delta < 1$, 
	with probability at least $1-\delta$, the following hold simultaneously,
	\begin{align*}
		||R_{12,\Xbf}^{*}R_{12,\Xbf} - R_{12}^{*}R_{12}||_{\tr} &\leq  2\kappa_1^2\kappa_2^2\left[ \frac{2\log\frac{6}{\delta}}{m} + \sqrt{\frac{2\log\frac{6}{\delta}}{m}}\right],
		\\ 
		||L_{K^1,\Xbf} - L_{K^1}||_{\HS} &\leq \kappa_2^2\left(\frac{2\log\frac{6}{\delta}}{m} + \sqrt{\frac{2\log\frac{6}{\delta}}{m}}\right). 
		\\
		||L_{K^2,\Xbf} - L_{K^2}||_{\HS} &\leq \kappa_2^2\left(\frac{2\log\frac{6}{\delta}}{m} + \sqrt{\frac{2\log\frac{6}{\delta}}{m}}\right). 
	\end{align*}
	It follows that $\Delta_3 \leq \kappa_1^2\kappa_2^2\left(2+\frac{\kappa_2^2}{\gamma}\right)\left(\frac{2\log\frac{6}{\delta}}{m} + \sqrt{\frac{2\log\frac{6}{\delta}}{m}}\right)$.
	Similarly, by Theorem \ref{theorem:logdet-Hilbert-Carleman-bound},
	\begin{align*}
		&\Delta_1 = \left|\log\dettwo\left(I+\frac{L_{K^1,\Xbf}}{\gamma}\right) - \log\dettwo\left(I+ \frac{L_{K^1}}{\gamma}\right)\right| 
		\\
		&\leq \frac{1}{2\gamma^2}[||L_{K^1,\Xbf}||_{\HS} + ||L_{K^1}||_{\HS}]||L_{K^1,\Xbf} - L_{K^1}||_{\HS}
		\leq \frac{\kappa_1^4}{\gamma^2}\left(\frac{2\log\frac{6}{\delta}}{m} + \sqrt{\frac{2\log\frac{6}{\delta}}{m}}\right). 
	\end{align*}
	\begin{align*}
		&\Delta_2 = \left|\log\dettwo\left(I+\frac{L_{K^2,\Xbf}}{\gamma}\right)^{-1} - \log\dettwo\left(I+ \frac{L_{K^2}}{\gamma}\right)^{-1}\right|
		\\
		&\leq \frac{1}{2\gamma^2}[||L_{K^2,\Xbf}||_{\HS} + ||L_{K^2}||_{\HS}]||L_{K^2,\Xbf} - L_{K^2}||_{\HS}
		\leq \frac{\kappa_2^4}{\gamma^2}\left(\frac{2\log\frac{6}{\delta}}{m} + \sqrt{\frac{2\log\frac{6}{\delta}}{m}}\right).
	\end{align*}
	Combining the three bounds for $\Delta_i$, $i=1,2,3$, we have, with probability at least $1-\delta$,
	\begin{align*}
	\Delta \leq \frac{1}{2\gamma^2}\left[\kappa_1^4 + \kappa_2^4 + \kappa_1^2\kappa_2^2\left(2+\frac{\kappa_2^2}{\gamma}\right)\right]\left(\frac{2\log\frac{6}{\delta}}{m} + \sqrt{\frac{2\log\frac{6}{\delta}}{m}}\right).
	\end{align*}
\end{proof}

\begin{proof}
	\textbf{of Theorem \ref{theorem:KL-estimate-unknown-1}}
	By Proposition \ref{proposition:concentration-empirical-covariance}, 
	$||K^i[\Xbf]||_F \leq m\kappa_i^2$, $i=1,2$, and
	for any $0 < \delta < 1$, the following hold simultaneously,
	\begin{align*}
		&||\hat{K^i}_{\Wbf^i}[\Xbf] - K^i[\Xbf]||_F \leq \frac{4\sqrt{3}m\kappa_i^2}{\sqrt{N}\delta},
		\;\;
		||\hat{K^i}_{\Wbf^i}[\Xbf]||_F \leq \frac{4m\kappa_i^2}{\delta}, \;\; i=1,2.
	\end{align*}
	By Theorem \ref{theorem:logdet-approx-infinite-sequence},
	\begin{align*}
		\Delta = &\left|\KL^{\gamma}\left[\Ncal\left(0, \frac{1}{m}\hat{K}^1_{\Wbf^1}[\Xbf]\right), \Ncal\left(0, \frac{1}{m}\hat{K}^2_{\Wbf^2}[\Xbf]\right)\right]
		\right.
		\nonumber
		\\
		&\quad \quad\left.
		- \KL^{\gamma}\left[\Ncal\left(0, \frac{1}{m}K^1[\Xbf]\right), \Ncal\left(0, \frac{1}{m}K^2[\Xbf]\right)\right]
		\right|
		\leq \frac{1}{\gamma^2}D_1D_2,
	\end{align*}
	where $D_1$ and $D_2$ are as follows
	\begin{align*}
		&D_1 = \left(1 + \frac{1}{m\gamma}||\hat{K}^2_{\Wbf^2}[\Xbf]||_F\right)
		\left(1 + \frac{1}{m\gamma}||K^2[\Xbf]||_F\right)
		\\
		&\times \left[\frac{1}{m}||\hat{K}^1_{\Wbf^1}[\Xbf]-K^1[\Xbf]||_F + 
		\left(1 + \frac{1}{2m\gamma}||\hat{K}^1_{\Wbf^1}[\Xbf]||_F
		+ \frac{1}{2m\gamma}||K^1[\Xbf]||_F\right)
		\frac{1}{m}||\hat{K}^2_{\Wbf^1}[\Xbf]-K^2[\Xbf]||_F
		\right]
		\\
		& \leq \left(1+\frac{4\kappa_2^2}{\delta\gamma}\right)\left(1+\frac{\kappa_2^2}{\gamma}\right)\left[\frac{4\sqrt{3}\kappa_1^2}{\sqrt{N}\delta} + \left(1 + \frac{2\kappa_1^2}{\delta\gamma} + \frac{\kappa_1^2}{2\gamma}\right)\frac{4\sqrt{3}\kappa_2^2}{\sqrt{N}\delta}\right]
		\\
		&
		= \left(1+\frac{4\kappa_2^2}{\delta\gamma}\right)\left(1+\frac{\kappa_2^2}{\gamma}\right)\left[\kappa_1^2 + \left(1 + \frac{2\kappa_1^2}{\delta\gamma} + \frac{\kappa_1^2}{2\gamma}\right)\kappa_2^2\right]\frac{4\sqrt{3}}{\sqrt{N}\delta},
	\end{align*}
	\begin{align*}
		D_2 &= \frac{1}{m}||\hat{K}^1_{\Wbf^1}[\Xbf]||_F + \frac{1}{m}||\hat{K}^2_{\Wbf^2}[\Xbf]||_F
		+ \frac{1}{m}||{K}^1[\Xbf]||_F + \frac{1}{m}||{K}^2[\Xbf]||_F
		\\
		&\quad+\frac{1}{m^2\gamma}(||\hat{K}^1_{\Wbf^1}[\Xbf]||_F + ||\hat{K}^2_{\Wbf^2}[\Xbf]||_F)(||{K}^1[\Xbf]||_F + ||{K}^2[\Xbf]||_F)
		\\
		& \leq (\kappa_1^2 + \kappa_2^2)\left(1+\frac{4}{\delta}\right) +\frac{4}{\delta\gamma}(\kappa_1^2 + \kappa_2^2)^2.
	\end{align*}
	Combining these expressions gives the final result.
\end{proof}

\begin{proof}
	\textbf{of Theorem \ref{theorem:KL-estimate-unknown-2}}
	Define the following quantities
	\begin{align*}
		\Delta_1 = &\left|\KL^{\gamma}\left[\Ncal\left(0, \frac{1}{m}\hat{K}^1_{\Wbf^1}[\Xbf]\right), \Ncal\left(0, \frac{1}{m}\hat{K}^2_{\Wbf^2}[\Xbf]\right)\right]
		\right.
		\nonumber
		\\
		&\quad \quad\left.
		- \KL^{\gamma}\left[\Ncal\left(0, \frac{1}{m}K^1[\Xbf]\right), \Ncal\left(0, \frac{1}{m}K^2[\Xbf]\right)\right]
		\right|,
	\end{align*}
	\begin{align*}
		\Delta_2 = 	&\left|\KL^{\gamma}\left[\Ncal\left(0, \frac{1}{m}{K}^1[\Xbf]\right), \Ncal\left(0, \frac{1}{m}{K}^2[\Xbf]\right)\right]
		\right.
		\nonumber 
		\\
		&\quad\left.- \KL^{\gamma}[\Ncal(0,C_{K^1}), \Ncal(0,C_{K^2})]\right|.
	\end{align*}
	By the triangle inequality,
	\begin{align*}
		\Delta = 	&\left|\KL^{\gamma}\left[\Ncal\left(0, \frac{1}{m}\hat{K}^1_{\Wbf^1}[\Xbf]\right), \Ncal\left(0, \frac{1}{m}\hat{K}^2_{\Wbf^2}[\Xbf]\right)\right]
		\right.
		\nonumber 
		\\
		&\quad\left.- \KL^{\gamma}[\Ncal(0,C_{K^1}), \Ncal(0,C_{K^2})]\right| \leq \Delta_1 + \Delta_2.
	\end{align*}
	By Theorem \ref{theorem:KL-estimate-unknown-1}, 
	for any $0 < \delta <1$, with probability at least $1-\frac{\delta}{2}$,
	\begin{align*}
		\Delta_1 \leq \frac{2D}{\gamma^2\sqrt{N}\delta},
	\end{align*}
	where
	\begin{align*}
		D =4\sqrt{3} \left(1+\frac{8\kappa_2^2}{\delta\gamma}\right)\left(1+\frac{\kappa_2^2}{\gamma}\right)\left[\kappa_1^2 + \left(1 + \frac{4\kappa_1^2}{\delta\gamma} + \frac{\kappa_1^2}{2\gamma}\right)\kappa_2^2\right]\left[
		(\kappa_1^2 + \kappa_2^2)\left(1+\frac{8}{\delta}\right) +\frac{8}{\delta\gamma}(\kappa_1^2 + \kappa_2^2)^2\right].
	\end{align*}
	By Theorem \ref{theorem:KL-estimate-finite-covariance}, for any $0 < \delta <1$, with probability at least $1-\frac{\delta}{2}$,
	\begin{align*}
		\Delta_2 \leq \frac{1}{2\gamma^2}\left[\kappa_1^4 + \kappa_2^4 + \kappa_1^2\kappa_2^2\left(2+\frac{\kappa_2^2}{\gamma}\right)\right]\left(\frac{2\log\frac{12}{\delta}}{m} + \sqrt{\frac{2\log\frac{12}{\delta}}{m}}\right).
	\end{align*}
	Combining these two expressions gives the final result.
\end{proof}

\subsection{Proofs for the general Gaussian process setting}
\label{section:proof-general-Gaussian}

\begin{lemma}
	\label{lemma:mean-square-concentration}
	Assume Assumptions B1-B3 and B5-B7. Let $\Xbf$ be independently sampled from $(T,\nu)$. For any $0 < \delta < 1$, with probability at least $1-\delta$,
	\begin{align}
		\left|\frac{1}{m}||\mu[\Xbf]||^2_{\R^m} - ||\mu||^2_{\Lcal^2(T,\nu)}\right| \leq B^2\left(\frac{2\log\frac{2}{\delta}}{m} + \sqrt{\frac{2\log\frac{2}{\delta}}{m}}\right).
	\end{align}
\end{lemma}
\begin{proof} Define the random variable $\eta_{\mu}:(T,\nu)\mapto \R$ by $\eta_{\mu}(x) = \mu(x)^2$.
	Then $|\eta_{\mu}(x)| \leq B^2$ $\forall x \in T$, $\frac{1}{m}||\mu[\Xbf]||^2_{\R^m} = \frac{1}{m}\sum_{j=1}^m\mu(x_j)^2$ and $\bE[\eta(x)] = ||\mu||^2_{\Lcal^2(T,\nu)}$.
	The result then follows from Proposition \ref{proposition:Pinelis}.
\end{proof}

\begin{lemma}
	\label{lemma:RK-concentration}
	Assume Assumptions B1-B3 and B5-B7. Let $\Xbf$ be independently sampled from $(T,\nu)$. 
	Then $\frac{1}{m}||S_{\Xbf}^{*}\mu[\Xbf]||_{\H_K}\leq \kappa B$, $||R_K\mu||_{\H_K} \leq \kappa B$.
	For any $0 < \delta < 1$, with probability at least $1-\delta$,
	\begin{align}
		\left\|\frac{1}{m}S^{*}_{\Xbf}\mu[\Xbf] - R_K\mu\right\|_{\H_K} \leq \kappa B\left(\frac{2\log\frac{2}{\delta}}{m} + \sqrt{\frac{2\log\frac{2}{\delta}}{m}}\right).
	\end{align}	
\end{lemma}
\begin{proof}
	Define the random variable $\eta_{K,\mu}:(T,\nu) \mapto \H_K$ by $\eta_{K,\mu}(x) = \mu(x)K_x$. Then
	$||\eta_{K,\mu}(x)||_{\H_K} = \sqrt{\mu(x)^2K(x,x)} \leq \kappa B$ $\forall x \in T$,
	$\frac{1}{m}\sum_{j=1}^m\eta_{K,\mu}(x_j) = \frac{1}{m}\sum_{j=1}^m\mu(x_j)K_{x_j} = \frac{1}{m}S^{*}_{\Xbf}\mu[\Xbf]$, and $\bE[\eta_{K,\mu}(x)] =\int_{T}K(x,t)\mu(t)d\nu(t) = R_K\mu(x)$, with
	\begin{align*}
		\frac{1}{m}||S^{*}_{\Xbf}\mu[\Xbf]||_{\H_K} \leq \kappa B, \;||R_K\mu||_{\H_K} \leq ||R_K:\Lcal^2(T,\nu)\mapto \H_K||\;||\mu||_{\Lcal^2(T,\nu)} \leq \kappa B.
	\end{align*} The result then follows from Proposition \ref{proposition:Pinelis}.
\end{proof}

\begin{proof}
	\textbf{of Proposition \ref{proposition:Renyi-mean-RKHS-representation}}
	Let $\ubf_1 = \sqrt{1-r}R_{K^1}(\mu^1-\mu^2) \in \H_{K^1}$, $\ubf_2 = \sqrt{r}R_{K^2}(\mu^1-\mu^2)\in \H_{K^2}$.
	By Lemma \ref{lemma:inverse-AAstar-plus-BBstar},
	\begin{align*}
		&\la \mu^1-\mu^2, [(1-r)(C_{K^1} + \gamma I) + r(C_{K^2}+\gamma I)]^{-1}(\mu^1 - \mu^2)\ra_{\Lcal^2(T,\nu)}
		\\
		& 
		= \frac{1}{\gamma}\left\la \mu^1-\mu^2, \left(\frac{1-r}{\gamma}R_{K^1}^{*}R_{K^1} + \frac{r}{\gamma}R_{K^2}^{*}R_{K^2} + I\right)^{-1}(\mu^1-\mu^2)\right\ra_{\Lcal^2(T,\nu)}
		\\
		& = \frac{1}{\gamma}||\mu^1-\mu^2||^2_{\Lcal^2(T,\nu)} 
		\\
		&\quad -\frac{1}{\gamma^2}\left\la
		\begin{pmatrix}
			\ubf_1
			\\
			\ubf_2
		\end{pmatrix},
		\left[\frac{1}{\gamma}\begin{pmatrix} (1-r)R_{K^1}R_{K^1}^{*} & \sqrt{r(1-r)}R_{K^1}R_{K^2}^{*}
			\\
			\sqrt{r(1-r)}R_{K^2}R_{K^1}^{*} & rR_{K^2}R_{K^2}^{*}
		\end{pmatrix}
		+ I_{\H_{K^1} \oplus \H_{K^2}}\right]^{-1}\begin{pmatrix}
			\ubf_1
			\\
			\ubf_2
		\end{pmatrix}\right\ra_{\H_{K^1}\oplus \H_{K^2}}
		\\
		& = \frac{1}{\gamma}||\mu^1-\mu^2||^2_{\Lcal^2(T,\nu)} 
		\\
		&\quad - \frac{1}{\gamma^2}\left\la
		\begin{pmatrix}
			\ubf_1
			\\
			\ubf_2
		\end{pmatrix},
		\left[\frac{1}{\gamma}\begin{pmatrix} (1-r)L_{K^1} & \sqrt{r(1-r)}R_{12}
			\\
			\sqrt{r(1-r)}R_{12}^{*} & r L_{K^2}
		\end{pmatrix}
		+ I_{\H_{K^1} \oplus \H_{K^2}}\right]^{-1}\begin{pmatrix}
			\ubf_1
			\\
			\ubf_2
		\end{pmatrix}\right\ra_{\H_{K^1}\oplus \H_{K^2}}.
	\end{align*}
	Similarly, for the empirical version,
	let $\ubf_{1,\Xbf} = \frac{\sqrt{1-r}}{m}S_{1,\Xbf}^{*}(\mu^1[\Xbf] - \mu^2[\Xbf]) \in \H_{K^1}$,
	$\ubf_{2,\Xbf} = \frac{\sqrt{r}}{m}S_{2,\Xbf}^{*}(\mu^1[\Xbf] - \mu^2[\Xbf]) \in \H_{K^2}$, then
	\begin{align*}
		&\la \mu^1[\Xbf] - \mu^2[\Xbf], [(1-r)(K^1[\Xbf]+m\gamma I) + r(K^2[\Xbf] + m\gamma I)]^{-1}(\mu^1[\Xbf] - \mu^2[\Xbf])\ra_{\R^m}
		\\
		& 
		= \frac{1}{m\gamma}
		\left\la \mu^1[\Xbf] - \mu^2[\Xbf], \left[\frac{(1-r)}{m\gamma}S_{1,\Xbf}S_{1,\Xbf}^{*} + \frac{r}{m\gamma}S_{2,\Xbf}S_{2,\Xbf}^{*} + I\right]^{-1}(\mu^1[\Xbf] - \mu^2[\Xbf])\right\ra_{\R^m}
		\\
		& = \frac{1}{m\gamma}||\mu^1[\Xbf] -\mu^2[\Xbf]||^2_{\R^m} 
		\\
		& - \frac{1}{\gamma^2}\left\la \begin{pmatrix}
			\ubf_{1,\Xbf}
			\\
			\ubf_{2,\Xbf}
		\end{pmatrix} \left[\frac{1}{m\gamma}\begin{pmatrix}
			(1-r)S_{1,\Xbf}^{*}S_{1,\Xbf} & \sqrt{r(1-r)}S_{1,\Xbf}^{*}S_{2,\Xbf}
			\\
			\sqrt{r(1-r)}S_{2,\Xbf}^{*}S_{1,\Xbf} & rS_{2,\Xbf}^{*}S_{2,\Xbf}
		\end{pmatrix} + I_{\H_{K^1}\oplus \H_{K^2}}\right]^{-1} \begin{pmatrix}
			\ubf_{1,\Xbf}
			\\
			\ubf_{2,\Xbf}
		\end{pmatrix}\right\ra_{\H_{K^1} \oplus \H_{K^2}}
		\\
		& = \frac{1}{m\gamma}||\mu^1[\Xbf] -\mu^2[\Xbf]||^2_{\R^m} 
		\\
		& \quad - \frac{1}{\gamma^2}\left\la \begin{pmatrix}
			\ubf_{1,\Xbf}
			\\
			\ubf_{2,\Xbf}
		\end{pmatrix} \left[\frac{1}{\gamma}\begin{pmatrix}
			(1-r)L_{K^1,\Xbf} & \sqrt{r(1-r)}R_{12,\Xbf}
			\\
			\sqrt{r(1-r)}R_{12,\Xbf}^{*} & rL_{K^2,\Xbf}
		\end{pmatrix} + I_{\H_{K^1}\oplus \H_{K^2}}\right]^{-1} \begin{pmatrix}
			\ubf_{1,\Xbf}
			\\
			\ubf_{2,\Xbf}
		\end{pmatrix}\right\ra_{\H_{K^1} \oplus \H_{K^2}}.
	\end{align*}
\end{proof}

\begin{proof}
	\textbf{of Proposition \ref{proposition:Renyi-mean-Gaussian-process-estimate}}
	Let $\ubf = \begin{pmatrix}\ubf_1\\ \ubf_2\end{pmatrix}$, $\ubf_{\Xbf} = \begin{pmatrix}\ubf_{1,\Xbf}\\ \ubf_{2,\Xbf}\end{pmatrix}$. As in the proof of Theorem \ref{theorem:Renyi-estimate-finite-covariance}, define $D = \begin{pmatrix}
		(1-r)L_{K^1} & \sqrt{r(1-r)}R_{12}
		\\
		\sqrt{r(1-r)}R_{12}^{*} & rL_{K^2}
	\end{pmatrix}$, $D_{\Xbf} = \begin{pmatrix}
		(1-r)L_{K^1,\Xbf} & \sqrt{r(1-r)}R_{12,\Xbf}
		\\
		\sqrt{r(1-r)}R_{12,\Xbf}^{*} & rL_{K^2,\Xbf}
	\end{pmatrix}$.
	By Proposition \ref{proposition:Renyi-mean-RKHS-representation},
	\begin{align*}
		&\Delta = \left|\la \mu^1[\Xbf] - \mu^2[\Xbf], [(1-r)(K^1[\Xbf]+m\gamma I) + r(K^2[\Xbf] + m\gamma I)]^{-1}(\mu^1[\Xbf] - \mu^2[\Xbf])\ra_{\R^m}\right.
		\nonumber
		\\
		&\quad \left.	-\la \mu^1-\mu^2, [(1-r)(C_{K^1} + \gamma I) + r(C_{K^2}+\gamma I)]^{-1}(\mu^1 - \mu^2)\ra_{\Lcal^2(T,\nu)}\right|
		\\
		& \leq \frac{1}{\gamma}\left|\frac{1}{m}||\mu^1[\Xbf]-\mu^2[\Xbf]||^2_{\R^m} - ||\mu^1-\mu^2||^2_{\Lcal^2(T,\nu)}\right|
		\\
		& + \frac{1}{\gamma^2}\left|\left\la \ubf_{\Xbf}, \left(\frac{1}{\gamma}D_{\Xbf} + I\right)^{-1}\ubf_{\Xbf}\right\ra_{\H_{K^1}\oplus \H_{K^2}} -\left \la \ubf, \left(\frac{1}{\gamma}D + I\right)^{-1}\ubf\right\ra_{\H_{K^1}\oplus \H_{K^2}} \right|
		\\
		& = \frac{1}{\gamma}\Delta_1 + \frac{1}{\gamma^2}\Delta_2.
	\end{align*}
	By Lemma \ref{lemma:mean-square-concentration}, for any $0 < \delta <1$, with probability at least $1-\delta$,
	\begin{align*}
		\Delta_1 \leq (B_1+B_2)^2\left(\frac{2\log\frac{2}{\delta}}{m} + \sqrt{\frac{2\log\frac{2}{\delta}}{m}}\right).
	\end{align*}
	By Lemma \ref{lemma:mean-term-difference},
	\begin{align*}
		\Delta_2 \leq \frac{1}{\gamma}||D_X - D||\;||\ubf_{\Xbf}||\;||\ubf|| + ||\ubf_{\Xbf}-\ubf||[||\ubf_{\Xbf}||+||\ubf||].
	\end{align*}
	By Lemma \ref{lemma:RK-concentration},
	\begin{align*}
		&||\ubf||^2_{\H_{K^1} \oplus \H_{K^2}} = ||\ubf_1||^2_{\H_{K^1}} + ||\ubf_2||^2_{\H_{K^2}} \leq (1-r)\kappa_1^2(B_1+B_2)^2 + r\kappa_2^2(B_1+B_2)^2
		\\
		&\imply ||\ubf||_{\H_{K^1} \oplus \H_{K^2}} \leq \sqrt{(1-r)\kappa_1^2 + r\kappa_2^2}(B_1+B_2).
	\end{align*}
	Similarly,
	$||\ubf_{\Xbf}||_{\H_{K^1} \oplus \H_{K^2}} \leq \sqrt{(1-r)\kappa_1^2 + r\kappa_2^2}(B_1+B_2)$.
	By Lemma \ref{lemma:mean-square-concentration}, for any $0 < \delta < 1$,
	the following hold simultaneously
	\begin{align*}
		||\ubf_{1,\Xbf}-\ubf_1||_{\H_{K^1}}\leq \sqrt{1-r}\kappa_1 (B_1+B_2)\left(\frac{2\log\frac{4}{\delta}}{m} + \sqrt{\frac{2\log\frac{4}{\delta}}{m}}\right),
		\\
		||\ubf_{2,\Xbf}-\ubf_2||_{\H_{K^2}}\leq \sqrt{r}\kappa_2 (B_1+B_2)\left(\frac{2\log\frac{4}{\delta}}{m} + \sqrt{\frac{2\log\frac{4}{\delta}}{m}}\right).
	\end{align*}
	It follows that
	\begin{align*}
		||\ubf_{\Xbf}-\ubf||_{\H_{K^1}\oplus \H_{K^2}} =\sqrt{||\ubf_{1,\Xbf}-\ubf_1||^2_{\H_{K^1}} +||\ubf_{2,\Xbf}-\ubf_2||^2_{\H_{K^2}} }
		\\
		\leq \sqrt{(1-r)\kappa_1^2 + r\kappa_2^2}(B_1+B_2)\left(\frac{2\log\frac{4}{\delta}}{m} + \sqrt{\frac{2\log\frac{4}{\delta}}{m}}\right).
	\end{align*}
	From the proof of Theorem \ref{theorem:Renyi-estimate-finite-covariance},
	for any $0 < \delta < 1$, with probability at least $1-\delta$,
	\begin{align*}
		||D_{\Xbf}-D||_{\HS} \leq [(1-r)\kappa_1^2 + r\kappa_2^2]\left(\frac{2\log\frac{6}{\delta}}{m} + \sqrt{\frac{2\log\frac{6}{\delta}}{m}}\right).
	\end{align*}
	Combining all the expressions for $D_{\Xbf},D,\ubf_{\Xbf}, \ubf$, we have that for any $0 < \delta < 1$, with probability at least $1-\delta$,
	\begin{align*}
		\Delta_2 &\leq \frac{1}{\gamma}[(1-r)\kappa_1^2 + r\kappa_2^2]^2(B_1+B_2)^2\left(\frac{2\log\frac{12}{\delta}}{m} + \sqrt{\frac{2\log\frac{12}{\delta}}{m}}\right)
		\\
		&\quad + 2[(1-r)\kappa_1^2 + r\kappa_2^2](B_1+B_2)^2\left(\frac{2\log\frac{8}{\delta}}{m} + \sqrt{\frac{2\log\frac{8}{\delta}}{m}}\right).
	\end{align*}
	Combining this with the expression for $\Delta_1$, we have that for any $0 < \delta < 1$, with probability at least $1-\delta$,
	\begin{align*}
		&\Delta \leq \frac{1}{\gamma}(B_1+B_2)^2\left(\frac{2\log\frac{4}{\delta}}{m} + \sqrt{\frac{2\log\frac{4}{\delta}}{m}}\right) + \frac{1}{\gamma^3}[(1-r)\kappa_1^2 + r\kappa_2^2]^2(B_1+B_2)^2\left(\frac{2\log\frac{24}{\delta}}{m} + \sqrt{\frac{2\log\frac{24}{\delta}}{m}}\right)
		\\
		&\quad + \frac{2}{\gamma^2}[(1-r)\kappa_1^2 + r\kappa_2^2](B_1+B_2)^2\left(\frac{2\log\frac{16}{\delta}}{m} + \sqrt{\frac{2\log\frac{16}{\delta}}{m}}\right)
		\\
		& \leq \frac{1}{\gamma}(B_1+B_2)^2\left[1 + \frac{(1-r)\kappa_1^2 + r\kappa_2^2}{\gamma}\right]^2\left(\frac{2\log\frac{24}{\delta}}{m} + \sqrt{\frac{2\log\frac{24}{\delta}}{m}}\right).
	\end{align*}
\end{proof}

\begin{proof}
	\textbf{of Theorem \ref{theorem:Renyi-estimate-Gaussian-process-general-case}}
	This follows by combining Proposition \ref{proposition:Renyi-mean-Gaussian-process-estimate} with Theorem \ref{theorem:Renyi-estimate-finite-covariance}.
\end{proof}
\begin{proof}
	\textbf{of Theorem \ref{theorem:KL-estimate-Gaussian-process-general-case}}
	This follows by combining Proposition \ref{proposition:Renyi-mean-Gaussian-process-estimate}, $r=1$, with 
	Theorem \ref{theorem:KL-estimate-finite-covariance}.
\end{proof}

\subsection{Miscellaneous technical results} 
\label{section:misc}

In this section, we present technical results that are used in the proofs of the main results.
The following results are obtained by applying L'Hopital's rule
\begin{lemma}
	\label{lemma:limit-log-sum}
	Let $x \in \R, x > -1$ be fixed. Then
	\begin{align}
		\lim_{\beta \approach 0}\frac{\log\left(\beta(1+x)^{1-\beta} + (1-\beta)(1+x)^{-\beta}\right)}{\beta (1-\beta)}
		&= x -\log(1+x),
		\\
		\lim_{\beta \approach 1}\frac{\log\left(\beta(1+x)^{1-\beta} + (1-\beta)(1+x)^{-\beta}\right)}{\beta (1-\beta)}
		&=\log(1+x) - \frac{x}{1+x}.
	\end{align}
\end{lemma}

\begin{lemma}
	\label{lemma:log-sum-inequality-1}
	Let $0 \leq \beta \leq 1$ be fixed. For $x > -1$,
	\begin{align}
		0 \leq \log\left(\beta(1+x)^{1-\beta} + (1-\beta)(1+x)^{-\beta}\right) \leq {\beta(1-\beta)}\frac{x^2}{1+x}. 
	\end{align}
	For $0 < \beta < 1$, equality on both sides happens if and only if $x = 0$.
	For $x \geq 0$,
	\begin{align}
		0 \leq \log\left(\beta(1+x)^{1-\beta} + (1-\beta)(1+x)^{-\beta}\right) \leq \frac{\beta(1-\beta)}{2}x^2. 
	\end{align}
	For $0 < \beta < 1$, equality on both sides happens if and only if $x = 0$.
\end{lemma}
\begin{proof}
	For $\beta = 0$ or $\beta =1$, all quantities involved are equal to zero. Consider the case $0 < \beta < 1$.
	Define the following function $f: (-1, \infty) \mapto \R$ by
	\begin{align*}
		f(x) &= \log[\beta(1+x)^{1-\beta} + (1-\beta)(1+x)^{-\beta}] = \log[\beta(1+x) + (1-\beta)] - \beta\log(1+x)
		\\
		& = \log(1+\beta x) - \beta\log(1+x).
	\end{align*}
	Then $f'(x)=
	\frac{\beta(1-\beta)x}{(1+\beta x)(1+x)}$.
	Since $0 < \beta < 1$, we have $1+\beta x > 0 \forall x > -1$.
	Thus $f'(0) =0$, $f'(x) > 0$ for $x > 0$, $f'(x) < 0$ for $-1<x < 0$. Hence $f$ has a unique global minimum $\min{f} = f(0) = 0$, showing that
	$f(x) \geq 0$ $\forall x > -1$, with equality if and only if $x = 0$.
	
	Define $g(x) = \log(1+\beta x) - \beta \log(1+x) - \beta(1-\beta)\frac{x^2}{1+x}$ on $(-1, \infty)$.
	Then
	\begin{align*}
		g'(x) = \beta(1-\beta)\left[\frac{x}{(1+\beta x)(1+x)} - 1 + \frac{1}{(1+x)^2}\right] = -\beta(1-\beta)\frac{x(1+2\beta x + \beta x^2)}{(1+x)^2(1+\beta x)}.
	\end{align*}
	For $0 < \beta < 1$ fixed, we have $1+2\beta x + \beta x^2 > 0$ $\forall x \in \R$. Thus
	$g'(x) = 0 \equivalent x = 0$, with $g'(x) > 0$ for $-1 < x < 0$ and $g'(x) < 0$ for $x > 0$.
	Hence $g$ has a unique global maximum,
	namely $\max{g} = g(0) = 0$,
	so that $g(x) \leq 0$ $\forall x > -1$, with equality if and only $x = 0$.
	
	Define $h(x) = \log(1+\beta x) - \beta \log(1+x) - \frac{\beta(1-\beta)}{2}x^2$ on $[0, \infty)$.
	Then $h'(x) = \frac{\beta(1-\beta)x}{(1+x)(1+\beta x)} -\beta(1-\beta)x $, 
	with $h'(0) =0$ and $h'(x) < 0$ for $x > 0$. Thus $h$ has a unique global maximum $\max{h} = h(0) = 0$.
	Hence $h(x) \leq 0$, with equality if and only if $x = 0$.
\end{proof}

\begin{lemma}
	\label{lemma:log-x2-inverse-inequality-1}
	Let $0 \leq \beta \leq 1$ be fixed. For $x,y \in (-1, \infty)$,
	\begin{align}
		&|\log[\beta(1+x)^{1-\beta} + (1-\beta)(1+x)^{-\beta}] - \log[\beta(1+y)^{1-\beta} + (1-\beta)(1+y)^{-\beta}]|
		\nonumber
		\\
		&\leq \beta(1-\beta)\left|\frac{x^2}{1+x} - \frac{y^2}{1+y}\right|.
	\end{align}
	For $x,y \geq 0$,
	\begin{align}
		&|\log[\beta(1+x)^{1-\beta} + (1-\beta)(1+x)^{-\beta}] - \log[\beta(1+y)^{1-\beta} + (1-\beta)(1+y)^{-\beta}]|
		\nonumber
		\\
		&\leq \frac{\beta(1-\beta)}{2}|x^2 - y^2|.
	\end{align}	
\end{lemma}
\begin{proof}
	(i) Consider first the setting $x, y \in (-1, \infty)$.
	For $\beta =0, 1$, both sides are equal to zero. Consider the case $0 < \beta < 1$.
	For $x \in (-1, \infty)$,
	define $f_1(x) = \log[\beta(1+x)^{1-\beta} + (1-\beta)(1+x)^{-\beta}] = \log(1+\beta x) - \beta (1+x)$, $f_2(x) = \beta (1-\beta)\frac{x^2}{1+x}$, and $f(x) = f_1(x) - f_2(x)$.
	
	From the proof of Lemma \ref{lemma:log-sum-inequality-1}, $f$ is increasing for $- 1<x<0$ and decreasing
	for $x > 0$. 
	With $f_2(x) = \beta (1-\beta)(x-1 + \frac{1}{1+x})$, $f_2'(x) = \beta(1-\beta)(1 - \frac{1}{(1+x)^2})$.
	Thus $f_2'(x) < 0$ for $-1 < x < 0$ and $f_2'(x) > 0$ for $x > 0$, so that
	$f_2$ is decreasing for $-1 < x < 0$ and increasing for $x > 0$. Hence
	for $-1 < x \leq y < 0$, $f_2(x) - f_2(y) \geq 0$.
	By the proof of Lemma \ref{lemma:log-sum-inequality-1}, the same is true for $f_1(x)$,
	so that for $-1 < x \leq y < 0$, $f_1(x) - f_1(y) \geq 0$.
	Thus for $-1 \leq x\leq y < 0$,
	\begin{align*}
		f(x) \leq f(y) \equivalent f_1(x) - f_2(x) \leq f_1(y) - f_2(y) \equivalent 0 \leq f_1(x) - f_1(y) \leq f_2(x) - f_2(y).
	\end{align*}
	
	Similarly, for $0 < x \leq y$, $f$ is decreasing but $f_1, f_2$ are both increasing, so that
	\begin{align*}
		f(x) \geq f(y) &\equivalent f_1(x) - f_2(x) \geq f_1(y) -f_2(y) \equivalent 0 \geq f_1(x) - f_1(y) \geq f_2(x) - f_2(y)
		\\
		&\equivalent 0 \leq f_1(y) - f_1(x) \leq f_2(y) - f_2(x). 
	\end{align*}
	Thus in all cases, we have $|f_1(x) - f_1(y)| \leq |f_2(x) - f_2(y)|$.
	
	(ii) Consider the setting $x,y \geq 0$, with $0 < \beta < 1$.
	Let $f_3(x) = \frac{\beta(1-\beta)}{2}x^2$, then $f_3$ is increasing on $[0, \infty)$.
	From the proof of Lemma \ref{lemma:log-sum-inequality-1},
	the function $h(x) = f_1(x) - f_3(x)$ is decreasing on $[0, \infty)$.
	Thus for $0 \leq x \leq y$,
	\begin{align*}
		h(x) \geq h(y) &\equivalent f_1(x) - f_3(x) \geq f_1(y) - f_3(y) \equivalent 0 \geq f_1(x) - f_1(y)\geq f_3(x) - f_3(y)
		\\
		& \equivalent 0 \leq f_1(y) - f_1(x) \leq f_3(y) - f_3(x).
	\end{align*}
	Thus for $x,y \geq 0$, we have $|f_1(x) - f_1(y)| \leq |f_3(x) - f_3(y)|$.
\end{proof}

\begin{lemma}
	\label{lemma:log-bound-1}
	For $x \in (-1, \infty)$,
	\begin{align}
		0 \leq x - \log(1+x) \leq \frac{x^2}{1+x},
		\;\;\;
		0 \leq \log(1+x) - \frac{x}{1+x} \leq \frac{x^2}{1+x}.
	\end{align}
	Furthermore, for $x \geq 0$,
	\begin{align}
		0 \leq x - \log(1+x) \leq \frac{x^2}{2},
		\;\;\;
		0  \leq \log(1+x) - \frac{x}{1+x} \leq \frac{x^2}{2}.
	\end{align}
\end{lemma}
\begin{proof}
	For $f(x) = x-\log(1+x)$, $f'(x) = 1- \frac{1}{1+x}$, with $f'(0) = 0$, $f'(x) < 0$ for $-1 < x < 0$ and
	$f'(x) > 0$ for $x > 0$. Thus $f$ has a global minimum at $x=0$, that is $\min{f} = f(0) = 0$. 
	
	Let $h(x) = \frac{x^2}{1+x} - x +\log(1+x) = -1 + \frac{1}{1+x} + \log(1+x)$ for $x > -1$, then
	$h'(x)
	= \frac{x}{(1+x)^2} < 0$ for $-1 < x < 0$, $h'(x) > 0$ for $x > 0$, and $h'(0) = 0$. Thus on $(-1, \infty)$, $\min{h} = h(0) = 0$.
	Together, these give $0 \leq x - \log(1+x) \leq \frac{x^2}{1+x}$ for $x \in (-1,\infty)$.
	Let $z = -\frac{x}{1+x}$, then $z \in (-1, \infty)$ for $x \in (-1, \infty)$. Thus $0 \leq \log(1+x) - \frac{x}{1+x} = z - \log(1+z) \leq \frac{z^2}{1+z} = \frac{x^2}{1+x}$.
	
	Let $g(x)= \frac{x^2}{2} - x + \log(1+x)$, then $g'(x) = x - 1+\frac{1}{1+x} = \frac{x^2}{1+x} > 0$ for $x \neq 0$, with $g'(0) =0$. Thus on $[0, \infty)$, $\min{g} = g(0) = 0$.
	
	Let $h_1(x) = \log(1+x) - \frac{x}{1+x} - \frac{x^2}{2} = \log(1+x)-1 + \frac{1}{1+x} - \frac{x^2}{2}$,
	then $h_1'(x) = x \left(\frac{1}{(1+x)^2}-1\right) < 0$ for $x > 0$. Thus
	$h_1$ is decreasing on $[0,\infty)$, with $\max{h_1} = h_1(0) = 0$.
\end{proof}

\begin{lemma}
	\label{lemma:log-bound-2}
	Let $f(x) = x- \log(1+x)$ for $x > -1$. Then
	\begin{align}
		|f(x) - f(y)| &\leq \left|\frac{x^2}{1+x} - \frac{y^2}{1+y}\right|, \;\;\; x > -1, y > -1,
		\\
		|f(x) - f(y)| &\leq \frac{1}{2}|x^2 - y^2|, \;\;\; x \geq 0, y \geq 0.
	\end{align}
\end{lemma}
\begin{proof}
	(i) Consider the case $x > -1,y > -1$.
	Let $f_1(x) = \frac{x^2}{(1+x)} = x-1 + \frac{1}{1+x}$, $x \neq -1$,
	then $f_1'(x) = 1-\frac{1}{(1+x)^2} < 0$ on $(-1,0]$, thus
	$f_1$ is decreasing on $(-1,0]$.
	By the proof of Lemma \ref{lemma:log-bound-1}, 
	$f$ is decreasing on $(-1,0]$ and
	the function 
	$h(x) = \frac{x^2}{(1+x)} - x +\log(1+x)$ is decreasing on $(-1,0]$. Thus
	for $-1 < x \leq y \leq 0$,
	\begin{align*}
		h(x) \geq h(y) \geq 0 \equivalent  f_1(x)- f(x) \geq f_1(y) - f(y) \equivalent 0 \leq f(x) - f(y) \leq f_1(x) - f_1(y).
	\end{align*}
	For $x \geq 0$, $f'_1(x) > 0$ for $x > 0$, so that $f_1$ is increasing on $[0,\infty)$. Similarly, both $f$ and $h$ are increasing on $[0,\infty)$. Thus for $0 \leq x\leq y$,
	\begin{align*}
		h(x) \leq h(y) &\equivalent f_1(x) - f(x) \leq f_1(y) - f(y) \equivalent f_1(x) - f_1(y) \leq f(x) - f(y) \leq 0
		\\
		&\equivalent 0 \leq f(y) - f(x) \leq f_1(y) - f_1(x).
	\end{align*}
	Thus in all cases, we have $|f(x) - f(y)| \leq |f_1(x) - f_1(y)|$.
	
	(ii) Consider the case $x \geq 0, y \geq 0$.
	The function $g(x) = \frac{x^2}{2} - x + \log(1+x)$ is increasing on $[0,\infty)$, thus
	for $0 \leq x \leq y$,
	\begin{align*}
		g(x) \leq g(y) &\equivalent \frac{x^2}{2} - f(x) \leq \frac{y^2}{2} - f(y) \equivalent 
		0 \geq f(x) - f(y) \geq \frac{x^2}{2} - \frac{y^2}{2}
		\\
		& \equivalent 0 \leq f(y) - f(x) \leq \frac{y^2}{2} - \frac{x^2}{2}.
	\end{align*}
	Thus in this case we have $|f(x) - f(y)| \leq \frac{1}{2}|x^2 - y^2|$.
\end{proof}

\begin{lemma}
	\label{lemma:log-bound-3}
	Let $f(x) = \log(1+x) - \frac{x}{1+x}$ for $x > -1$. Then
	\begin{align}
		|f(x) - f(y)| &\leq \left|\frac{x^2}{1+x} - \frac{y^2}{1+y}\right|, \;\;\; x > -1, y > -1,
		\\
		|f(x) - f(y)| &\leq \frac{1}{2}|x^2 - y^2|, \;\;\; x \geq 0, y \geq 0.
	\end{align}
\end{lemma}
\begin{proof}
	This is entirely similar to Lemma \ref{lemma:log-bound-2}.
\end{proof}



\acks{This work is partially supported by JSPS KAKENHI Grant Number JP20H04250.}







\vskip 0.2in
\bibliography{/Users/Minhs/Dropbox/cite_RKHS}

\end{document}